\documentclass[12pt,reqno]{amsart}

\usepackage{latexsym}
\usepackage{amsmath}
\usepackage{amscd}
\usepackage{amssymb}
\usepackage{mathrsfs}
\usepackage{mathtools}
\usepackage{comment}
\usepackage{color,framed}
\usepackage{enumerate}
\usepackage{soul}
\usepackage{stmaryrd}
\usepackage[ruled]{algorithm2e} 
\usepackage{braket}

\usepackage{graphicx}
\usepackage{subfigure}
\usepackage{url}

\def\Id{{\rm Id}}

\newtheorem{theorem}{Theorem}[section]

\newtheorem{assumption}{Assumption}

\newtheorem{lemma}[theorem]{Lemma}

\newtheorem{proposition}[theorem]{Proposition}
\newtheorem{remark}[theorem]{Remark}

\numberwithin{equation}{section}
\begin{document}

\title[A Geometric Framework for Stochastic Shape Analysis]{A Geometric Framework for \\ Stochastic Shape Analysis}  

\author[A. Arnaudon]{Alexis Arnaudon}
\author[D. Holm]{Darryl D Holm}
\author[S. Sommer]{Stefan Sommer}

\address{AA, DH: Department of Mathematics, Imperial College, London SW7 2AZ, UK}
\address{SS: Department of Computer Science (DIKU), University of Copenhagen,
  DK-2100 Copenhagen E, Denmark}

\subjclass[2010]{60G99,70H99,65C30}
\maketitle

\begin{abstract}
We introduce a stochastic model of diffeomorphisms, whose action on a variety of data types descends to stochastic evolution of shapes, images and landmarks.
The stochasticity is introduced in the vector field which transports the data in the Large Deformation Diffeomorphic Metric Mapping (LDDMM) framework for shape analysis and image registration. 
The stochasticity thereby models errors or uncertainties of the flow in following the prescribed deformation velocity.
The approach is illustrated in the example of finite dimensional landmark manifolds, whose stochastic evolution is studied both via the Fokker-Planck equation and by numerical simulations. 
We derive two approaches for inferring parameters of the stochastic model from landmark configurations observed at discrete time points.  
The first of the two approaches matches moments of the Fokker-Planck equation to sample moments of the data, while the second approach employs an Expectation-Maximisation based algorithm using a Monte Carlo bridge sampling scheme to optimise the data likelihood.
We derive and numerically test the ability of the two approaches to infer the spatial correlation length of the underlying noise.

\end{abstract}

\setcounter{tocdepth}{1}

%\tableofcontents

\section{Introduction}

In this work, we aim at modelling variability of shapes using a theory of stochastic perturbations consistent with the action of the diffeomorphism group underlying the Large Deformation Diffeomorphic Metric Mapping framework (LDDMM, see \cite{younes_shapes_2010}).  
In applications, such variability arise and can be observed, for example, when human organs are influenced by disease processes, as analysed in computational anatomy \cite{younes_evolutions_2009}.
Spatially independent white noise contains insufficient information to describe these large-scale variabilities of shapes. In addition, the coupling of the spatial correlations of the noise must be adapted to a variety of transformation properties of the shape spaces. 
The theory developed here addresses this problem by introducing spatially correlated transport noise which respects the geometric structure of the data. 
This method provides a new way of characterising stochastic variability of shapes using spatially correlated noise in the context of the standard LDDMM framework.

We will show that this specific type of noise can be used for all the data structures to which the LDDMM framework applies. 
The LDDMM theory was initiated by \cite{trouve_infinite_1995,christensen_deformable_1996,dupuis_variational_1998,MiTrYo2002,beg2005computing} based on the pattern theory of \cite{grenander_general_1994}. LDDMM models the dynamics of shapes by the action of diffeomorphisms (smooth invertible transformations) on shape spaces. 
It gives a unified approach to shape modelling and shape analysis that is valid for a range of structures such as landmarks, curves, surfaces, images, densities or even tensor-valued images. 
For any such data structure, the optimal shape deformations are described via the Euler-Poincar\'e equation of the diffeomorphism group, usually referred to as the EPDiff equation \cite{holm1998euler,HoMa2004,younes_evolutions_2009}. 
In this work, we will show how to obtain a stochastic
EPDiff equation valid for any data structure, and in particular for the finite dimensional spaces of landmarks.
For this, we will follow the LDDMM derivation in \cite{bruveris2011momentum} based on geometric mechanics \cite{marsden_introduction_1999,holm_geometric_2011}. This view is based on the existence of momentum maps, which are characterized by the transformation properties of the data structures for images and shapes.
These momentum maps persist in the process of introducing noise into the EPDiff equation, and they thereby preserve most of the technology developed for shape analysis in the deterministic context and in computational anatomy.

This work is not the first to consider stochastic evolutions in LDDMM.
Indeed, \cite{TrVi2012,vialard2013extension} and more recently \cite{marsland2016langevin} have already investigated the possibility of stochastic perturbations of landmark dynamics. 
In these works, the noise is introduced into the momentum equation, as though it was an external random force acting on each landmark independently. 
In \cite{marsland2016langevin}, an extra dissipative force was added to balance the energy input from the noise and to make the dynamics correspond to a certain type of heat bath used in statistical physics.
\cite{staneva_learning_2017,sommer_bridge_2017} considered evolutions
on the landmark manifold with stochastic parts being Brownian motion with respect to a Riemannian
metric and estimated parameters of the models from observed data.
Here, we will introduce Eulerian noise directly into the reconstruction relation used to find the deformation flows from the velocity fields, which are solutions of the EPDiff equation \cite{HoMa2004,younes_shapes_2010}. 
As we will see, this derivation of stochastic models is compatible with variational principles, preserves the momentum map structure and yields a stochastic EPDiff equation with a novel type of multiplicative noise, depending on the gradient of the solution, as well as its magnitude. 
This model is based on the previous works \cite{holm2015variational, arnaudon2016noise}, where, respectively, stochastic perturbations of infinite and finite dimensional mechanical systems were considered. 
The Eulerian nature of the noise discussed here implies that the noise correlation depends on the image position and not, as for example in \cite{TrVi2012,marsland2016langevin}, on the landmarks themselves.  
Consequently, the present method for the introduction of noise is compatible with any data structure, for any choice of its spatial correlation.  
We also mention the conference paper \cite{arnaudon2016stochastic2} in which the basic theory underlying the present work was applied to shape transformations of the corpus callosum. We discuss possibilities for including Lagrangian noise advected with the flow in contrast to the present Eulerian case, and possibilities for including non-stationary correlation statistics that responds to the evolution of advected quantities, in the conclusion of the paper.
\begin{figure}[htpb]
    \centering
    \subfigure[Low resolution and large noise correlation ($100$ landmarks, $6\times 6$ noise fields)]{\includegraphics[scale=0.6]{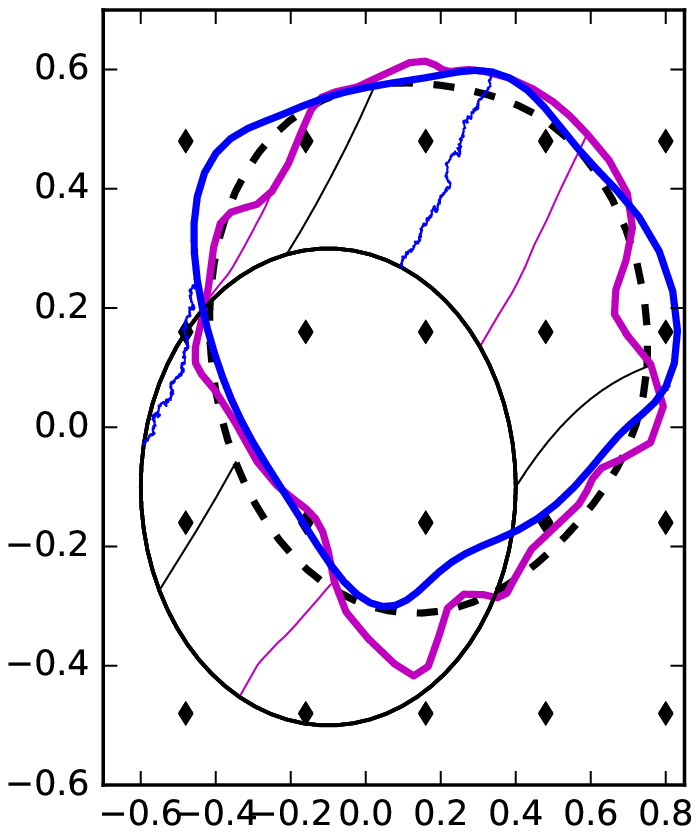}}
    \subfigure[High resolution and large noise correlation ($200$ landmarks, $6\times 6$ noise fields)]{\includegraphics[scale=0.6]{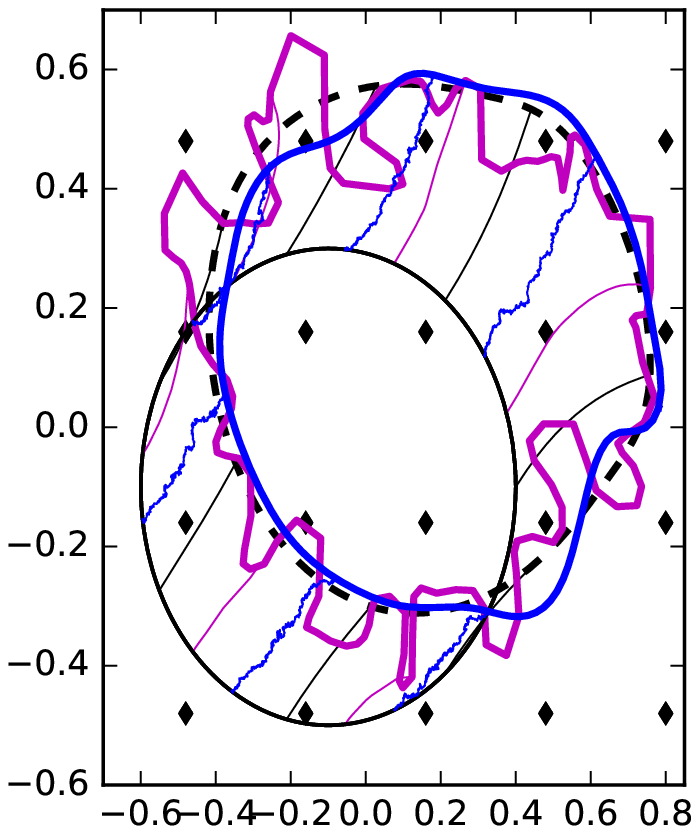}}
    \subfigure[High resolution and  small noise correlation ($200$ landmarks, $12\times 12$ noise fields) ]{\includegraphics[scale=0.6]{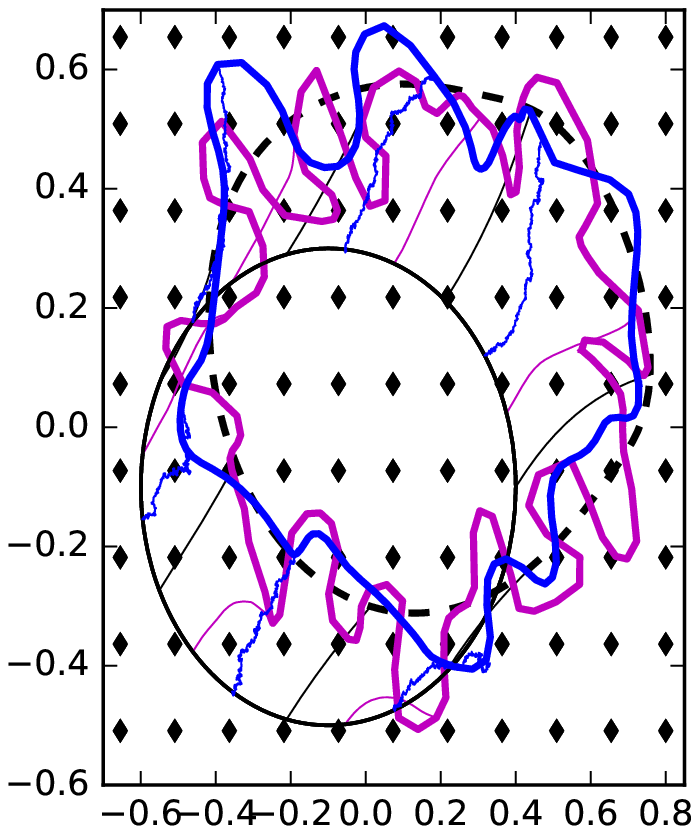}}
    \caption{\small In this figure, we compare the deterministic evolution of landmarks arranged in an ellipse (black line) with a translated ellipse as final position (black dashed line), to two different stochastically perturbed evolutions.
    The radius for the landmark kernel is twice their average initial distances. 
    In blue is the stochastic perturbation developed in this paper. The black dots represent the $J$ Eulerian noise fields arranged in a grid configuration.
    In magenta is the evolution resulting from additive noise in the momentum equation, different for each landmark but with the same amplitude as the Eulerian noise. 
    We run three initial value simulations to compare the limit of a large number of landmarks and small noise correlation. 
    The Eulerian noise model (blue) is robust to the continuum limit and can reproduce the general behaviour of the additive noise model. 
    Furthermore, the choice of the noise fields provides an additional freedom in parameterisation which will be studied and exploited in this work. 
    }
    \label{fig:comparison}
\end{figure}

To illustrate this framework and give an immediate demonstration of stochastic landmark dynamics, we display in Figure~\ref{fig:comparison} three experiments which compare the proposed model with a stochastic forcing model, of the type studied in \cite{TrVi2012}. 
The proposed model introduces the following stochastic Hamiltonian system for the positions of the landmarks, $\mathbf q_i$,  and their canonically conjugate momenta, $\mathbf p_i$, 
\begin{align}
    \begin{split}
    d \mathbf q_i &= \frac{\partial h}{\partial \mathbf p_i} dt + \sum_l\sigma_l(\mathbf q_i) \circ d W_t^l
    \,,\\
    d \mathbf p_i &= -\frac{\partial h}{\partial \mathbf q_i} dt 
- \sum_l \frac{\partial}{\partial \mathbf q_i}\left (\mathbf p_i \cdot\sigma_l(\mathbf q_i)\right )  \circ  d W_t^l \,.
    \end{split}
    \label{sto-Ham-intro}
\end{align}
In \eqref{sto-Ham-intro}, the $\sigma_i$ are prescribed functions of space which represent the spatial correlations of the noise. 
In Figure~\ref{fig:comparison}, the $\sigma_i$ fields are Gaussians whose variance is equal to twice their separation distance and locations are indicated by black dots. 
We compare this model with the system, 
\begin{align}
    d \mathbf q_i^\alpha &= \frac{\partial h}{\partial \mathbf p_i^\alpha} dt \qquad \mathrm{and} \qquad d \mathbf p_i^\alpha = -\frac{\partial h}{\partial \mathbf q_i^\alpha} dt + \sigma dW_t^i\, , 
    \label{sto-TrVi}
\end{align}
where $\sigma$ is a constant.  
In this case, the noise corresponds to a stochastic force acting on the landmarks, whose corresponding Brownian motion is different for each landmark. 
We show on the first panel of Figure~\ref{fig:comparison} that for a small number of landmarks and a large range of spatial correlations of the noise, both types of stochastic deformations in \eqref{sto-Ham-intro} and \eqref{sto-TrVi}  visually coincide. 
This is shown for a simple experiment in translating a circle (from the black circle to the black dashed circle). 
By doubling the number of landmarks (middle panel), the dynamics of \eqref{sto-TrVi} results in small-scale noise correlation (magenta), whereas the proposed model (blue) remains equivalent to the first experiment.
This figure illustrates shape evolution when the noise is Eulerian and independent of the data structure. 
Indeed, the limit of a large number of landmarks corresponds to a certain continuum limit, in this case corresponding to curve dynamics. 
Finally, in the right-most panel, we reduce the range of the spatial correlation of the noise by adding more noise fields. 
This arrangement allows us to qualitatively reproduce the dynamics of the equation \eqref{sto-TrVi} with the same number of landmarks as the amount of noise and its spatial correlation is similar in both cases. 
Indeed, the spatial correlations are dictated by the Eulerian functions $\sigma_l$ defined in fixed space for our model, and by the density of landmarks in the stochastically forced landmark model.

Modelling large-scale shape variability with noise is of interest for applications in computational anatomy, in which sources of variability include natural ageing, the influence of diseases such as Alzheimer's disease, and intra-subject population scale variations.
In the LDDMM context, these effects are sometimes modelled using the random orbit model \cite{miller_statistical_1997}. The random orbit approach models variability in the observed data by using an ensemble of initial velocities in matching a template to a set of observations via geodesic flows, see \cite{vaillant_statistics_2004}. 
The randomness is confined to the initial
velocity as opposed to the evolving stochastic processes used in the present work. A prior can be defined by assuming a
distribution of the initial velocities, and 
Bayesian approaches can then be used for inference of the template shape as well as 
additional unknown parameters
\cite{allassonniere_towards_2007,ma_bayesian_2008,zhang_bayesian_2013}.
The stochastic model developed here can also be applied to model random warps and to generate distributions used in Bayesian shape modelling, and for coupling warps and functional variations such as those in \cite{raket2014nonlinear,kuhnel_most_2017}.
Indeed, because the proposed probabilistic approach assigns a likelihood to random deformations, the model can be used for general likelihood-based inference tasks.

In the present model, the observed shape variability indicates the required spatial correlation of the noise, which must be specified or inferred for each application. As this correlation is generally unknown, estimating the parameters of the correlation structure becomes an important part of the framework.
We will address the problem of inferring the noise parameters by considering two different methods in the context of the representation of shapes by landmarks:
The first method is based on estimating the time evolution of the probability distribution of each landmark. 
We will derive a set of differential equations approximating the time evolution of the complete distribution via its first moments. 
We can then solve the inverse problem of estimating the noise correlation from known initial and final distribution of landmarks by minimization of a certain cost function, solved using a genetic algorithm. 
The second method is based on an Expectation-Maximisation (EM) algorithm which can
infer unknown parameters for a parametric statistical model from observed data.
In this context, since only initial and final landmarks positions are observed, the full stochastic trajectories are regarded as missing information. 
For this algorithm, we need to estimate the likelihood of stochastic paths connecting sets of observed landmarks. 
We achieve this by adapting the theory of diffusion bridges to the stochastic landmark equation. 
As discussed in the concluding remarks, inference methods for other data structures, in particular for infinite dimensional shape representations, are not treated in this paper and left as outstanding problems for future work.

Finally, we wish to mention that multiple additional approaches for shapes analysis exist outside the LDDMM context, particularly exemplified by the Kendall shape spaces \cite{kendall1984shape}, see also \cite{dryden2016statistical}. We focus this paper on the LDDMM framework leaving possibilities for extending the presented methods to include stochastic dynamics and noise inference in other shape analysis approaches to future work.

\subsection*{Plan of this work}

We begin by developing a general theory of stochastic perturbations for inexact
matching in section~\ref{theory}. 
We then focus on exact landmark matching in section~\ref{landmark-matching}, which is the simplest example of this theory. 
In particular, we derive the Fokker-Planck equation in section~\ref{FP-section}
and diffusion bridge simulation in section~\ref{bridges-section}.
In section~\ref{estimate-correlation}, we describe the two methods we use for estimating parameters of the noise from observations. 
The Fokker-Planck based method is discussed in section~\ref{moment-method} and the
Expectation-Maximisation algorithm is treated in section~\ref{EM}. 
We end the paper with numerical examples in section~\ref{landmark-example}, in which
we investigate the effect of the noise on landmark dynamics and compare the two methods for estimating the noise amplitude.

\section{Stochastic Large Deformation Matching}\label{theory}

In this section, we will first review the geometrical framework of LDDMM, following \cite{bruveris2011momentum}, and then introduce noise following \cite{holm2015variational} to preserve the geometrical structure of LDDMM. 
The key ingredient for both topics is the momentum map, which we will use as the main tool for reducing the infinite dimensional equation on the diffeomorphism group to equations on shape spaces.

\subsection{The Deterministic LDDMM Model}

Here, we will briefly review the theory of reduction by symmetry, as applied to the LDDMM context, following the presentation of \cite{bruveris2011momentum}.
We detail the proof of the formulas below in the next section when we include noise. 
Define an energy functional $E$ by
\begin{align}
    E(u_t) = \int_0^1 l( u_t) dt + \frac{1}{2\lambda^2} \|g_1.I_0 - I_1\|^2\, ,
    \label{E-functional}
\end{align}
where $I_0,I_1\in V$ are shapes represented in a vector space $V$ on which the diffeomorphism group $\mathrm{Diff}(\mathbb R^d)$
acts, $u_t$ is a time-dependent vector field, and
$\lambda$ is a weight, or tolerance, which allows the matching to be inexact.
The flow $g_t\in \mathrm{Diff}(\mathbb R^d)$ corresponding to $u_t$  is found by solving the reconstruction relation
\begin{align}
    \partial_t g_t = u_t g_t\, ,    
    \label{rec-det}
\end{align}
and $I_0$ is matched against $I_1$ through the action $g_1.I_0$ of $g_1$ on $s_0$.
The vector field $u_t$ can be considered an element of the Lie algebra $\mathfrak X(\mathbb R^d)$. 
In the case of $I_0,I_1$ being images $I: \mathbb R^d\to \mathbb R$, the action is by push-forward, $g.I=I\circ g^{-1}$, and when $I$ represents $N$ landmarks with positions $\mathbf q_i\in \mathbb R^d$, the action is by evaluation $g.\mathbf q = (g(\mathbf q_1),\ldots, g(\mathbf q_N))$ (see \cite{bruveris2011momentum} for more details).
The group elements can act on various additional shape structures such as tensor fields.

\begin{remark}[Nonlinear shape structures]
    This framework can be extended to structures that are not represented by a vector space $V$, such as curves or surfaces. We leave this extension for future work.
\end{remark}

Using the calculus of variations for the functional \eqref{E-functional} results in the equation of motion for $u_t$ of the form
\begin{align}
    \frac{d}{dt}\frac{\delta l}{\delta u} + \mathrm{ad}^*_{u_t} \frac{\delta l}{\delta u}=0\, ,  
    \label{EP}
\end{align}
which is called the Euler-Poincar\'e equation. 
The operation $\mathrm{ad}^*$ is the coadjoint action of the Lie algebra of vector fields associated to the diffeomorphism group. The operation $\mathrm{ad}^*$ acts on the variations ${\delta l}/{\delta u}$, which are 1-form densities, in the dual of the Lie algebra of vector fields, under the $L^2$ pairing. 
When $l(u)$ is a norm, this equation is the geodesic equation for that norm, in the case that $\lambda=\infty$; that is, with exact matching. 
We will focus on this case later in section~\ref{landmark-matching} when discussing landmark dynamics.
Here, the inexact matching term constrains the form of the momentum $m= \frac{\partial l}{\partial u}$ to depend on the geodesic path. 
Following the notation of \cite{bruveris2011momentum}, the momentum map is defined as
\begin{align}
    m(t)= - \frac{1}{\lambda^2} J_t^0\diamond ( g_{t,1}(J_1^0-J_1^1)^\flat )\, ,
    \label{momentum}
\end{align}
where $g_{t,s}$ is the solution of \eqref{rec-det} at time $t$ with initial conditions at time $s$, while $J_t^0 = g_{t,0} I_0$ and $J_t^1= g_{t,1} I_1$. The value
$J_1^0$ corresponds to the initial shape, pushed forward to time $t=1$, and $J_1^1= I_1$ is the target shape. 

The operations $\diamond$ and $\flat$ in the momentum map formula  \eqref{momentum} are defined, as follows.
The Lagrangian $l$ in \eqref{E-functional} may be taken as kinetic energy, which defines a scalar product and norm $l(u) = \langle u,L u\rangle_{L^2}= \|u\|^2_{L^2}$ on the space of vector fields $\mathfrak X(\mathbb R^d)$. 
The quantity $Lu={\delta l}/{\delta u}$ may then be regarded as the momentum conjugate to the velocity $u$.
Similarly, for the image data space $V$, we define the dual space $V^*$  with the $L^2$ pairing $\langle f,I\rangle = \int_\Omega f(x)I(x) dx$, where $f\in V^*$ and $\Omega$ is the image domain $\Omega \in \mathbb R^d$. 
This identification defines the $\flat$ operator as $\flat: V\to V^*$.
When an element $g_t$ of the diffeomorphism group acts on $V$ by push-forward, $I_t=g_t.I_0 = (g_t)_*I_0$, the corresponding infinitesimal action of the velocity $u$ in the Lie algebra of vector fields $u\in\mathfrak X(\mathbb R^d)$ is given by  $u.I:= [g_t^*\frac{d}{dt}(g_t)_*I_0]_{t=0}$. 
In terms of this infinitesimal action, we can then define the operation $\diamond:V\times V^*\to \mathfrak g^*$ as
\begin{align}
    \langle I\diamond f, u\rangle_{\mathfrak g \times \mathfrak g^*}:= \langle f, u.I\rangle_{V\times V^*}\, . 
\end{align}
A detailed derivation of this formula for the momentum map can be found in \cite{bruveris2011momentum}. 

\begin{remark}[Solving this equation]
We will just add here the important remark that the relation \eqref{momentum} introduces nonlocality into the problem, as the momentum implicitly depends on the value of the group at later times. 
This is exactly what is needed in order to solve the boundary value problem coming from the matching of images $I_1$ and $I_0$. 
The optimal vector field can be found with a shooting method or a gradient descent algorithm on the energy functional \eqref{E-functional}, see \cite{beg2005computing}.
For more information about the relation of the momentum map approach of \cite{bruveris2011momentum} to the LDDMM approach of \cite{beg2005computing}, see \cite{bruveris2015geometry}.
\end{remark}

\subsection{Stochastic Reduction Theory}

The aim here is to introduce noise in the Euler-Poincar\'e equation \eqref{EP} while preserving the momentum map \eqref{momentum}; so that the noise descends to the shape spaces.
Following \cite{holm2015variational}, we introduce noise in the reconstruction relation \eqref{rec-det} and proceed with the theory of reduction by symmetry.  
We will focus on a noise described by a set of $J$ real-valued independent Wiener processes $W^i_t$ together with $J$ associated vector fields $\sigma_i\in\mathfrak X(\mathbb R^d)$ on the data domain.
We will later discuss particular forms of these fields and methods for estimating unknown parameters of the fields in the context of landmark matching.  

\begin{remark}[Dimension of the noise]
  \label{rem:dimension_noise}
We proceed here with a finite number of $J$ associated vector fields and finite dimensional noise while leaving possible extension to infinite dimensional noise such as done by \cite{vialard2013extension} for later works.
\end{remark}

We replace the reconstruction relation \eqref{rec-det} by the following stochastic process
\begin{align}
  dg_t = u_t g_t dt + \sum_{l=1}^J \sigma_l g_t \circ dW^l_t\, ,    
    \label{rec-sto}
\end{align}
where $\circ$ denotes Stratonovich integration. 
That is, the Lie group trajectory $g_t$ is now a stochastic process.
With this noise construction, the previous derivations of \eqref{EP} and \eqref{momentum} in \cite{bruveris2011momentum} still apply and we obtain the following result for the stochastic vector field, $u_t$. 

\begin{proposition}
    Under stochastic perturbations of the form \eqref{rec-sto}, the momentum map \eqref{momentum} persists, and the Euler-Poincar\'e equation takes the form 
\begin{align}
  d \frac{\delta l}{\delta u} + \mathrm{ad}^*_{u_t} \frac{\delta l}{\delta u}dt+ \sum_{l=1}^J \mathrm{ad}^*_{\sigma_l} \frac{\delta l}{\delta u}\circ dW^l_t=0\, .  
    \label{EP-sto}
\end{align}
\end{proposition}

\begin{proof}
    We first show that the momentum map formula \eqref{momentum} persists in the presence of noise. 
    The key step in its computation is to prove the formula in the lemma $2.5$ of \cite{bruveris2011momentum} which is given by $\partial_t (g^{-1} \delta g ) = \mathrm{Ad}_g\delta u$, where $\mathrm{Ad}$ is the adjoint action on the diffeomorphism group on its Lie algebra. 
    We first compute the variations of \eqref{rec-sto}
    \begin{align*}
      \delta d g_t = \delta u g dt + u \delta g dt + \sum_{l=1}^J \sigma_l \delta g \circ dW_t^l\, ,
    \end{align*}
    and then prove this formula by a direct computation
    \begin{align*}
        d ( g^{-1} \delta g) &= - g^{-1} dg g^{-1}\delta g  + g^{-1} d \delta g     \\
	&= - g^{-1}( u dt + \sum_{l=1}^J \sigma_l\circ dW_t^l ) \delta g  + g^{-1}( \delta u g dt + u \delta g dt + \sum_{l=1}^J \sigma_l \delta g \circ dW_t^l) \\
        &=  g^{-1} \delta u g\,  dt\\
        &:= \mathrm{Ad}_g \delta u\,  dt \, .
    \end{align*}
    This key formula is the same as in \cite{beg2005computing} and \cite{bruveris2011momentum} for the deterministic case. In particular, it does not explicitly depend on the Wiener processes $W_t^l$. This ensures that the momentum map formula \eqref{momentum} remains the same as in the deterministic case. 
    The last step of the proof is to derive the stochastic Euler-Poincar\'e equation \eqref{EP-sto}. 
    This is done by computing the stochastic evolution of the momentum, given by
    \begin{align*}
        \frac{\delta l}{\delta u} =\mathrm{Ad}^*_{g^{-1}} ( I_0\diamond (g_1^{-1} \pi)), \quad \mathrm{where} \quad \pi=  \frac{1}{\lambda^2} (g_iI_0 - I_1)^\flat\, .  
    \end{align*}
    The only time dependence is in the coadjoint action, and, by the standard formula 
    \begin{align*}
        d \mathrm{Ad}^*_{g^{-1}} \eta = - \mathrm{ad}^*_{dg g^{-1}} \mathrm{Ad}^*_{g^{-1}} \eta\, ,
    \end{align*}
    we obtain the result
    \begin{align*}
        d\frac{\delta l }{\delta u} &= - d \mathrm{Ad}^*_{g^{-1}} ( I_0\diamond (g_1^{-1} \pi))\\
         &=  \mathrm{ad}^*_{dg g^{-1}}\mathrm{Ad}^*_{g^{-1}} ( I_0\diamond (g_1^{-1} \pi))\\
         &=  \mathrm{ad}^*_{dg g^{-1}}\frac{\delta l }{\delta u}\, ,
    \end{align*}
    where we have used the stochastic reconstruction relation \eqref{rec-sto} in the form \[dg g^{-1}= u dt + \sum_{l=1}^N \sigma_l \circ dW_l^t\,.\]
\end{proof}

In summary, this stochastic perturbation of the LDDMM framework preserves the form of momentum map \eqref{momentum}, although it does affect the reconstruction relation \eqref{rec-sto} and the Euler-Poincar\'e equation \eqref{EP-sto}. 
As shown in \cite{bruveris2011momentum}, various data structures fit into this framework including landmarks, images, shapes, and tensor fields. 
In practice, for inexact matching, a gradient descent algorithm can be used to minimise the energy functional \eqref{E-functional}. 
The noise will only appear in the evaluation of the matching cost via the reconstruction relation. 
The algorithm of \cite{beg2005computing} then directly applies, provided the stochastic reconstruction relation can be integrated with enough accuracy. 
We will not treat the full inexact matching problem here. Instead, we will study the simpler case of exact matching, where the energy functional consists only of the kinetic term.

The exact matching problem in computational anatomy possesses many parallels with the geometric approach to classical mechanics and ideal fluid dynamics. 
In particular, Poincar\'e's fundamental paper in 1901, which started the field of geometric mechanics in finite dimensions, has recently been generalised to the stochastic case \cite{cruzeiro2017momentum}. In addition, the fundamental analytical properties of Euler's fluid equations have been shown to extend to the stochastic case in \cite{holm2017euler}.  

We expect that these advances in the analysis of SPDEs occurring in fluid dynamics and other parallel fields will inform computational anatomy, and eventually will apply to infinite-dimensional representations of shape. One reason for our optimism is that the fundamental analytical properties of incompressible Euler fluid dynamics in three dimensions have already been found in \cite{holm2017euler} to persist under the introduction of the present type of stochasticity. Namely, the properties of local-in-time existence and uniqueness, as well as the Beal-Kato-Majda criterion for blow-up for the deterministic 3D Euler fluid motion equations, all persist in detail for stochastic Euler fluid motion, under the introduction of the type of stochastic Lie transport by cylindrical Stratonovich noise that we have proposed here for stochastic shape analysis. 

The persistence of deterministic analytical properties in passing to the SPDEs governing stochastic 3D incompressible continuum fluid dynamics is a type of infinite-dimensional result that has not yet been proven for the evolution of shapes. The corresponding results in the analysis of SPDEs for embeddings, immersions and curves representing data structures for shape evolution, for example, have not yet been discovered, and they remain now as outstanding open problems. However, we believe that the prospects for successfully performing the necessary analysis are hopeful because the type of noise we propose here preserves the fundamental properties of diffeomorphic flow for both continuum fluids and shapes. For example, the momentum maps for the deterministic and stochastic evolution of shapes of any data structure are identical. Thus, the only difference in the present approach from the deterministic case is that the diffeomorphic time evolution of the various shape momentum maps proceeds by the action of Lie derivative by a stochastic vector field, instead of a deterministic one. Since the stochastic part of the vector field is as smooth as we wish, we are hopeful that the analytical properties for the deterministic evolution of a large class of infinite-dimensional representations of shape (such as smooth embeddings) will also persist under the introduction of the type of stochastic transport proposed here. For the remainder of the paper, we restrict ourselves to the treatment of stochastic landmark dynamics.
%\footnote{ Landmark solutions persist under the introduction of the present type of stochasticity, precisely because they comprise a momentum map \cite{HoMa2004}. In the present case, they comprise the momentum map for the left action of diffeomorphisms on points in the plane, and momentum maps persist under the introduction of stochasticity proposed here.}

\section{Exact stochastic Landmark Matching}\label{landmark-matching}

In this section, we apply the previous ideas of stochastic deformation of LDDMM to exact matching with landmark dynamics. 
This is the simplest data structure in the LDDMM framework, and it will serve to give interesting insights into the effect of the noise in this context. 
Since exact matching means that the energy functional contains only a kinetic energy, the optimal vector field is found from a boundary value problem with the Euler-Poincar\'e equation \eqref{EP}. 
For exact matching, the singular momentum map for landmarks takes the simple familiar form for the reduction of the EPDiff equation (see \cite{CaHo1993,HoMa2004})
\begin{align}
    \mathbf m(x,t)= \sum_{i=0}^N \mathbf p_i(t) \delta (\mathbf x-\mathbf q_i(t))\, , 
    \label{momap}
\end{align}
for $N$ landmarks with momenta $\mathbf p_i$ and positions $\mathbf q_i$, with $i=1,2,\dots,N$. 
A direct substitution of $u= K*m$ into the stochastic Euler-Poincar\'e equation \eqref{EP-sto} gives the stochastic landmark equations in \eqref{sto-Ham}. 
Here, $K$ is a given kernel corresponding to the Green's function of the differential operator $L$ used to construct the Lagrangian. 
Below, we take a different approach and proceed from a variational principle in which the stochastic landmark dynamics is constrained. 
We refer the interested reader to for example \cite{jacobs2014higher} for a detailed exposition of this derivation in the deterministic context. 

\subsection{Stochastic Landmarks dynamics}

Recall that for $N$ landmarks in $\mathbb R^d$, the diffeomorphism group elements $g$ act on the landmarks by evaluation of their position $g.\mathbf q= (g(q_1),\ldots, g(q_N))$, and the associated momentum map is \eqref{momap}.
The original action functional \eqref{E-functional} can be equivalently written as a constrained variational principle where the $\mathbf p_i$ play the role of Lagrange multipliers enforcing the stochastic reconstruction relation \eqref{rec-sto}. 
This procedure is based on the Clebsch action principle, which for landmark dynamics has been studied for one dimensional motion of landmarks on the real line in \cite{HoTy2016}
\begin{align}
    S(\mathbf u, \mathbf q,\mathbf p) = \iint l(\mathbf u)\, d\mathbf x\, dt + \sum_i \int  \mathbf p_i\cdot\left (  \circ d\mathbf q_i - \mathbf u( \mathbf q_i)\,  dt + \sum_l \sigma_l(\mathbf q_i)\circ dW_t^l\right )\,  .
    \label{S-functional}
\end{align}
Notice that only the Lagrangian depends on the spatial (Eulerian) variable $\mathbf x$ on the image domain. 
We now use the singular momentum map \eqref{momap} which provides us with the relation 
\begin{align*}
    2\,l(\mathbf u)= \int \mathbf m(\mathbf q,\mathbf p)(\mathbf x)\cdot  \mathbf u(\mathbf x) d\mathbf x = \sum_i \mathbf p_i \cdot \mathbf u(\mathbf q_i)\, .
\end{align*}
This relation reduces the action functional \eqref{S-functional} to the finite dimensional space of landmarks. 
We arrive at the action integral
\begin{align}
    S(\mathbf q,\mathbf p) &=
    \int h(\mathbf q,\mathbf p) \,dt + \sum_i \int \mathbf p_i\cdot\left (   \circ \,d\mathbf q_i+ \sum_{l}  \sigma_l(\mathbf q_i) \circ dW_t^l\right )\,  ,
    \label{S-ham}
\end{align}
where the Hamiltonian only depends on the landmark variables, as  
\begin{align}
    h(\mathbf q,\mathbf p) = \frac12 \sum_{i,j=1}^N (\mathbf p_i\cdot \mathbf p_j) K(\mathbf q_i-\mathbf q_j)\, .
    \label{h-landmarks}
\end{align}
The action integral in \eqref{S-ham} involves the phase space Lagrangian \eqref{h-landmarks} and the  stochastic potential, given by
\begin{align}
    \phi_l(\mathbf q,\mathbf p):= \sum_i \mathbf p_i\cdot \sigma_l(\mathbf q_i)\, . 
    \label{stoch-pot}
\end{align}
Taking free variations of \eqref{S-ham} gives the stochastic Hamilton equations in the form
\begin{align}
    \begin{split}
    d \mathbf q_i &= \frac{\partial h}{\partial \mathbf p_i} dt + \sum_l \frac{\partial \phi_l}{\partial \mathbf p_i} \circ d W_t^l\,,\\
    d \mathbf p_i &= -\frac{\partial h}{\partial \mathbf q_i} dt - \sum_l\frac{\partial \phi_l}{\partial \mathbf q_i} \circ  d W_t^l\, . 
    \end{split}
    \label{sto-Ham}
\end{align}
Explicitly, we have
\begin{align}
    \begin{split}
        d \mathbf q_i &= \sum_j \mathbf p_jK(\mathbf q_i-\mathbf q_j) dt + \sum_l\sigma_l(\mathbf q_i) \circ d W_t^l
        \,,\\
        d \mathbf p_i &= -\sum_j \mathbf p_i\cdot \mathbf p_j\frac{\partial }{\partial \mathbf q_i}K(\mathbf q_i-\mathbf q_j)\ dt - \sum_l \frac{\partial}{\partial \mathbf q_i}\left (\mathbf p_i \cdot\sigma_l(\mathbf q_i)\right )  \circ  d W_t^l \,.
    \end{split}
    \label{sto-Ham-explicit}
\end{align}
In coordinates, the stochastic equations \eqref{sto-Ham} become
\begin{align}
    \begin{split}
    d q_i^\alpha &= \frac{\partial h}{\partial p_i^\alpha} dt + \sum_l \sigma_l^\alpha(\mathbf q_i) \circ d W_t^l   \,, \\
    d p_i^\alpha &= -\frac{\partial h}{\partial q_i^\alpha} dt - 
        \sum_{l,\beta}\frac{\partial\sigma_l^\beta(\mathbf q_i)}{\partial q_i^\alpha } p_i^\beta \circ  d W_t^l \, ,
    \end{split}
    \label{sto-Ham-coor}
\end{align}
where $\alpha, \beta$ run through the domain directions, $\alpha,\beta=1,\ldots,d$. 

In order to have a unique strong solution of this stochastic differential equation, we need the drift and volatility to be Lipschitz functions with a linear growth condition after converting to It\^o form, and for the volatility to be uniformly bounded, see \cite{karatzas1991brownian}. 
This requirement is achieved when the functions $\sigma_l$ are twice continuously differentiable and uniformly bounded in the position variable. 
The latter property will hold with these functions being $C^2$ kernel functions.
The particular form of the stochastic potential in \eqref{stoch-pot} arises from the Legendre transformation of \eqref{S-functional}. The solutions of \eqref{sto-Ham-coor} represent the singular solutions of the stochastic EPDiff equation, corresponding to a stochastic path in the diffeomorphism group. 
In previous works such as \cite{TrVi2012,vialard2013extension,marsland2016langevin}, noise has been introduced additively and only in the momentum equation, corresponding to a stochastic force.
Also, the noise has typically been taken to be different for each landmark, and one can interpret it having been attached to each landmark. 
In the present case, the noise is not additive and the Wiener processes are not related to the landmarks, but to the domain of the image.  
Nearby landmarks will thus be affected by a similar noise, controlled by the spatial correlations of the noise. 
We refer to Figure~\ref{fig:comparison} in the Introduction for a numerical experiment demonstrating this effect. 

\begin{remark}[Geometric noise]
The geometric origin of the Hamiltonian stochastic equations in \eqref{sto-Ham} deserves a bit more explanation. 
    In the position equation \eqref{sto-Ham}, the noise arises as the infinitesimal transformation by the action of the stochastic vector field in \eqref{rec-sto}, namely $dg g^{-1}= u dt + \sum_l  \sigma_l \circ dW_l^t$, on the manifold of positions of the landmarks, which is generated by the $J$ stochastic potentials, $\Phi_l(\mathbf q_i,\mathbf p_i):= \mathbf p_i \cdot\sigma_l(\mathbf q_i) )$. Since this stochastic Hamiltonian is linear in the canonical momenta, the noise perturbing the evolution of the landmark positions is independent of the landmark momenta. On the other hand, the noise in the momentum equations arises as the \emph{cotangent lift} of the action of the stochastic vector field $dg g^{-1}$ on the positions of the landmarks. This cotangent lift determines the action on the momentum fibres attached to the perturbed position of each of the landmarks in phase space. The cotangent lift transformation is given explicitly by the product of the momentum and the gradient of the spatial fields $\sigma_l$ with respect to the position $\mathbf q_i$ of the $i$-th landmark. 
This difference increases the effect of the noise in regions where the $\sigma_l$ fields have large spatial gradients, provided the landmarks are moving rapidly enough for their momenta to be non-negligible. 
We will see in the example that in certain cases this balance in the product of the momentum and the spatial gradient of the noise parameters can significantly affect the dynamics of the landmarks. 
\end{remark}

\subsection{The Fokker-Planck Equation}\label{FP-section}

In this section, we study the evolution of the probability density function of the stochastic landmarks by using the Fokker-Planck equation. 
This study is possible in the case of landmarks because the associated phase space is finite dimensional. 

We will denote the probability density by $\mathbb P(\mathbf q,\mathbf p,t)$, on the phase space $\mathbb R^{2dN}$ at time $t$. 
The Fokker-Planck equation can be computed using standard procedures and is given in the following proposition. 

\begin{proposition}
    The Fokker-Planck equation associated to the stochastic process \eqref{sto-Ham} for the probability distribution $\mathbb P:\mathbb R^{2dN}\times \mathbb R\to \mathbb R$ is given by 
\begin{align}
    \frac{d}{dt}\mathbb P(\mathbf q,\mathbf p,t )  = \{h,\mathbb P\}_\mathrm{can} + \frac12 \sum_l \{\phi_l,\{\phi_l,\mathbb P\}_\mathrm{can}\}_\mathrm{can}:= \mathscr L^* \mathbb P\, ,
    \label{FP}
\end{align}
where $\{F,G\}_\mathrm{can} = \nabla F^T \mathbb J \nabla G$ is the canonical Poisson bracket with $\mathbb J=
    \begin{psmallmatrix}
        0 &1\\
        -1 &0 
    \end{psmallmatrix}$
    and $\phi_l(\mathbf q,\mathbf p)= \sum_i \mathbf p_i\cdot \sigma_l(\mathbf q_i)$ are the stochastic potentials.
    This formula also defines the forward Kolmogorov operator, $\mathscr L^*$. 
\end{proposition}

\begin{proof}
    The proof follows the standard derivation of the Fokker-Planck equation, by taking into account the geometrical structure of the stochastic process \eqref{sto-Ham}. 
  The time evolution of an arbitrary function $f:\mathbb R^{2dN}\to \mathbb R$ can be written as
    \begin{align*}
        df(\mathbf p,\mathbf q)= \{f,h\}_\mathrm{can} dt + \sum_l \{f,\phi_l\}_\mathrm{can}\circ dW_t^l\, ,
    \end{align*}
    as both drift and volatility have the same Hamiltonian form in the Stratonovich formulation. 
    We then compute the It\^o correction of this stochastic process, which is can be written as a double Poisson bracket form; namely,  $\frac12 \sum_l \{\{ f,\phi_l\}_\mathrm{can},\phi_l\}_\mathrm{can}dt$.
    The It\^o correction is the quadratic variation of the Stratonovich term in the stochastic differential equation, which equals the non-stochastic part of one half of the time derivative of the volatility (where a square Brownian motion becomes $dt$). 
    We refer to \cite{arnaudon2016noise, cruzeiro2017momentum} for a more detailed derivation of this formula in a general setting.
    Taking the expectation of the It\^o process then removes the noise term and defines the forward Kolmogorov operator such that $\dot f = \mathscr L f$.
    By pairing this formula with the density function $\mathbb P(\mathbf q,\mathbf p,t)$ over the phase space $(\mathbf q,\mathbf p)$ by using the usual $L^2$ pairing, as 
    \begin{align*}
        \int \mathbb P(\mathbf q ,\mathbf p,t) \mathscr Lf(\mathbf q ,\mathbf p) d\mathbf q d\mathbf p = \int \mathscr L^* \mathbb P(\mathbf q ,\mathbf p,t) f(\mathbf q ,\mathbf p) d\mathbf q d\mathbf p\, ,
    \end{align*}
    we obtain the Fokker-Planck equation $\dot{ \mathbb P}= \mathscr L^* \mathbb P$, which is explicitly given by \eqref{FP} as the double bracket term is self-adjoint and the advection term anti-self-adjoint. 
    Notice that here we have used a special property of the Poisson bracket; namely, that the Poisson bracket is also a symplectic $2$-form, which is exact and whose integral over the whole phase space vanishes, provided we choose suitable boundary conditions. 
    We again refer to \cite{arnaudon2016noise, cruzeiro2017momentum}  for more details about this derivation. 
\end{proof}

Of course, the direct study of this equation is not possible, even numerically, because of its high dimensionality. 
The main use here of the Fokker-Planck equation will be to understand the time evolution of uncertainties around each landmark. 
Indeed, for each landmark $\mathbf q_i$, the corresponding marginal distribution (integrating $\mathbb P$ over all the other variables) will represent the time evolution of the error on the mean trajectory of this landmark. 
We will show in the next section how to approximate the Fokker-Planck equation with a finite set of ordinary differential equations which describe the dynamics of the first moments of the distribution $\mathbb P$. 
This will then be used to estimate parameters of the noise fields $\sigma_l$ for given sets of initial and final landmarks. 

\begin{remark}[On ergodicity]
    The question of ergodicity of the process \eqref{sto-Ham} is not relevant here, as we will only consider this process for finite times, usually between $t=0$ and $t=1$. 
    The existence of stationary measures of the Fokker-Planck equation via H\"ormander's theorem is thus not needed. 
    Nevertheless, we will rely on a notion of reachability in the landmark position in the next section, where we will show how to sample diffusion bridges for landmarks with fixed initial and final positions. 
    This ensures that there exists a noise realisation which can bring any set of landmarks to any other set of landmarks. 
    This property is weaker than the H\"ormander condition and was introduced in \cite{sussmann1973orbits}.
\end{remark}

\subsection{Diffusion Bridges}\label{bridges-section} \label{sec:bridges}

The transition probability and solution to the Fokker-Planck equation $\mathbb P(\mathbf q,\mathbf p,t)$ can also be estimated by Monte Carlo sampling of diffusion bridges. This approach will, in particular, be natural for maximum likelihood estimation of parameters of landmark processes using the Expectation-Maximisation (EM) algorithm that will involve expectation over unobserved landmark trajectories, or for direct optimization of the data likelihood.
The EM estimation approach will be used in section~\ref{EM}. Here, we develop a theory of conditioned bridge processes for landmark dynamics which we will employ in the estimation. 
The approach is based on the method of \cite{delyon2006simulation} with two main modifications. 
The scheme and its modifications will be detailed after a short description of the general situation.
Alternative methods for simulating conditioned diffusion bridges can be found in e.g. \cite{papaspiliopoulos_importance_2012,bladt_simulation_2016,schauer_guided_2017}.

In \cite{delyon2006simulation}, a Girsanov formula \cite{girsanov_transforming_1960}, generalized to account for unbounded drifts, is used to show that when the diffusion field $\Sigma(\mathbf x,t)$ of an $\mathbb R^d$-valued diffusion process
\begin{align}
  d\mathbf x = b(\mathbf x,t)dt + \Sigma(\mathbf x,t)dW \ ,\  \mathbf x_0=\mathbf u
  \label{eq:orgdiff}
\end{align}
is uniformly invertible, the corresponding process
conditioned on hitting a point $\mathbf v\in\mathbb R^d$ at time $T>0$ is absolutely continuous with respect to an explicitly constructed unconditioned process $\hat{\mathbf x}$ that will hit
$\mathbf v$ at time $T$ a.s..  
The modified process $\hat{\mathbf x}$ is constructed by adding an additional drift term that forces the process towards the target $\mathbf v$.
In \cite{delyon2006simulation}, this process is constructed as a modification of \eqref{eq:orgdiff} 
\begin{align}
  d\hat{\mathbf x} = b(\hat{\mathbf x},t)dt - \frac{\hat{\mathbf x}-\mathbf v}{T-t} dt + \Sigma(\hat{\mathbf x},t)dW \, .
  \label{y-process}
\end{align}
Letting $P_{\mathbf x|\mathbf v}$ denote the law of $\mathbf x$ conditioned on hitting $\mathbf v$ with corresponding expectation $\mathbb E_{\mathbf x|\mathbf v}$,
the Cameron-Martin-Girsanov theorem implies that 
$P_{\mathbf x|\mathbf v}$ is absolutely continuous with respect to $P_{\hat{\mathbf x}}$, see for example \cite{oksendal_stochastic_2003} and the discussion in \cite{papaspiliopoulos_importance_2012}.
An explicit expression for the Radon-Nikodym derivative 
$dP_{\mathbf x|\mathbf v}/dP_{\hat{\mathbf x}}$ can be computed,
and this derivative is central for using simulations of the process
$\hat{\mathbf x}$ to compute expectations over the conditioned process $\mathbf x|\mathbf v$.
In particular, as shown in \cite{delyon2006simulation}, the conditioned process $\mathbf x|\mathbf v$ and the modified process $\hat{\mathbf x}$ are related by
\begin{align}
  \mathbb E_{\mathbf x|\mathbf v}(f(\mathbf x))
  =
  \frac{\mathbb E_{\hat{\mathbf x}}\Big(
  f(\hat{\mathbf x})
  \varphi(\hat{\mathbf x}) \Big)}{ \mathbb E_{\hat{\mathbf x}}(\varphi(\hat{\mathbf x})) }\, , 
  \label{DH-correction}
\end{align}
where $\varphi(\hat{\mathbf x})$ is a correction factor applied to each stochastic bridge $\hat{\mathbf x}$. 
Notice here that $f$ is a real-valued function of the stochastic path from $t=0$ to $t=T$.

Returning to landmark evolutions in the phase space $\mathbb R^{2dN}$, the process \eqref{sto-Ham} has two vector variables
$(\mathbf q,\mathbf p)$ that typically will be conditioned on hitting a fixed set of landmark positions $\mathbf v$ at time $T$. The conditioning thus happens only in the $\mathbf q$
variables by requiring $\mathbf q_T=\mathbf v$. 
To construct bridges with an approach similar to the scheme of
\cite{delyon2006simulation}, we need to find an appropriate extra drift term and handle the fact that the diffusion field may not be invertible in general.
Recall first that the landmark process \eqref{sto-Ham} has diffusion field
\begin{align}
  \Sigma(\mathbf q,\mathbf p)
  =
  \begin{pmatrix}
    \Sigma_{\mathbf q}(\mathbf q)\\
    \Sigma_{\mathbf p}(\mathbf q,\mathbf p)
  \end{pmatrix}
  :=
  \begin{pmatrix}
    \sigma_1(\mathbf q),
    \ldots,
    \sigma_J(\mathbf q) \\
    -\nabla_{\mathbf q}(\mathbf p\cdot \sigma_1(\mathbf q)),
    \ldots,
    -\nabla_{\mathbf q}(\mathbf p\cdot \sigma_J(\mathbf q))
  \end{pmatrix},
  \label{Sigma}
\end{align}
where $\sigma_j(\mathbf q)$ denotes the vector $(\sigma_j(q_1),\ldots,\sigma_j(q_N))^T$. 
Notice that this matrix is not square and has dimension $2dN\times J$ so that $\Sigma(\mathbf q,\mathbf p) \circ dW_t$ with $dW_t$ a $J$-vector corresponds to the stochastic term of \eqref{sto-Ham}.
Though $\Sigma(\mathbf q,\mathbf p)$ couples the $\mathbf q$ and $\mathbf p$ equation, when the number of noise fields $J$ is sufficiently large, the $\mathbf q$ part 
$\Sigma_{\mathbf q}(\mathbf q)$ will often be surjective as
a linear map $\mathbb R^J\rightarrow\mathbb R^{dN}$. 
In this situation, by letting $\Sigma_{\mathbf q}(\mathbf q)^\dagger$ denote the Moore-Penrose pseudo-inverse of
$\Sigma_{\mathbf q}(\mathbf q)$, we can construct a guiding drift term as
\begin{align}
    G(\mathbf q, \mathbf p): = -\frac{
        \Sigma(\mathbf q,\mathbf p)\Sigma_{\mathbf q}(\mathbf q)^\dagger
  (\mathbf q-\mathbf v)}{T-t}\, . 
  \label{guidance-drift}
\end{align}
This term, when added to the process \eqref{sto-Ham} and when measures are taken to control the unbounded drift of \eqref{sto-Ham}, will ensure that the modified process hits $\mathbf q_T$ a.s. at time $T$. 
The drift term \eqref{guidance-drift} is a direct generalization of the term added in \eqref{y-process}. 
If $\Sigma$ had been invertible then $\Sigma \Sigma^\dagger=\Id$ resulting in the guiding term  of \cite{delyon2006simulation} used in equation \eqref{y-process}. 
In the current non-invertible case, $\Sigma \Sigma_\mathbf{q}^\dagger(\mathbf q-\mathbf v)$ uses the difference $\mathbf q-\mathbf v$ which only involves the landmark position but affects both the position and the momentum equations.
We stress here the fact that the introduction of noise in
the $\mathbf q$ equation in \eqref{sto-Ham} is essential in our present approach.
When conditioning on the $\mathbf q$ variable, a guided process could not directly be constructed in this way, if the noise had been introduced only in the $\mathbf p$ equation, as in \cite{TrVi2012,vialard2013extension,marsland2016langevin}.
The fact that this term is weighted by $\Sigma \Sigma^\dagger$ is also important
as it allows the guiding term to be more efficient in the noisy regions of the
image, where there is more freedom to deviate from the deterministic path. The guiding term can be interpreted as originating from a time-rescaled gradient flow, and with the guiding term added, the diffusion process can be seen as a stochastically perturbed gradient flow, see \cite{arnaudon2016stochastic2}.

The guiding term \eqref{guidance-drift} is, in practice, not always appropriate for landmarks.
Because the correction is dependent only on the difference to the target in the position equation, a phenomenon of over-shooting is often observed. 
In such cases, the landmarks travel too fast initially due to a large momentum, strengthened by the guiding term that forces the landmarks towards $\mathbf v$. 
The high initial speed is only corrected when the time approaches
$T$ and the guiding term brings the landmark back to their final position. 
This effect is illustrated in Figure~\ref{fig:stoch_bridge_vis} in
section~\ref{sec:bridgeex} and results in low values of the correction factor $\varphi(\mathbf q,\mathbf p)$ used to compute the expectation in \eqref{DH-correction}. 
This effect tends to produce inefficient samples when approximating \eqref{DH-correction} by Monte Carlo sampling.
As an alternative, upon letting $b(\mathbf q,\mathbf p)$ denote the drift term of
\eqref{sto-Ham}, we employ a guided diffusion process of the form
\begin{align}
  \begin{pmatrix}
    d\hat{\mathbf q} \\
    d\hat{\mathbf p}
  \end{pmatrix}
  = b(\hat{\mathbf q},\hat{\mathbf p})dt 
    - \frac{\Sigma(\hat{\mathbf q},\hat{\mathbf p})\Sigma_{\mathbf q}(\hat{\mathbf q})^\dagger(\phi_{t,T}(\hat{\mathbf q},\hat{\mathbf p})-\mathbf v)}
  {T-t}dt 
  + \Sigma(\hat{\mathbf q},\hat{\mathbf p})
  \circ
  dW\, , 
  \label{eq:guiddiff}
\end{align}
for some appropriately chosen function $\phi_{t,T}:\mathbb R^{2dN}\rightarrow\mathbb R^{dN}$ that gives an estimate of the value of $\hat{\mathbf q}_T$ using the value of the modified stochastic process $(\hat{\mathbf q}_t,\hat{\mathbf p}_t)$ at time $t$. 
The hat denotes the solution of the process \eqref{eq:guiddiff}, which is different from the original dynamics of the process \eqref{sto-Ham} written without the hats.
The choice $\phi_{t,T}(\hat{\mathbf q},\hat{\mathbf p}):=\hat{\mathbf q}$ recovers the guiding term \eqref{guidance-drift}.
It would be natural to define $\phi_{t,T}(\hat{ \mathbf q},\hat {\mathbf p}):=\mathbb E_{( \mathbf q,\mathbf p)}(\mathbf q_T| (\mathbf q_t, \mathbf p_t)=(\hat{ \mathbf q},\hat{ \mathbf p}))$.
The resulting guiding term will only be driven by the expected amount needed at the endpoint, not from the value at time $t$. 
A similar choice but easier to handle is to let $\phi_{t,T}(\hat{\mathbf q},\hat{\mathbf p})$ be the solution at time $T$ of the original deterministic landmark dynamics \eqref{EP}, obtained from the initial conditions $(\hat{\mathbf q}_t,\hat{\mathbf p}_t)=(\hat{\mathbf q},\hat{\mathbf p})$. 
We will use this latter choice with a modification to ensure its time derivative is bounded. The approach is visualised in Figure~\ref{fig:stoch_bridge_vis}. 
To ensure convergence of $\hat{\mathbf q}_t\to\mathbf v$ for $t\rightarrow T$, a bounded approximation $\tilde{b}$ will be chosen to replace the original unbounded drift $b$ in \eqref{eq:guiddiff}. As it turns out, this choice has little influence in practice.

The matrix $\Sigma(\hat{\mathbf q},\hat{\mathbf p})\Sigma_{\mathbf q}(\hat{\mathbf q})^\dagger$
in \eqref{eq:guiddiff} only accounts for the $\mathbf q$ dynamics in the pseudo-inverse $\Sigma_{\mathbf q}(\hat{\mathbf q})^\dagger$. 
When the momentum is high and the noise fields $\sigma_j$ have high gradients, this fact can again lead to improbable
sample paths.
In such cases, the scheme can be further generalised by
using a guiding term of the form
\begin{align}
    \frac{1}{T-t}
  \Sigma(\hat{\mathbf q},\hat{\mathbf p})
  \Big(
  D_h\big(
  \phi_{t,T}(\Sigma(\hat{\mathbf q},\hat{\mathbf p})\mathbf h)
  \big)|_{\mathbf h=\mathbf 0}
    \Big)^{\dagger}(\phi_{t,T}(\hat{\mathbf q},\hat{\mathbf p})-\mathbf v)
  \, .
  \label{eq:differential}
\end{align}
The matrix 
$
  D_h\big(
  \phi_{t,T}(\Sigma(\hat{\mathbf q},\hat{\mathbf p})\mathbf h)
  \big)|_{\mathbf h=\mathbf 0}
$
is a linear approximation of the expected endpoint dynamics as a function of the
noise vector $\mathbf h\in\mathbb R^J$.
Again, with $\phi_{t,T}(\hat{\mathbf q},\hat{\mathbf p}):=\hat{\mathbf q}$, the
original guiding term \eqref{guidance-drift} is recovered,
and the term is close to the guiding term of \eqref{eq:guiddiff}
when the momentum or gradients of $\sigma_j$ are low.
We use this term for the experiments in section~\ref{sec:bridgeex} involving
high momentum dynamics, e.g. Figure~\ref{fig:sigma_single_mle}.

The following result is an extension of \cite[Theorem 5]{delyon2006simulation}
and \cite[Theorem 3]{marchand_conditioning_2011} 
to the modified guided SDE \eqref{eq:guiddiff}. 
It is the basis for the EM approach for estimating the parameters of the landmark processes developed in section~\ref{EM}.
Please note that the Girsanov theorem \cite[Thm. 1]{delyon2006simulation} which relates
the modified and original process, does not assume that $\Sigma$ is invertible.
The main analytic consequence of the non-invertibility is that the process is
semi-elliptic and the transition density, therefore, cannot be bounded by
Aronson's estimation \cite{aronson_bounds_1967}. 
Instead, we here assume
continuity and boundedness of the density of $\mathbf q$ in small intervals of $(0,T]$ 
in the sense of the assumption below. We write 
  $\mathbb P(\mathbf q_0, \mathbf p_0; \mathbf q, \mathbf p, t)$
  for the transition density at time $t$ of a solution 
$(\mathbf q,\mathbf p)$ to \eqref{sto-Ham}
  started at $(\mathbf q_0, \mathbf p_0)$. Similarly, when conditioning only on 
  $\mathbf q$, we write
  $\mathbb P(\mathbf q_0, \mathbf p_0;\mathbf q,t) = 
  \int_{\mathbb R^{dN}}
    \mathbb P(\mathbf q_0, \mathbf p_0; \mathbf q,\mathbf p,t)d\mathbf p$.

\begin{assumption}
\label{bridge-assumption}
  For any $(\mathbf q_0,\mathbf p_0)$ and $(\mathbf q,\mathbf p)$, the process $(\mathbf q_t,\mathbf p_t)$ has a density $\mathbb P(\mathbf q_0, \mathbf p_0; \mathbf q,\mathbf p,t)$ and the map $(\mathbf q,t)\mapsto\int_{\mathbb R^{dN}}g_0(\mathbf q_0,\mathbf p_0)\mathbb P(\mathbf q_0,\mathbf p_0; \mathbf q,t)d(\mathbf q_0,\mathbf p_0)$ is continuous in $t$ and $\mathbf q$ and bounded on sets $\{(\mathbf q,t)|s-\epsilon\le t\le s\}$ for $s\in (0,T]$, sufficiently small $\epsilon>0$, and any integrable function $g_0$.
\end{assumption}
The interpretation of Assumption \ref{bridge-assumption} is that, given any
distribution of initial conditions $(\mathbf q_0,\mathbf p_0)$ with density $g_0$, the
resulting $\mathbf q$-transition density of the process is continuous and bounded in $\mathbf q$ and $t$. As shown in Lemma~\ref{lem:density_lemma}, Assumption \ref{bridge-assumption} can be slightly weakened if Theorem~\ref{thm:endpoint} is only used to approximate the transition density at time $T$ as opposed to expectations $\mathbb E[f(\mathbf q,\mathbf p)|\mathbf q_T=\mathbf v]$ for arbitrary measurable functions $f$.

We let $W(\mathbb R^{2dN})$ denote the Wiener space of continuous paths
$[0,T]\rightarrow \mathbb R^{2dN}$.
\begin{theorem}
  Assume $\Sigma_{\mathbf q}(\mathbf q):\mathbb R^J\rightarrow\mathbb R^{dN}$ is surjective for all $\mathbf q$ with 
  $\Sigma_{\mathbf q}(\mathbf q)^\dagger$ bounded, and that $\Sigma$ is $C^{1,2}$, bounded, and with bounded derivatives.
  Let $\tilde{b}_\mathbf q$ be a bounded approximation of the $\mathbf q$-part of the drift $b$, and set $\tilde{b}=b+\Sigma(\mathbf q,\mathbf p)\Sigma_\mathbf q(\mathbf q)^\dagger(\tilde{b}_\mathbf q-b_\mathbf q)$.
  Let $v\in\mathbb R^{dN}$ be a point with $\mathbb P(\mathbf q_0,\mathbf p_0; \mathbf v,t)$ positive, and let $P_{(\mathbf q,\mathbf p)|\mathbf v}$ be the law of 
    $(\mathbf q,\mathbf p)\,|\,\mathbf q_T=\mathbf v$. Let 
  $(\hat{\mathbf q},\hat{\mathbf p})$ be solution to
  \eqref{eq:guiddiff}, $(\hat{\mathbf q}_0,\hat{\mathbf p}_0)=(\mathbf q_0,\mathbf p_0)$ with $\varphi_{t,T}:\mathbb R^{2dN}\rightarrow\mathbb R^{dN}$ a map with
  $\frac{\varphi_{t,T}(\mathbf q,\mathbf p)-\mathbf q}{T-t}$ bounded on $[0,T)$.
Then, for positive measurable $f:W(\mathbb R^{2dN})\rightarrow\mathbb R$,
\begin{align}
    \mathbb E_{(\mathbf q,\mathbf p)|\mathbf v}[f(\mathbf q,\mathbf p)]
    =
    \lim_{t\rightarrow T}
    \frac{
    \mathbb E_{(\hat{\mathbf q},\hat{\mathbf p})}\left [ f(\hat{\mathbf q},\hat{\mathbf p}) 
    \varphi(\hat{\mathbf q},\hat{\mathbf p},t) \right ]
  }{\mathbb E_{(\hat{\mathbf q},\hat{\mathbf p})}[\varphi(\hat{\mathbf q},\hat{\mathbf p},t)]}\, ,
    \label{Ebridges}
  \end{align}
  with
  \begin{align*}
    &\log\varphi({\mathbf q},{\mathbf p},t) =
    -\int_0^t\frac{({\mathbf q}-\mathbf v)^TA({\mathbf q})\tilde{b}({\mathbf q},{\mathbf p})ds }{T-s}
  \\
    &\qquad
    -\int_0^t\frac{
      ({\mathbf q}-\mathbf v)^T\big(dA({\mathbf q})\big)({\mathbf q}-\mathbf v)
    }{2(T-s)}
    -\sum_{i,j}\int_0^t\frac{d[A_{ij}({\mathbf q}),({\mathbf q}-\mathbf v)_i({\mathbf q}-\mathbf v)_j)]}{T-s}
  \\
    &\qquad
    +\int_0^t(b_\mathbf q({\mathbf q},{\mathbf p})-\tilde{b}_\mathbf q({\mathbf q},{\mathbf p}))^T\Sigma_\mathbf q({\mathbf q})^{\dagger,T}dW
    -\frac{1}{2}
    \int_0^t
    \|\Sigma_\mathbf q({\mathbf q})^{\dagger}
    (b_\mathbf q({\mathbf q},{\mathbf p})-\tilde{b}_\mathbf q({\mathbf q},{\mathbf p}))
    \|^2 ds 
    \\&\qquad
      +\int_0^t\frac{(\varphi_{t,T}(\mathbf q,\mathbf p)-\mathbf q)^T\Sigma_\mathbf q(\mathbf q)^{\dagger,T}dW}{T-t}
      -\frac12\int_0^t\left\|\frac{
    \Sigma_\mathbf q(\mathbf q)^{\dagger}
    (\varphi_{t,T}(\mathbf q,\mathbf p)-\mathbf q)}{T-t}\right\|^2ds
    \, ,
  \end{align*}
    where
      $A({\mathbf q})=\big(\Sigma_{\mathbf q}({\mathbf q})\Sigma_{\mathbf q}({\mathbf q})^T\big)^{-1}$.
  \label{thm:endpoint}
  In addition,
  \begin{equation}
    \mathbb P(\mathbf q_0,\mathbf p_0; \mathbf q,T)
    =
    \left(\frac{\left|A(\mathbf v)\right|}{2\pi T}\right)^{\frac d2}
    e^{-\frac{\|\Sigma_\mathbf q(\mathbf q_0)^\dagger(\mathbf q_0-v)\|^2}{2T}}
    \lim_{t\to T}
    \mathbb E_{(\hat{\mathbf q},\hat{\mathbf p})}
    [\varphi(\hat{\mathbf q},\hat{\mathbf p},t)]
    \ .
    \label{eq:density_lim_E}
  \end{equation}
\end{theorem}

In the Theorem, $[\cdot,\cdot]$ denotes the quadratic variation of semimartingales.
As mentioned above, a bounded approximation $\tilde{b}$ must be used to replace the original drift term $b$ in \eqref{eq:guiddiff}. The last integrals in the expression for $\log\varphi(\mathbf q,\mathbf p,t)$ are results of this approximation, and the use of the map $\varphi_{t,T}$.

The result is proved in Appendix~\ref{sec:bridgesamplingscheme}. If $\Sigma$ had been invertible and if the guidance scheme \eqref{y-process} was used, the result of \cite{delyon2006simulation} would imply that the right hand side limit of \eqref{Ebridges} would equal
\begin{align*}
    \frac{\mathbb E_{(\hat{\mathbf q},\hat{\mathbf p})}\left [ f(\hat{\mathbf q},\hat{\mathbf p}) 
    \varphi(\hat{\mathbf q},\hat{\mathbf p},T) \right ]
  }{\mathbb E_{(\hat{\mathbf q},\hat{\mathbf p})}[\varphi(\hat{\mathbf q},\hat{\mathbf p},T)]}\, . 
\end{align*}
Extending the convergence argument to the present non-invertible case is non-trivial, and we postpone investigating this to future work.
For numerical computations, $\varphi(\hat{\mathbf q},\hat{\mathbf p},t)$ can be approximated by finite differences. As described later in the paper, we do this using a framework that allows symbolic evaluation of gradients and thus subsequent optimization for parameters of the processes.

\section{Estimating the Spatial Correlation of the Noise} \label{estimate-correlation}

We now assume a set of $n$ observed landmark configurations $\mathbf q^1,\ldots,\mathbf q^n$ at time $T$, i.e. the observations are considered realisations of the stochastic process at some positive time $T$. 
From this data, we aim at inferring parameters of the model. This can be both parameters of the noise fields $\sigma_l$ and parameters for the initial configuration $(\mathbf q(0),\mathbf p(0))$. 
The initial configuration can be deterministic with fixed known or unknown parameters, or it can be randomly chosen from a distribution with known or unknown parameters.
We develop two different strategies for performing the inference. 
The first inference method in section~\ref{moment-method} is a shooting method based on solving the evolution of the first moments of the probability distribution of the landmark positions while the second method in section \ref{EM} is based on the Expectation-Maximisation (EM) algorithm.
The discussion is here in the context of landmarks, although these ideas may also apply in the more general context of section~\ref{theory}.

\subsection{The Noise Fields}

We start by discussing the form of the unknown $J$ noise fields $\sigma_l$. 
To estimate them from a finite amount of observed data, we are forced to require the fields to be specified by a finite number of parameters.
A possible choice for a family of noise fields is to select $J$ linearly independent elements $\sigma_1,\ldots,\sigma_J$ from a dense subset of $C^1(\mathbb R^d,\mathbb R^d)$. 
We here use a kernel $k$ with length-scale $r_l$ and a noise amplitude $\lambda_l\in \mathbb R^d$, that is
\begin{align}
    \sigma_l^\alpha (\mathbf q_i)
    = \lambda_l^\alpha  k_{r_l}(\|\mathbf q_i-\delta_l\|)\,,
    \label{kernels}
\end{align}
where $\delta_l$ denotes the kernel positions. 
Possible choices for the kernel include Gaussians 
$k_{r_l}(x)=e^{-x^2/(2r_l^2)}$, or cubic B-splines $k_{r_l}(x)=S_3(x/r_l)$.
The Gaussian kernel has the advantage of simplifying
calculations of the moment equations whereas the B-spline representation is compactly supported and gives a partition of unity when used in a regular grid.
Other interesting choices may include a cosine or a polynomial basis of the image domain.

In principle, the methods below allow all parameters of the noise fields to be estimated given sufficient amount of data. 
However, for simplicity, we will fix the length-scale and the position of the kernels.
The unknown parameters for the noise can then be specified in a single vector variable $\theta=(\lambda_1,\ldots,\lambda_K)$. 
The aim of the next sections will be to estimate this vector, possibly in addition to the initial configuration $(\mathbf q(0),\mathbf p(0))$, from data using the method of
moments in \ref{moment-method} and EM in \ref{EM}, respectively.

\begin{remark}
  For the bridge simulation scheme, we required $\Sigma_{\mathbf q}(\mathbf q)$ to be surjective as a linear map $\mathbb R^J\rightarrow\mathbb R^{dN}$. This assumption can be satisfied when the number of landmarks is low relative to the number of noise fields having non-zero support in the area where the landmarks reside. On the other hand, if the number of landmarks is increased while the number of noise fields is fixed, the assumption eventually cannot be satisfied. Intuitively, in such cases, the extra drift added to the bridge SDE must guide through a nonlinear submanifold of the phase space to ensure the landmarks will hit the target point $\mathbf v$ exactly. This limitation can be handled in three ways: (1) The method of moments as described below avoids matching individual point configurations, and it can, therefore, be used in situations where the surjectivity condition is not satisfied. (2) As discussed in remark \ref{rem:dimension_noise}, the noise can be made infinite dimensional. This can be done while keeping correlation structure similarly to the case with finite $J$. See also \cite{arnaudon2016stochastic2} for a discussion of noise in the form of a Gaussian process. (3) The bridge matching can be made inexact mimicking the inexact matching pursued in deterministic LDDMM. This could potentially relax the requirements on the extra drift term to only ensure convergence towards a given distance of the target. Inexact observations of stochastic processes are for example treated in \cite{van_der_meulen_continuous-discrete_2017}.
\end{remark}

\subsection{Method of moments}\label{moment-method}

We describe here our first method for estimating the parameters $\theta$ by solving a shooting problem on the space of first and second order moments. 
Given an estimate of the endpoint distributions $\mathbb P(\mathbf q,\mathbf p, T)$, we will solve the inverse problem which consists in using the Fokker-Planck equation \eqref{FP} to find the values of $\theta$ such that we can reproduce the observed final distribution.
Solving the Fokker-Planck equation directly is infeasible due to its high dimensionality. Instead, we will derive a set of finite dimensional equations approximating the solution of the Fokker-Planck equation \eqref{FP} for the probability distribution $\mathbb P$ in terms of its first moments.
This approach has been developed in the field of plasma physics for the Liouville equation, which is similar to the Fokker-Planck equation \eqref{FP}. 

\begin{remark}[Geometric moment equation]
    As the Fokker-Planck \eqref{FP} is written in term of the canonical bracket, we could expect to be able to apply a geometrical version of the method of moments such as the one developed by \cite{holm2007geometric}. 
Although this method seems to fit the present geometric derivation of the stochastic equations, we will not use it as it is not in our case practically useful.
Indeed, it requires the expansions of the Hamiltonian functions in term of the moments, but we cannot obtain here a valid expansion with a finite number of terms. 
This is due to the fact that the LDDMM kernel and the noise kernels cannot generally be globally approximated by finite polynomials with bounded approximation error for large distances.
This would, in turn, produce spurious strong interactions between distant landmarks.
\end{remark}

The method for approximating the Fokker-Planck that we will use here is the following.
We first define the moments
\begin{align}
    \braket{q_i^\alpha} &:= \int q_i^\alpha \mathbb P_\theta(\mathbf q,\mathbf p,t)\, d\mathbf q d\mathbf p \\
    \braket{q_i^\alpha p_j^\beta}  &:= \int q_i^\alpha p_j^ \beta  \mathbb P_\theta(\mathbf q,\mathbf p,t)\, d\mathbf q d\mathbf p\,,
\end{align}
where we have written only two possible moments, although any combinations of $p$ and $q$ at any order are possible.
In this work we will only consider moments up to the second order, that is the moments $\braket{q_i^\alpha},\braket{p_i^\alpha},\braket{q_i^\alpha q_j^\beta},\braket{q_i^\alpha p_j^\beta}$ and $\braket{p_i^\alpha p_j^\beta}$.
Notice that the first moment are $(1,1)$-tensors, and the second moments are $(2,2)$-tensors, although we will only use index notation here.

We illustrate this method with the first moment $\braket{q_i^\alpha}$, which represents the mean position of the landmarks.
We compute its time derivative and use the property of the Kolmogorov operator $\mathscr L$ defined in \eqref{FP} to obtain
\begin{align}
    \frac{d}{dt} \braket{q_i^\alpha} = \int q_i^\alpha \mathscr L^*\mathbb P_\theta\,  d\mathbf q d\mathbf p  = \int \mathscr Lq_i^\alpha \mathbb P_\theta\,   d\mathbf q d\mathbf p=\Braket{\mathscr Lq_i^\alpha}\, .
\end{align}
We thus first need to apply the Kolmogorov operator $\mathscr L $ to $q_i^\alpha$ to obtain
\begin{align}
    \begin{split}
    \mathscr Lq_i^\alpha  &= -\{h,q_i^\alpha\}_\mathrm{can} + \frac12\sum_l  \{\phi_l,\{\phi_l,q_i^\alpha\}_\mathrm{can}\}_\mathrm{can}\\
        &=   \frac{\partial h}{\partial p_i^\alpha} +\frac12\frac{\partial\sigma_l^\alpha(\mathbf q_i) }{\partial q_i^\beta} \sigma_l^\beta(\mathbf q_i) \, ,
    \end{split}
    \label{Lq}
\end{align}
which corresponds to the $q$ part of the drift of the stochastic process with It\^o correction.
Similarly, for the momentum evolution, we obtain
\begin{align}
    \begin{split}
        \mathscr L p_i^\alpha &=      -\{h,p_i^\alpha\}_\mathrm{can} + \frac12\sum_l  \{\phi_l,\{\phi_l,p_i^\alpha\}_\mathrm{can}\}_\mathrm{can}\\
        &= - \frac{\partial h}{\partial q_i^\alpha}  +\frac12 p_i^\gamma \frac{\partial\sigma_l^\gamma(\mathbf q_i)}{\partial q_i^\beta}\frac{\partial\sigma_l^\beta(\mathbf q_i)}{\partial q_i^\alpha} - \frac12 p_i^\beta \frac{\partial^2\sigma_l^\beta(\mathbf q_i) }{\partial q_i^\alpha \partial q_i^\gamma} \sigma_l^\gamma(\mathbf q_i)\, .
    \end{split}
    \label{lqq}
\end{align}
Most of the terms on the right hand side of \eqref{Lq} and \eqref{lqq} are nonlinear; so their expected value cannot be written in terms of only the first moments.
This is the usual closure problem of moment equations, such as the BBGKY problem arising in many-body problems in quantum mechanics. 
The solution to this problem is to truncate the hierarchy of moments for a given order and consider the system of ODEs as an approximation of the complete Fokker-Planck solution. 
Here we will apply the so-called cluster expansion method described in \cite{kira2011semiconductor}.
We refer to the Appendix~\ref{cluster} for more details about this method.

Apart from the first approximation $\braket{q_i^\alpha q_j^\beta}\to  \braket{q_i^\alpha}\braket{q_j^\beta}$, the next order of approximation is to keep track of the correlations
\begin{align}
    \Delta_2 \braket{q_i^\alpha q_j^\beta} := \braket{q_i^\alpha q_j^\beta}-\braket{q_i^\alpha}\braket{ q_j^\beta}\,.
    \label{D2}
\end{align}
This quantity is also called a centred second moment as for $i=j$ it corresponds to the covariance of the probability distribution for the landmark $i$. 
In general, it corresponds to the correlation between the positions of two landmarks.
The dynamical equation for this correlation is found from the equation of the second moment, which gives
\begin{align*}
    \frac{\partial }{\partial t} \Delta_2\braket{ q_i^\alpha q_j^\beta} &= \frac{\partial}{\partial t} \braket{ q_i^\alpha q_j^\beta } - \braket{ q_i^\alpha}\frac{\partial}{\partial t} \braket{ q_j^\beta } + T \\
        &=  \sum_l \Braket{\sigma_l^\alpha(\mathbf q_i) \sigma_l^\beta(\mathbf q_j)} + \Braket{q_i^\alpha\frac{\partial h}{\partial p_j^\beta}}-  \Braket{q_i^\alpha} \Braket{\frac{\partial h}{\partial p_j^\beta}}\\
        & +\frac12 \sum_l\left (  \Braket{ q_i^\alpha \sigma_l^\gamma(\mathbf q_j) \frac{\partial \sigma_l^\beta(\mathbf q_j)}{\partial q_j^\gamma}}- \Braket{q_i^\alpha}\Braket{\sigma_l^\gamma(\mathbf q_j) \frac{\partial \sigma_l^\beta(\mathbf q_j)}{\partial q_j^\gamma}}\right )  + (i\leftrightarrow j) \, ,
\end{align*}
where $(i\leftrightarrow j) $ stands for the same term but with $i$ and $j$ exchanged. 
This equation is interesting to study in more detail, as it already gives us information about the nature of the dynamics for the spatial covariance of landmarks.
Indeed, we have three types of terms with the following effects.
\begin{enumerate}
    \item {\it The $\sigma_l$-dependent terms.} This first term is quadratic in the $\sigma$'s, not proportional to any linear or quadratic polynomial in $q$ or $p$. 
        This term is a direct contribution from the noise in the $q$ equation and will have the effect of almost linearly increasing the centred covariance, wherever a $\sigma_l>0$.
    \item {\it The $h$-dependent terms.} From the form of this term, we expect it to be proportional to a correlation. 
        It will thus have an exponential effect on the dynamics, triggered by the linear contribution of the first term.
        Notice that this term only depends on the Hamiltonian, and, thus, on the interaction between landmarks. 
        If two landmarks interact, we expect their covariance to be averaged. 
        This term will capture their averaged covariance.
    \item {\it The $\nabla_q\sigma_l$-dependent terms.} These terms are related to the noise in the $p$ equation and will account for the effect on the landmark position of the interaction of the momentum of the landmark with the gradients of the noise.
\end{enumerate}
Notice that the last two types of terms describe second order effects with respect to the spatial covariance of the landmarks, as they depend linearly on the correlations.
In the expansion of these nonlinear terms, the other correlations involving $p$ will appear. This means that all of the possible second order correlations must be computed.
This computation is done in Appendix~\ref{moment-landmark}, where we also approximate the expected value of the kernels as $\braket{K(\mathbf q)}\approx K(\braket{\mathbf q})$.
As we will see in the numerical examples in section \ref{landmark-example}, these approximations can give a reliable estimate of the landmark covariance, but this should be rigorously justified to obtain a precise estimate of the errors.
Such a study is beyond the scope of this work and is left open.

Given the equations for the moment evolution, we can estimate the parameters $\theta$ by minimising the cost function
\begin{align}
    C(\braket{\mathbf p}(0), \lambda_l) = \frac{1}{\gamma_1} \left \|\braket{\mathbf q}- \braket{\mathbf q}(T)\right \|^2 + \frac{1}{\gamma_2}\left  \|\Delta_2 \braket{\mathbf {qq}}- \Delta_2\braket{\mathbf {qq} }(T)\right \|^2\, ,
    \label{cost-moment}
\end{align}
where $\gamma_1$ and $\gamma_2$ are weights. 
We denote by $\braket{\mathbf q}$ and $\Delta_2\braket{\mathbf {qq}}$ the target first and second moments and by $\braket{\mathbf q}(T)$ and $\Delta_2\braket{\mathbf{qq}}(T)$ the estimated moments which implicitly depend on the parameters of the noise and the initial momentum. 
The choice of the norm is free here, and we chose a norm which only considers $i=j$ and normalises each term to $1$ so that all the covariance of the landmarks contribute equally to the cost.
Other choices could be made, depending on applications.
Also, the cost function may depend on other parameters, but this would make its minimisation more difficult.

To minimise the cost \eqref{cost-moment}, we can use gradient based methods such as the BFGS algorithm.
Such methods require the evaluation of the Jacobian of $C$ with respect to all of its arguments. 
Usually, for the estimation of the initial momentum, a linear adjoint equation is used. However, the derivative with respect to the parameters of the noise cannot be computed in this way.
We will evaluate the gradients symbolically by using the Theano library in Python \cite{theano}.
To improve the efficiency of the algorithm, we first match the mean final position, by only updating the initial momentum. 
Then, with this initial condition, we match for both first and second moments and update the initial momentum as well as the parameters $\lambda_l$. 
As we will see in the numerical experiments in section \ref{landmark-example}, gradient-based methods are not optimal, and genetic algorithms, such as the differential evolution algorithm of \cite{storn1997differential} designed for global minimizations, turn out to perform better.

\subsection{Maximum Likelihood and Expectation-Maximization}\label{EM}

We now describe how to estimate the unknown parameters collected in the variable $\theta$ by a maximum likelihood estimation based on the expectation-maximisation (EM) algorithm of \cite{dempster_maximum_1977}.
The likelihood of $n$ independent observations 
$(\mathbf q^1,\ldots,\mathbf q^{n})$ at time $T$ given parameters
$\theta$ takes the form
\begin{align}
  \mathcal{L}(\mathbf q^1,\ldots,\mathbf q^{n},\theta) = \prod_{i=1}^{n}
  \mathbb P_\theta(\mathbf q^i,T) = \prod_{i=1}^{n}
  \int_{\mathbb R^{dN}}
    \mathbb P_\theta(\mathbf q^i,\mathbf p,T)
  d\mathbf p \, .
  \label{eq:likelihood}
\end{align}
The parameters $\theta$ can be estimated by maximizing the likelihood, that is by letting
\begin{align*}
  \hat{\theta} \in \mathrm{argmax}_\theta\,
  \mathcal{L}(\theta; \mathbf q^1,\ldots,\mathbf q^{n}) \, .
\end{align*}
For this, the likelihood could be directly computed by numerical approximation of 
$\mathbb P_\theta (\mathbf q_i,T)$ using an approximation of the Fokker-Planck equation \eqref{FP}.
Alternatively, the fact that the stochastic process is only sampled at time $T$ suggests a missing data approach that marginalises out the unobserved trajectories up to time $T$.
Let $(\mathbf q,\mathbf p; \theta)$ denote the stochastic landmark process with parameters $\theta$, and let $P(\mathbf q,\mathbf p; \theta)$ denote its law. 
Let $\mathcal L(\mathbf q,\mathbf p; \theta)$ denote the likelihood of the entire stochastic path for a given realisation of the noise, and computed with respect to the parameter $\theta$.
Notice that this likelihood is only defined for finite time discretizations of the process and there is no notion of path density for the infinite dimensional process.
We thus proceed formally, while noting that the approach
can be justified rigorously, see e.g. \cite{donnet_parametric_2008}.
An alternative approach is to optimize the likelihood 
\eqref{eq:likelihood} directly using \eqref{eq:density_lim_E}. This is pursued in \cite{sommer_bridge_2017}.

The EM algorithm finds a sequence of parameter estimates $\{\theta_k\}$ converging to a $\hat{\theta}$ by iterating over the following two steps: 
\begin{enumerate}
  \item {\it Expectation:} Compute the expected value of the log-likelihood 
    given the previous parameter estimate $\theta_{k-1}$:
    \begin{align}
      \begin{split}
        Q(\theta| \theta_{k-1})
        &:=
        \mathbb E_{\theta_{k-1}}
        [ \log \mathcal L(\mathbf q,\mathbf p;\theta) \,|\, \mathbf q^1,\ldots,\mathbf q^{n} ] \\
        &= \sum_{i=1}^{n} \mathbb E_{\theta_{k-1}}
          [ \log \mathcal L (\mathbf q,\mathbf p; \theta|\mathbf q^i) ] \, .
      \end{split}
      \label{eq:Estep}
    \end{align}
        The expectation \eqref{eq:Estep} over the process conditioned on the
observations $\mathbf q_i$ integrates the likelihood over
all sample paths reaching $\mathbf q_i$. 
For this, we employ the bridge simulation approach developed in
section~\ref{sec:bridges}. 
For each $\mathbf q^i$, we thus exchange $(\mathbf q_t,\mathbf p_t; \theta)$ with a guided process $(\hat{\mathbf q},\hat{\mathbf p}; \theta, \mathbf q^i)$ and use the equality \eqref{Ebridges} from Theorem~\ref{thm:endpoint}. 
The expectation on the right-hand side of \eqref{Ebridges} can be approximated by drawing samples from the guided process. 
Note that the correction factor $\varphi(\mathbf q,\mathbf p|\theta_{k-1}, \mathbf q_i)$ makes the approach equal
to importance sampling of the conditioned process with the guided process as proposal distribution.

    \item {\it Maximisation:} Find the new parameter estimate 
        \begin{align}
            \theta_k= \mathrm{argmax}_\theta\, Q(\theta|\theta_{k-1})\, . 
        \end{align}
The maximisation step can be approximated by updating $\theta_k$ such that it increases $Q(\theta| \theta_{k-1})$ instead of maximising it.
This is the approach of the generalised EM algorithm \cite{neal_view_1998}.
The update of $\theta$ is thus computed by taking a gradient step
\begin{align}
    \theta_k=\theta_{k-1}+\epsilon\nabla_\theta Q(\theta| \theta_{k-1})\, ,
\end{align}
where $\epsilon>0$. 
The gradient which is evaluated for each of the sampled paths can be computed
symbolically using the Theano library \cite{theano}. Theano allows
the entire computational chain from the definition of the Hamiltonian and noise fields to the time-discrete stochastic integration to be specified symbolically. The framework can therefore automatically derive gradients symbolically before the expressions are compiled to efficient numerical code. See also \cite{kuhnel_differential_2017} for more details on the use of Theano for differential geometric and stochastic computations.
\end{enumerate}

The resulting estimation algorithm is listed in Algorithm~\ref{alg:em}.
For each $\mathbf q^i$, the expectation 
$\mathbb E_{\theta_{k-1}}[\log \mathbb P_\theta(\mathbf q,\mathbf p|\mathbf q_i)]$ is estimated by sampling $N_{\mathrm{bridges}}$ bridges. 
The algorithm can perform a fixed number $K$ of updates to the estimate $\theta_k$ or stop at
convergence. 

\begin{algorithm} \DontPrintSemicolon \SetAlgoLined
\tcp{Initialization}
  $\theta_0\leftarrow$ initialization value
  \\
  \tcp{Main loop}
  \For{$k=1$ to $K$}{ 
    \For{$i=1$ to $n$}{ 
      \For{$j=1$ to $N_{\mathrm{bridges}}$}{
        sample bridge $(\hat{\mathbf q}(\omega_{j}),\hat{\mathbf p}(\omega_{j}); \theta_{k-1}, \mathbf q^i)$\\
    compute $\log \mathbb P_{\theta_k}(\hat{\mathbf q}(\omega_{j}),\hat{\mathbf p}(\omega_{j}))$ and 
      $\varphi(\hat{\mathbf q}(\omega_{j}),\hat{\mathbf p}(\omega_{j}))$\\
      }
      set $C_{\mathbf q_i}=\mathrm{mean}_j\big(\varphi(\hat{\mathbf q}(\omega_{j}),\hat{\mathbf p}(\omega_{j}))\big)$\\
    set $\mathbb E_{(\mathbf q,\mathbf p)|\mathbf q^i}[\log \mathbb P_{ \theta_{k-1}}(\mathbf q,\mathbf p)]\approx
      C_{\mathbf q_i}^{-1}\mathrm{mean}_j\big(
    \log \mathbb P_\theta(\hat{\mathbf q}(\omega_{j}),\hat{\mathbf p}(\omega_{j}))
      \varphi(\hat{\mathbf q}(\omega_{j}),\hat{\mathbf p}(\omega_{j}))
      \big)$\\
    }
    set $Q(\theta|\theta_{k-1})=\mathrm{mean}_i\big(
    \mathbb E_{(\mathbf q,\mathbf p|\mathbf q_i)}[\log
    \mathbb P_{\theta_{k-1}}(\mathbf q,\mathbf p)]
      \big)
      $\\
      compute $\nabla_{\theta}Q(\theta|\theta_{k-1})$\\
      update $\theta$: $\theta_k=\theta_{k-1}+\epsilon\nabla_{\theta}Q(\theta|\theta_{k-1})$\\
  }
    \caption{\small Stochastic EM-estimation of parameters $\theta$.}
  \label{alg:em}
\end{algorithm}

\section{Numerical examples}\label{landmark-example}

We now present several numerical tests of the stochastic perturbation of the landmark dynamics. 
In particular, we want to illustrate aspects of the effect of the noise on the
landmarks and test the algorithms for estimation of the spatial correlation of the noise. 
We will focus here on synthetic examples and refer to
\cite{arnaudon2016stochastic2} for an application of the methods on a dataset of
Corpora Callosa shapes represented by 77 landmarks.
The numerical simulations of this work have been done in Python, using the
symbolic computation framework Theano \cite{theano}.
The code is available from the public repository
\url{https://bitbucket.org/stefansommer/stochlandyn}.
See also \cite{kuhnel_differential_2017} for additional details.

\subsection{Solution of the Fokker-Planck equation}

We first consider a simple experiment with a single landmark, subjected to a square
array of noise fields with Gaussian noise kernel. 
To a first order approximation, the mean trajectory of the landmark is a straight line with constant momenta as the Hamiltonian is a pure kinetic energy. 

\begin{figure}[htpb]
    \centering
    \subfigure[Moment dynamics]{\includegraphics[scale=0.45]{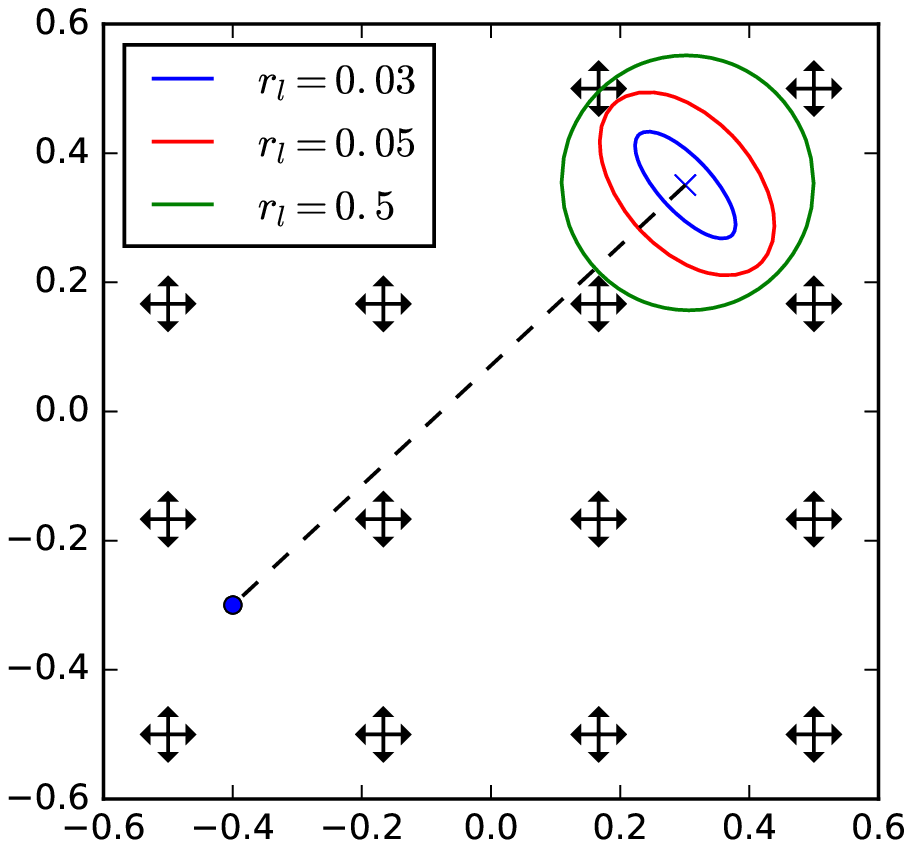} 
    \label{fig:moment-single} }
    \subfigure[$r_l=0.5$]{\includegraphics[scale=0.45]{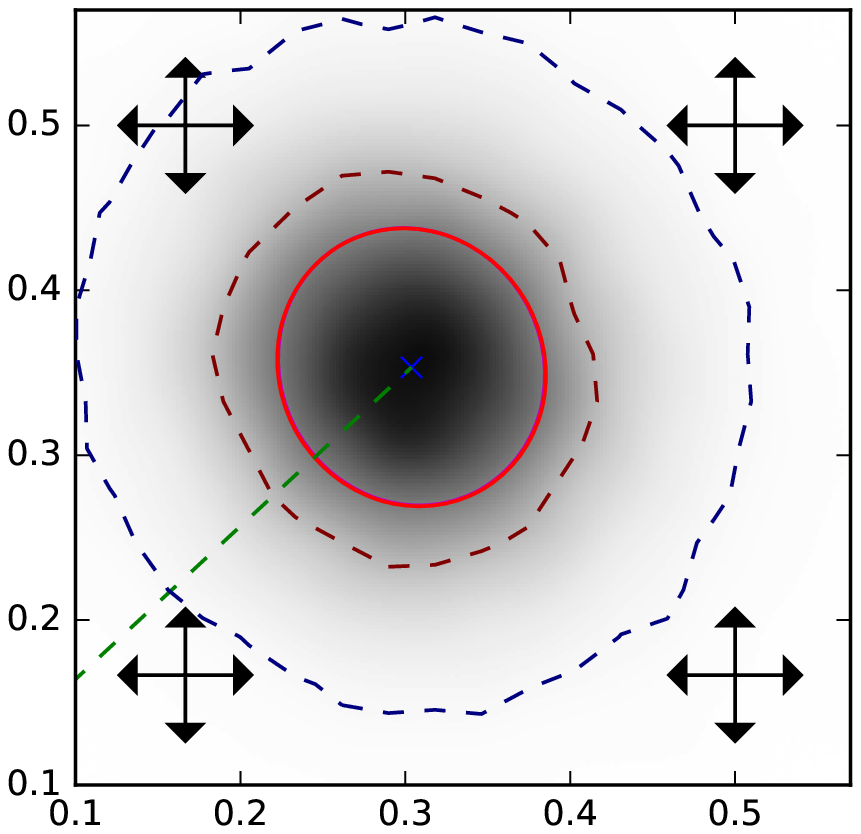}\label{fig:moment-single-samplea}}
    \subfigure[$r_l=0.03$]{\includegraphics[scale=0.45]{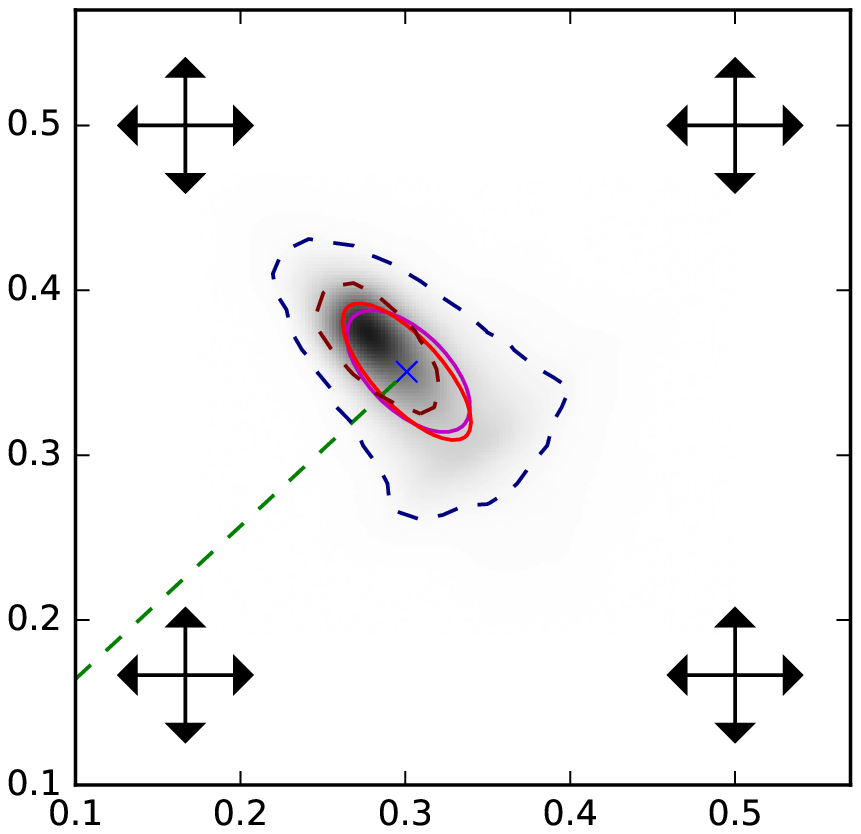}
    \label{fig:moment-single-sampleb}
    }
    \caption{\small 
    We plot on the left panel three simulations of a single landmark dynamics subject to an array of Gaussian noise fields. 
    Their parameters are either $\lambda_l= (0.08,0)$ or $\lambda_l= (0,0.08)$. 
    We used three different length-scales $r_l$ for the noise fields to analyse the effects of small or large Gaussian fields $\sigma_l$ on the mean path of the landmark (with Gaussian kernels) and final covariance (ellipses). 
    We used $2000$ timesteps to integrate the moment equation forwards from $t=0$ to $t=1$. 
    The initial momenta were found using a shooting method in the deterministic landmark equation.
    We display on the other two panels a zoom on two of the simulations of the left panel and compare the estimation of the final covariance from a Monte Carlo sampling of $10,000$ realisations (magenta) and from the solution of the moment equation (red) for two values of $r_l$. 
    The black density represents the probability distribution of the landmark
    estimated from samples, and the dashed lines two level sets.}
\end{figure}

This experiment is displayed in Figure~\ref{fig:moment-single} where we used two arrays of four by four noise fields with either $\lambda_l= (0.08,0)$ or $\lambda_l= (0,0.08)$ and three values of the noise radius $r_l=0.5,0.05,0.03$. 
For large values of $r_l$, the noise is mostly uniform and the gradients of the $\sigma_l$ are negligible.
 The only term contributing to the final covariance of the landmark is therefore $\Braket{\sigma_l^\alpha(\mathbf q_i) \sigma_l^\beta (\mathbf q_j)}$.  
Notice that because there is only one landmark, thus a linear drift, the deterministic part does not affect the covariance.
This term only has a linear effect on the covariance which is thus an ellipse proportional to the noise fields. 
Here the noise has equal strength in both the $x$ and $y$ coordinate thus we observe a circle. 
For smaller values of $r_l$, the gradient of $\sigma_l$ is large enough for the other term in the momentum equation which couples the momentum and the gradient of $\sigma$ to affect the moment dynamics. 
This effect is shown in Figure~\ref{fig:moment-single} where the covariance has a larger value in the direction of the gradient of $\sigma_l$ than in the other directions.
This is explained by the fact that this coupling is of the form $\frac{\partial}{\partial q_i}(\sigma_l(\mathbf q_i)\cdot \mathbf p_i)$, thus the ellipse is in the direction of the gradient, not the momenta. 
Notice that there should be some noise in the direction of the momenta for this term to have an effect. 

Using the same experiment, we compared the estimation of the covariance from the moment equation with a direct sampling obtained by solving the stochastic landmarks equations. 
We did this experiment for $r_l= 0.5$ and $r_l=0.03$ in Figure~\ref{fig:moment-single-samplea},\ref{fig:moment-single-sampleb}. 
The left panel with $r_l=0.5$ shows an excellent agreement between the two
methods but the right panel with $r_l=0.03$ shows differences.
This type of error in the estimation of the covariance is explained by the fact that the final distribution has a large skewness. 
This effect is not captured by the moment equations as we neglected the effects of order higher than $2$, and the skewness is a third order effect described by terms such as $\Delta_3\braket{q_i^\alpha q_j^\beta q_k^\gamma}$.
Nevertheless, the final covariance is close enough to the correct one to be able to use it in the estimation of the noise fields. 
This demonstrates that even in rather extreme cases, which are not realistic for applications, the second order approximation used to derive the moment equation still produces reliable results. 

\begin{figure}[htpb]
    \centering
    \subfigure[Landmark interaction]{\includegraphics[scale=0.6]{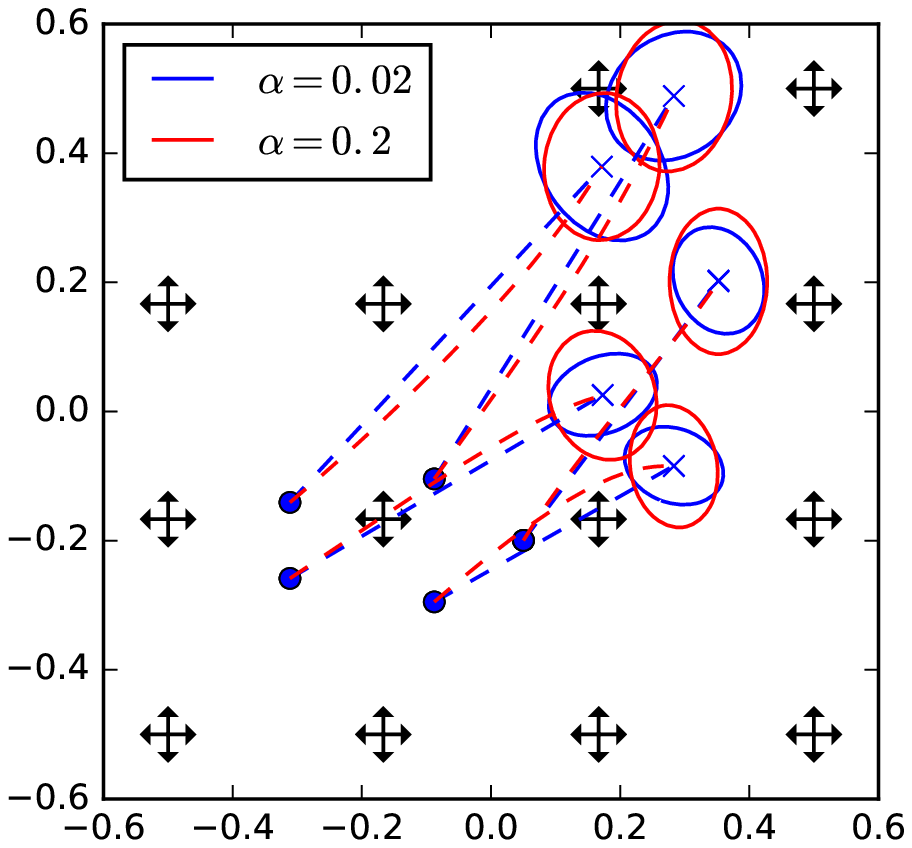}}
    \subfigure[Sampling comparison]{\includegraphics[scale=0.6]{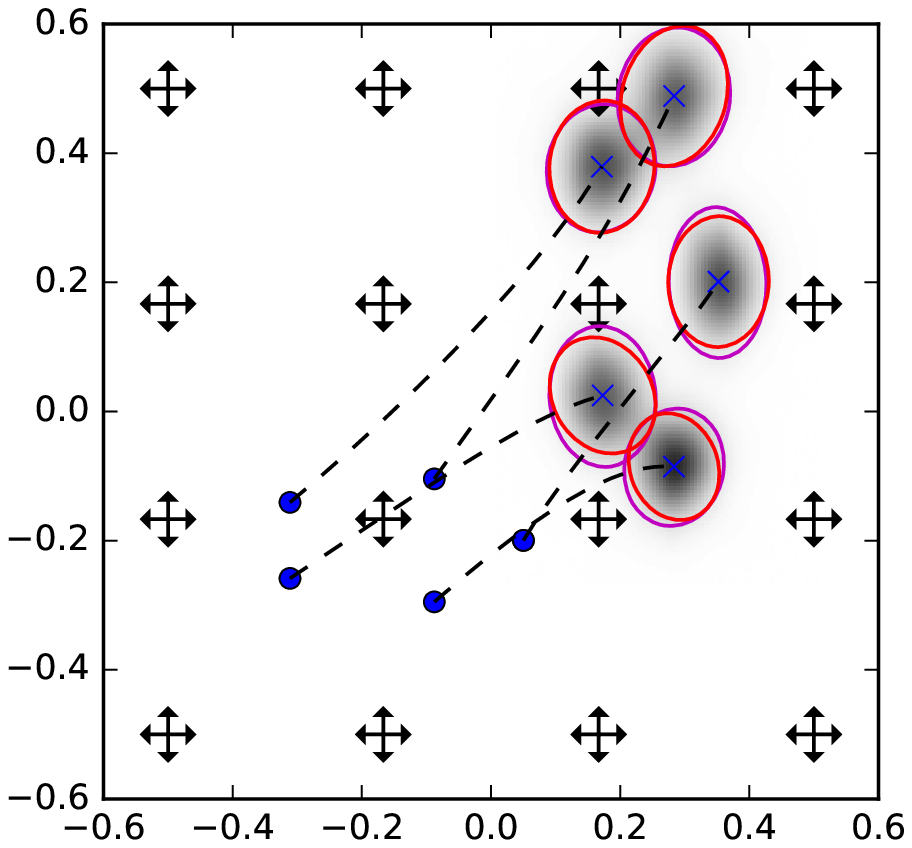}}
    \caption{\small In these two panels, we present a study similar to Figure~\ref{fig:moment-single} but with $5$ interacting landmarks. 
    On the left, we illustrate the effect of varying the landmark length scale
    $\alpha$, and, on the right, we compare the result of the moment equation and a Monte Carlo simulation in the case of $\alpha=0.1$, with also $10,000$ realisations. 
    As before, the black density plot shows the probability density of the landmarks, the magenta curve the covariance from sampling and the red curve the covariance from the moment equation. }
    \label{fig:moment-ellipse}
\end{figure}

We did a similar experiment but with $5$ interacting landmarks arranged in an
ellipse configuration and with initial conditions obtained from the
deterministic shooting method such that the endpoint of the deterministic landmark equations match another ellipse.
We display these experiments in Figure~\ref{fig:moment-ellipse} with the same noise as in the previous tests and with $r_l = 0.2$. 
We modified here the landmark interaction length scale $\alpha$ from $\alpha=0.02$ (no-interactions) to $\alpha=0.2$ (neighbours interactions) to see the effect of the noise with the landmark interactions.
Due to the different length scales, the trajectories to the target ellipse are slightly different so the landmarks will be subject to different noise. 
The larger length scale has the effect of reducing the differences between the covariances of interacting landmarks. 

\subsection{Bridge Sampling} \label{sec:bridgeex}

Here, we aim at visualising the effect of the constructed bridge sampling scheme.
In Figure~\ref{fig:stoch_bridge_vis}, the effect of the guiding term is
visualized on a sample path. At $t=T/2$, the predicted endpoint
$\phi_{t,T}(\hat{\mathbf q}(t),\hat{\mathbf p}(t))$ is calculated and the
difference $\phi_{t,T}(\hat{\mathbf q},\hat{\mathbf p})-\mathbf v$ is used to
guide the evolution of the path towards the target $\mathbf v$. The guiding term ensures
that $\hat{\mathbf q}$ will hit $\mathbf v$ almost surely at time $T$. Notice
that the difference $\phi_{t,T}(\hat{\mathbf q},\hat{\mathbf p})-\mathbf v$ is
generally much smaller than the difference $\hat{\mathbf q}-\mathbf v$. The introduction of
$\phi_{t,T}$ therefore implies that the process is modified less giving more
likely bridges. Without $\phi_{t,T}$, the process is generally attracted too
quickly towards the target as can be seen by the landmarks at $t=0.5$ being
almost at their final positions in Figure~\ref{fig:stoch_bridge_vis} (b). The
path thus overshoots the target. This effect is not present when using
$\phi_{t,T}$ in Figure~\ref{fig:stoch_bridge_vis} (a).
\begin{figure}[htpb]
    \centering
    \subfigure[guided process using $\phi_{t,T}$]{
    \includegraphics[width=.46\columnwidth,trim=130 90 130 120,clip=true]{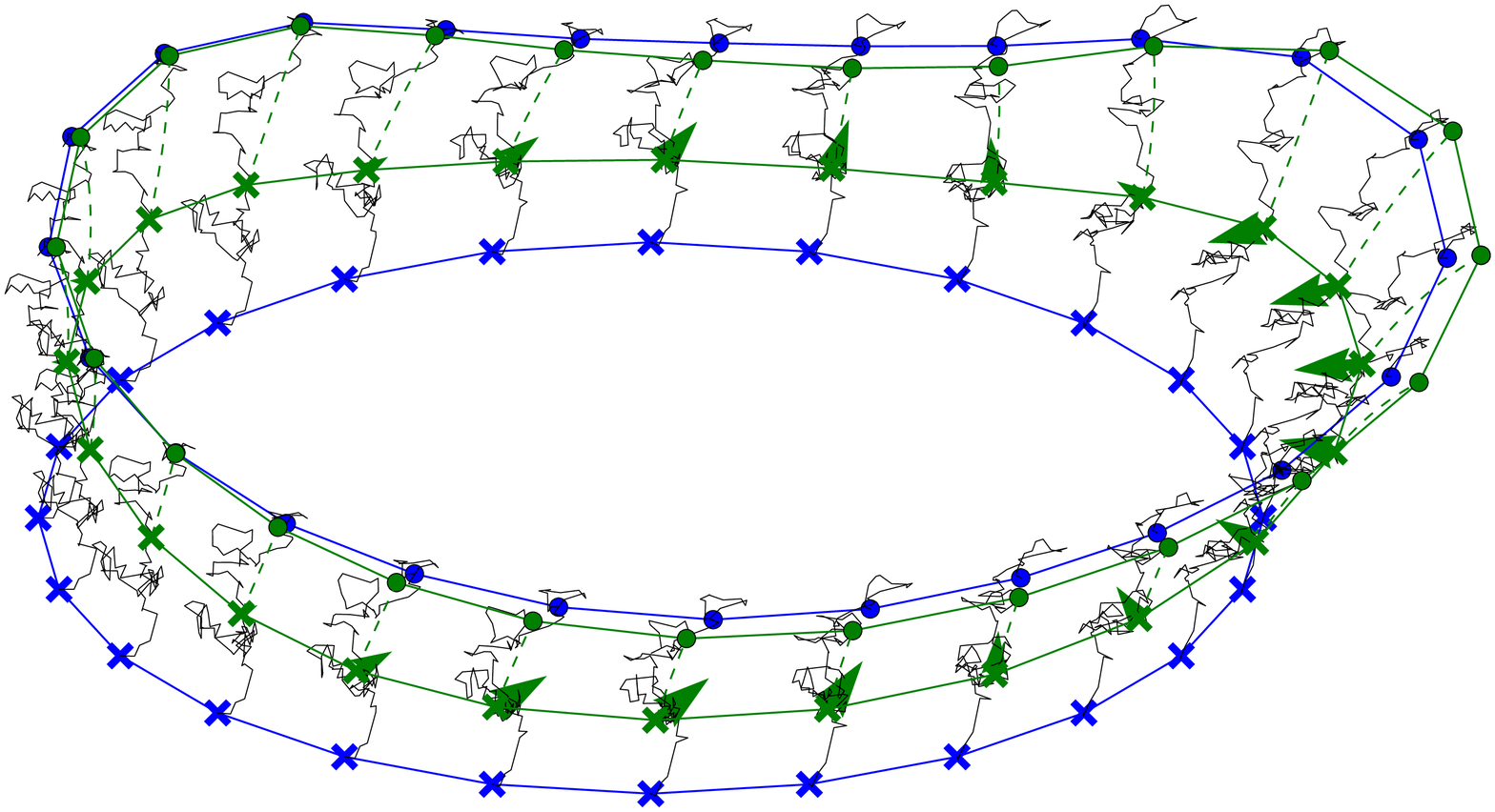}
  }
    \subfigure[guided process without $\phi_{t,T}$]{
    \includegraphics[width=.46\columnwidth,trim=130 90 130 120,clip=true]{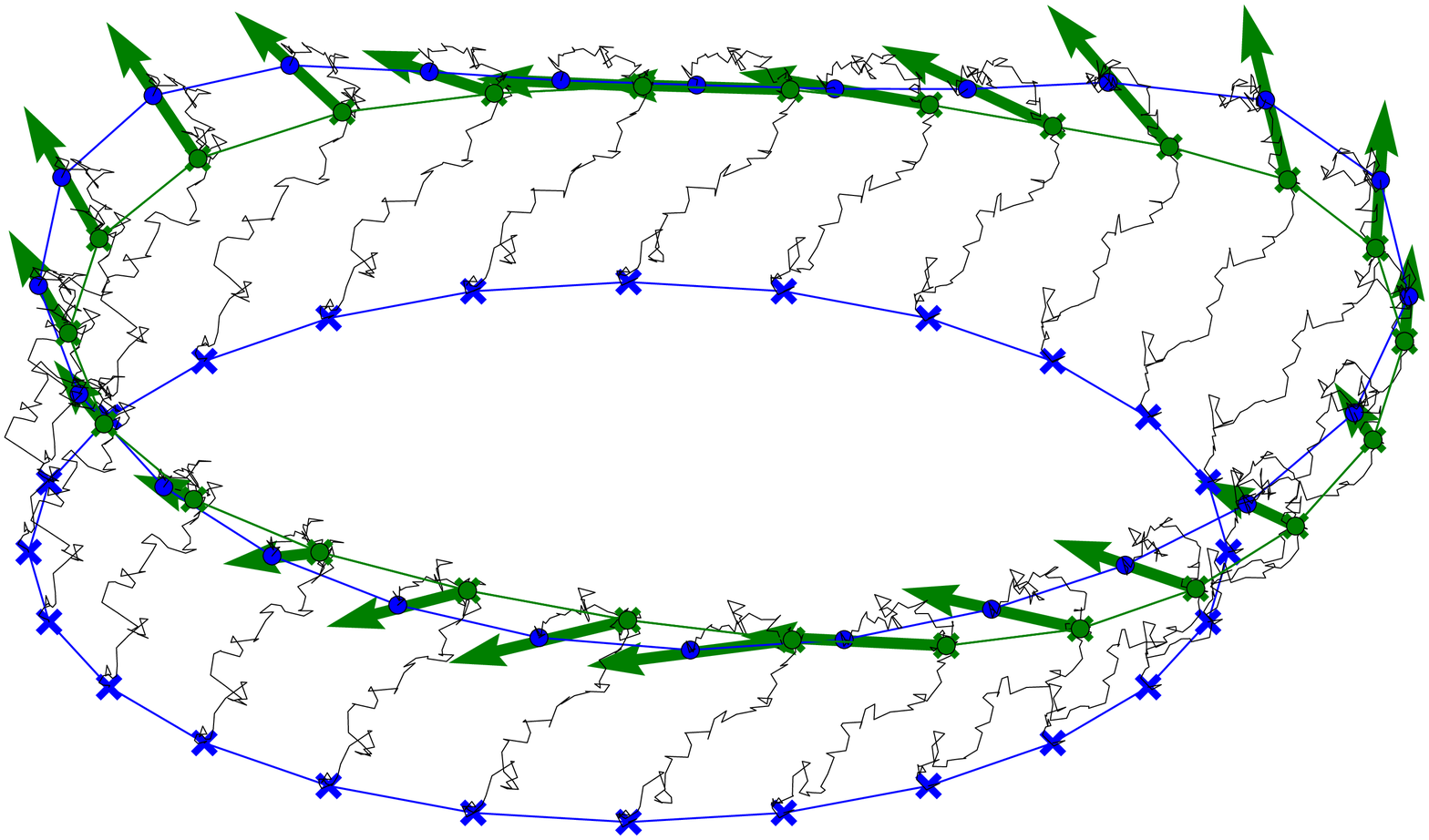}
  }
    \caption{\small
    (a) Visualization of the process \eqref{eq:guiddiff}. From the
    initial landmark configuration $\mathbf q(0)$ (blue crosses), the target
    $\mathbf v$ (blue dots) is hit using the modified process 
    $(\hat{\mathbf q},\hat{\mathbf p})$ (black lines: $\hat{\mathbf q})$. At time $t=T/2$, 
    $\phi_{t,T}(\hat{\mathbf q}(t),\hat{\mathbf p}(t))$ is calculated 
    (green dots) and the process is guided by
    $-(T-t)^{-1}\Sigma\Sigma_{\mathbf q}^{\dagger}(\phi_{t,T}(\hat{\mathbf q},\hat{\mathbf p})-\mathbf v)$ ($\mathbf q$ part: green arrows, length doubled for visualization).
    The use of $\phi_{t, T}$ implies small guiding and high probability sample bridges.
    (b) Similar setup but using the guiding term 
    $-(T-t)^{-1}\Sigma\Sigma_{\mathbf q}^\dagger(\hat{\mathbf q}-\mathbf v)$ 
    without $\phi_{t,T}$. The momentum couples with the guiding term, and, intuitively, the path travels too fast towards the target ($\mathbf q$ at $t=T/2$ much closer than halfway towards $\mathbf v$) and overshoots.
    This effect gives low probability sample bridges and the guiding term (green arrows) is much larger than in (a).
  }
    \label{fig:stoch_bridge_vis}
\end{figure}

\subsection{Estimating the noise amplitudes}

We here aim at estimating the noise amplitude from sampled data using both the
method of moments and maximum likelihood.

\begin{figure}[htpb]
    \centering
    \subfigure[Genetic algorithm]{\includegraphics[scale=0.6]{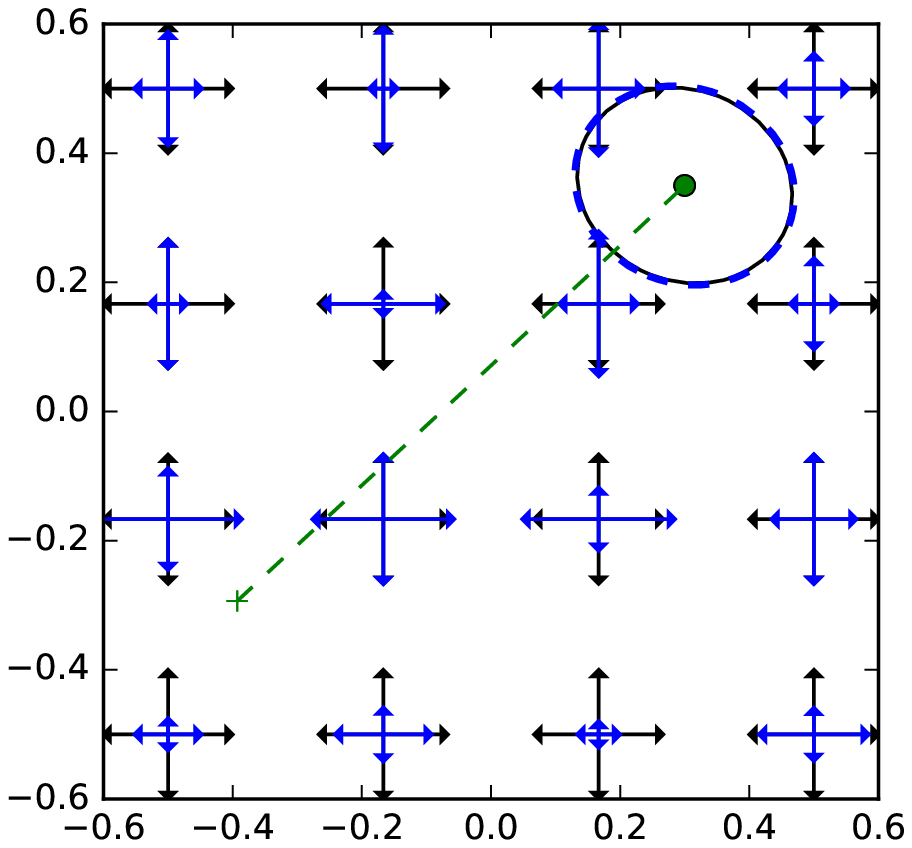}}
    \subfigure[Gradient descent]{\includegraphics[scale=0.6]{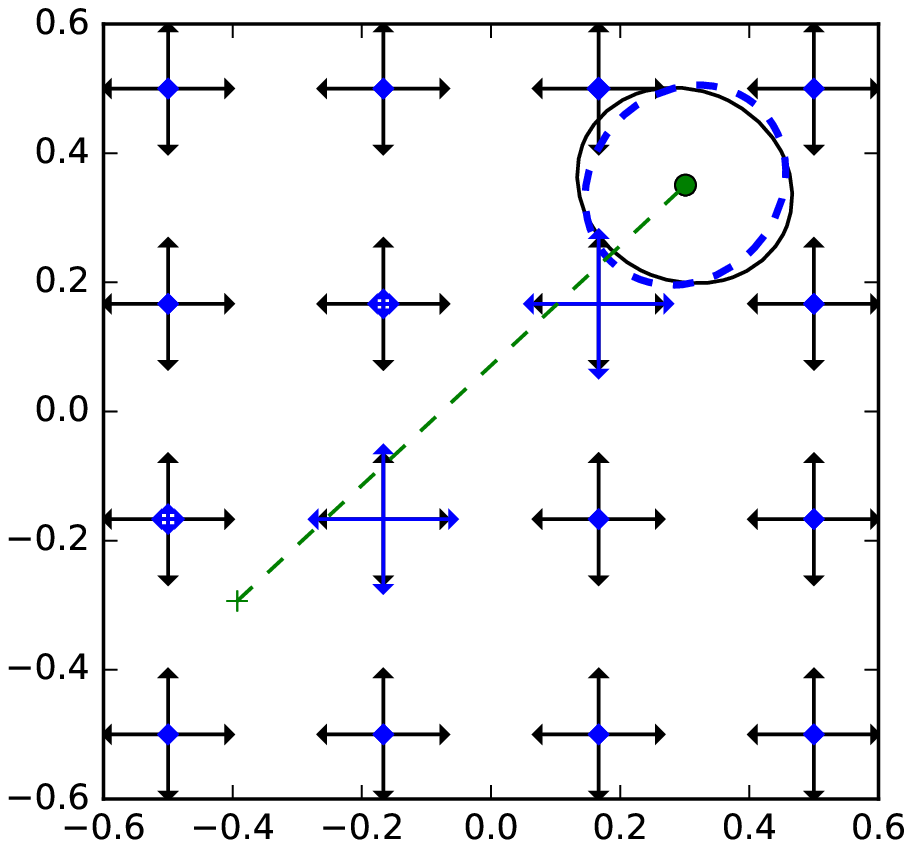}}
    \caption{\small 
        This figure shows the results of estimating parameters of the $\sigma_l$ fields with the moment equation.
    Black arrows: The original $\sigma_l$. Blue arrows: The estimated $\sigma_l$.
        The error in the final covariance for the differential equation genetic algorithm is of the order of $10^{-10}$ and for the BFGS algorithm, it is of the order of $5\cdot 10^{-2}$. 
  }
  \label{fig:sigma_single}
\end{figure}

We first use the genetic algorithm of \cite{storn1997differential} called differential evolution algorithm  to minimise the cost function $C$ in \eqref{cost-moment}. 
This algorithm has in experiments proven successful in avoiding local minima
during the optimisation.
We compared it with the standard BFGS gradient descent algorithm with a single landmark in
Figure~\ref{fig:sigma_single}.
This algorithm relies on the Jacobian of the cost functional computed symbolically using the Theano package of \cite{theano}. 
It is able to estimate the noise amplitude along the trajectory of the landmark where the signal from the gradient of $C$ is the strongest. 
For the other regions of the image, the algorithm cannot detect any signal to update the noise fields. 
The genetic algorithm can overcome this issue as it is based on evolving a population of solutions which uniformly cover the entire parameter space. 
In this way, the solution obtained is a better approximation of the global minimum of $C$. 
It is interesting that even if the final moment of figure \ref{fig:sigma_single} is well matched with the genetic algorithm, the noise amplitude is not perfectly recovered. 
This illustrates the expected degeneracy of this model for a low number of landmarks. 
When more landmarks are added, the noise amplitude estimation is closer to the expected one, see figure \ref{fig:sigma_ellipse} below.
In these experiments, we set the initial variance of the momentum, and the position/momentum correlation to $0$ for simplicity, and because we used these values to generate the final variance. 
In practice, one may expect to have other prior for the initial variance of the momenta or can try to find it as an unknown parameter of the problem. 
Having them as unknown may results in a large parameter space, thus simplifications such as all the landmarks have the same initial variance in the momentum could be used. 
We leave such investigation for later when applied to real data, with a possible meaningful prior.
\begin{figure}[htpb]
    \centering
    \subfigure[Landmark and estmated noise, low momentum]{\includegraphics[width=.4\columnwidth]{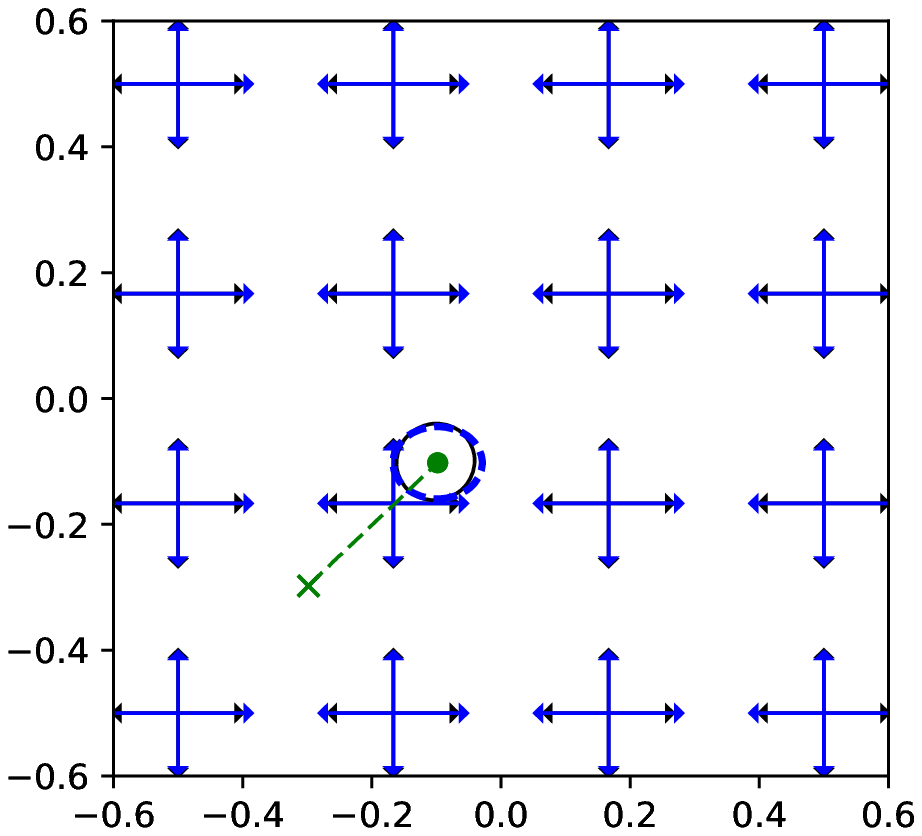}} 
    \subfigure[Landmark and estmated noise, high momentum]{\includegraphics[width=.4\columnwidth]{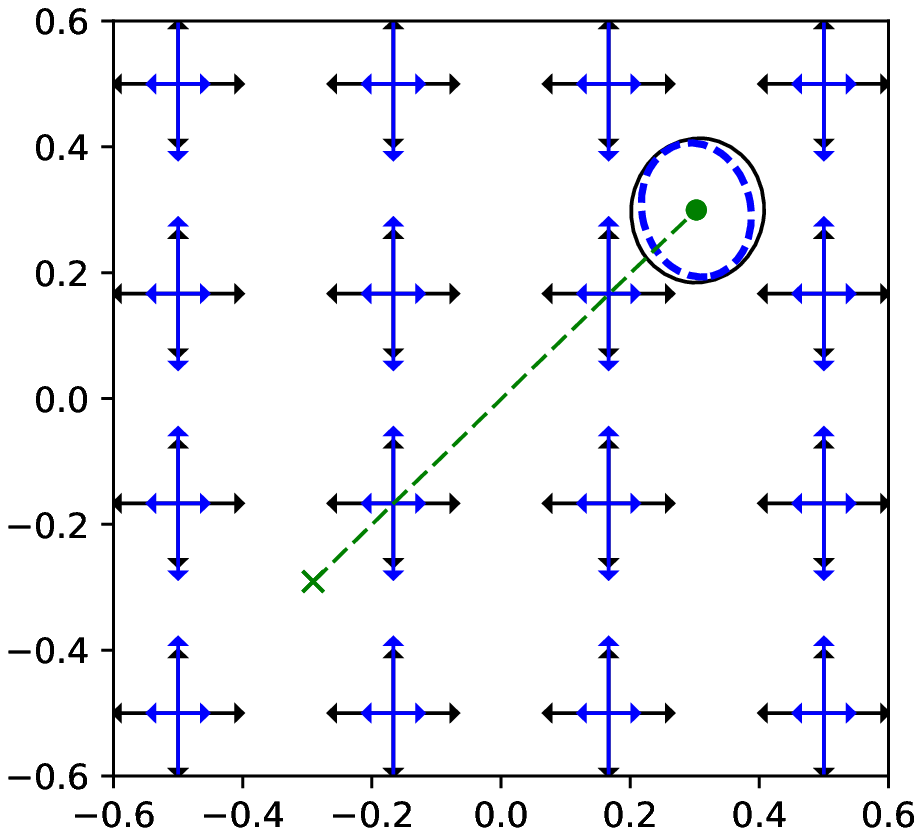}} 
    \caption{\small
The noise amplitudes are here estimated using maximum likelihood with the bridge sampling scheme. We assume $\lambda_l$ are equal for all $l$ resulting in two parameters for the noise. Thus by assumption, the estimated noise will be uniform over the domain. (a) The parameters are estimated correctly in the low momentum setting. (b) While the sample covariance matches the covariance of the original data in the high momentum case, the estimated parameters are different from the original.
  }
  \label{fig:sigma_single_mle}
\end{figure}

In Figure~\ref{fig:sigma_single_mle}, the
same experiment is performed with MLE and the bridge sampling scheme. 
The noise kernels are in this experiment cubic B-splines placed in a
grid providing a partition of unity. In the optimisation, $\lambda_l$ are fixed
to be equal for all $l=1,\ldots, J$ implying that the total noise variance will be uniform at each point of the domain. The figure shows the experiment performed with low momentum (Figure~\ref{fig:sigma_single_mle} (a)) and high momentum 
(Figure~\ref{fig:sigma_single_mle} (b)). 
In the low momentum case, the noise parameters are estimated correctly and the sample covariance with the estimated parameters matches the covariance of the original samples. 
The SDE \eqref{eq:guiddiff} is here used for the bridge sampling scheme. In contrast to the previous
method, the algorithm is now optimising for the maximum likelihood of the
samples and not directly for matching the final covariance. A higher difference in the endpoint
covariance is, therefore, to be expected.

With higher initial momentum, the coupling
between the guidance and noise makes the scheme \eqref{eq:guiddiff} overestimate
the variance. Instead, the guidance term \eqref{eq:differential} is used. Notice
that even though the sample covariance with the estimated parameters matches the
covariance of the original samples, the estimated $\lambda_l$ are different than the
original values. This indicates that the maximum likelihood estimate of the
parameters may not match the original setting in the highly nonlinear case
occurring when the coupling between noise and momentum is high. Because of the
nonlinearity, the noise is able
to generate horizontal variation in the position of the final the landmark even though 
the variation with the estimated
parameters are mainly vertical along the trajectory.

\begin{figure}[htpb]
    \centering
      \subfigure[Genetic algorithm]{\includegraphics[scale=0.6]{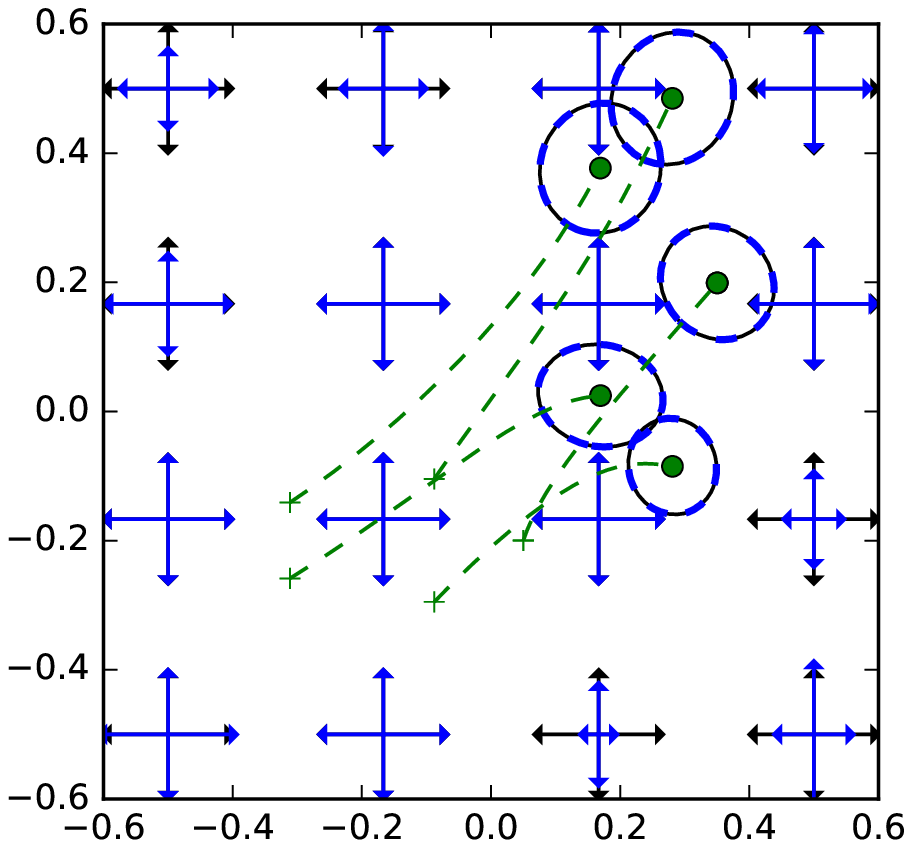}}
      \subfigure[Gradient descent]{\includegraphics[scale=0.6]{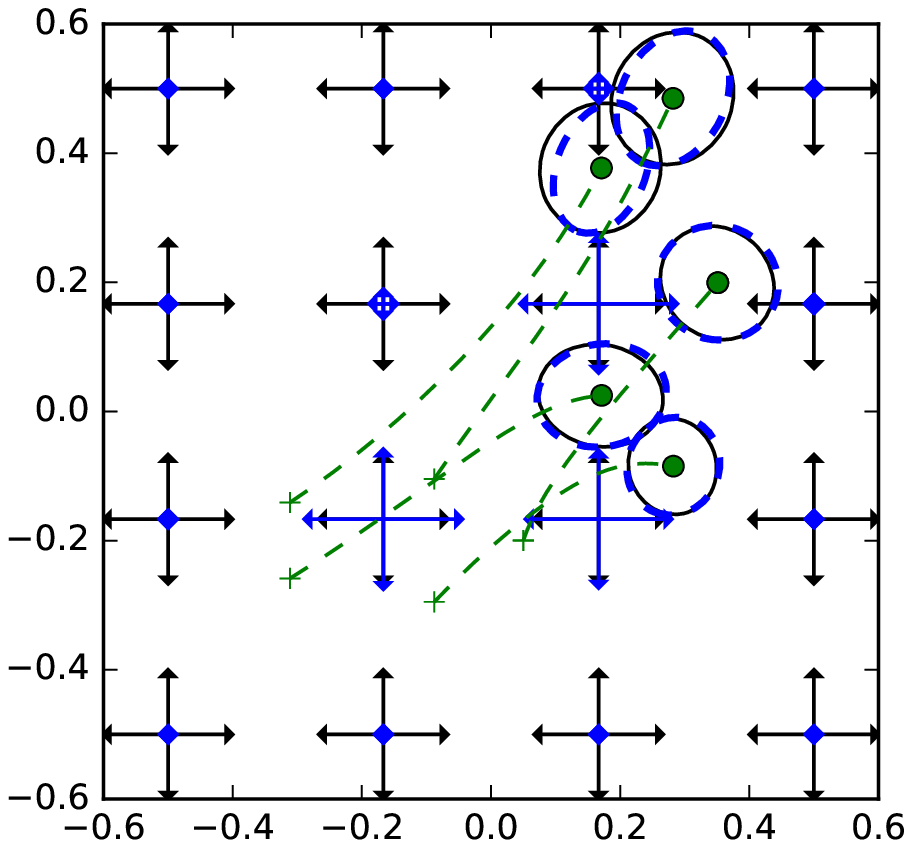}}
    \caption{\small This figure shows the result of noise estimation using the
      moment equation as in Figure~\ref{fig:sigma_single} but for the ellipse experiment. 
        The error in the final covariance for the differential equation genetic algorithm is of the order of $10^{-9}$ and for the BFGS algorithm it is of the order of $5\cdot 10^{-3}$. 
  }
  \label{fig:sigma_ellipse}
\end{figure}

Figure~\ref{fig:sigma_ellipse} and Figure~\ref{fig:sigma_ellipse_mle} show the result of noise estimation using different configurations of the ellipse and both the
method of moments and MLE. The noise parameters $\lambda_l$ are allowed to vary
with $l$ in both cases giving spatially non-uniform noise amplitude. The
algorithms find the correct noise parameters in the areas covered by the landmark
trajectories.

\begin{figure}[htpb]
    \centering
    \subfigure[Landmarks and estmated noise]{\includegraphics[width=.4\columnwidth]{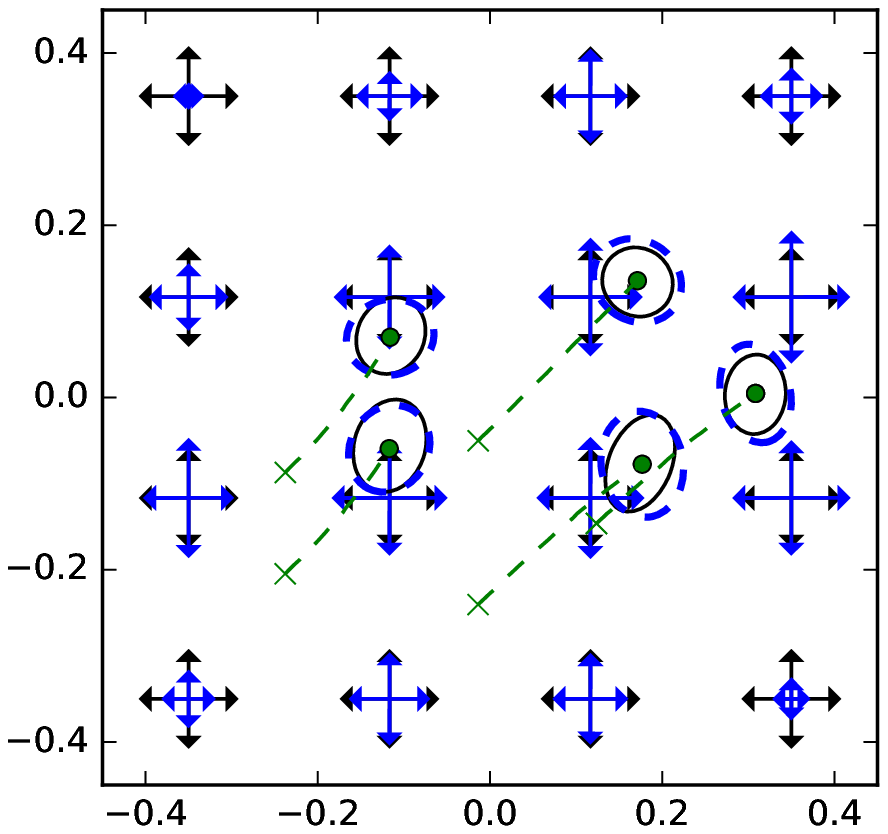}}
    \subfigure[EM iterations]{\includegraphics[width=.4\columnwidth]{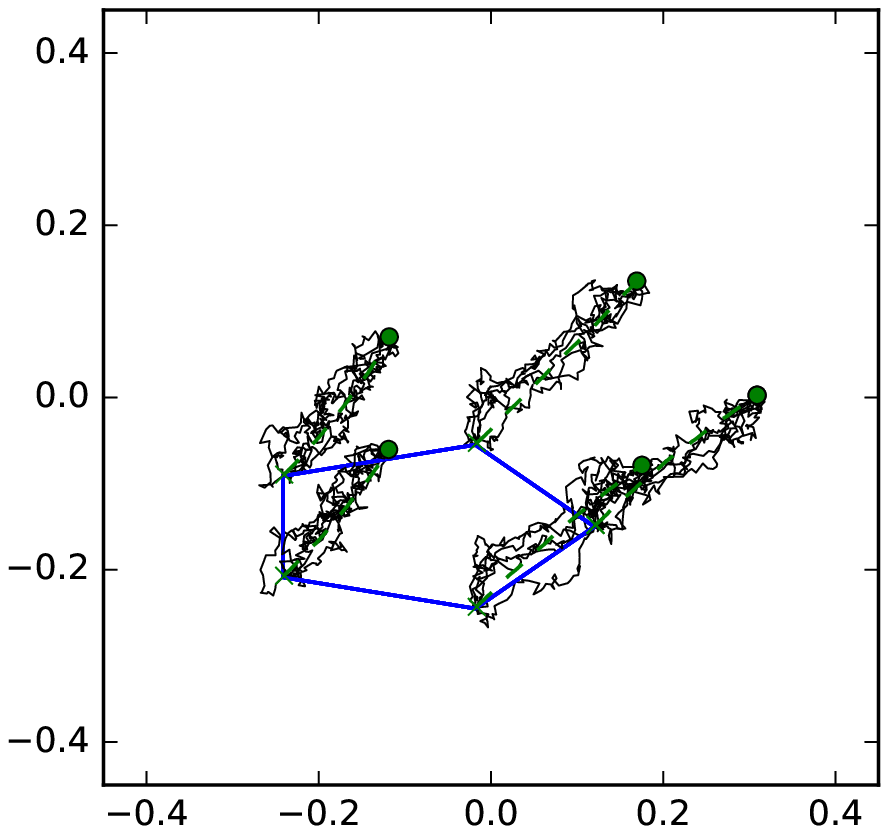}}
    \caption{\small (a) Setup as Figure~\ref{fig:sigma_single_mle} but with five
  landmarks in an ellipsis configuration. (b) Examples of simulated bridges as used in the approximation of the $Q$ function in the EM procedure.
  }
  \label{fig:sigma_ellipse_mle}
\end{figure}

\section{Discussion and Outlook}

As the first topic of this work, we raised the issue of how to
include stochasticity and uncertainty in the framework of large deformation
matching in a systematic and geometrically consistent way. 
In section \ref{theory}, we exposed a general theory of stochastic deformations
in the LDDMM framework, based on the momentum map representation of images in
\cite{bruveris2011momentum}, by
introducing spatially correlated time-dependent noise in the reconstruction relation that is used to compute the deformation map from its velocity field. 
By taking this approach, we have preserved most of the advantages of the theory of reduction by symmetry. In particular, we have preserved the capability of applying this stochastic model to general data structures. 
The dynamical equation is the stochastic EPDiff equation, in which the noise appears in a certain multiplicative form with spatial correlations encoded in a set of spatially dependent functions $\sigma_l$. 
The key feature of this noise is that the structure of the original
equation provided by the theory of reduction by symmetry still remains. 
In particular, the persistence of the momentum map allows for both exact and inexact matching in this stochastic context. 

The question of local in time existence and uniqueness of this equation is important, but it is not treated in this work. 
We refer to \cite{brzezniak2016existence} for such a study for the 2D Euler equation and to \cite{holm2017euler} for the 3D case.  
Another possible extension would be to consider an infinite number of $\sigma_l$
fields with an infinite dimensional Wiener process for the stochastic EPDiff
equation as investigated in \cite{vialard2013extension}, also in the context of stochastic shape analysis. 
We considered have time-independent $\sigma_l$ fields. 
However, there are several approaches for making these fields time dependent besides simply prescribing them as functions of time.
Some of these other approaches were derived by \cite{holm2017stochastic} in the context of stochastic
fluid dynamics. 
In particular, the idea of having the noise fields being carried by the deformation could be of interest in this context as well. 
Yet another possibility would be to introduce two different types of noise fields, one modelling small-scale noise correlations and the other one for larger scale noise correlations. 
In this case, it would make sense for the small scale variability to be advected
by the large-scale deformation, as in the multi-scale model of \cite{holm2012multiscale}. 

After defining the general model in section \ref{theory}, we applied it to exact landmark matching in section \ref{landmark-matching}, which is the simplest non-trivial application of the LDDMM framework. 
This approach allowed investigation of the effects of the noise on large deformation matching in a finite dimensional model.
Introducing the noise in both the momentum and the position equations of the
landmarks made the landmark trajectories rougher than they would have been, otherwise, had the noise been only in the momentum equation. 
The noise in the position equation also increased the flexibility for controlling the landmark trajectories.
This flexibility was used to derive a scheme for simulating diffusion bridges with
corresponding sampling correction factor that allowed evaluation of expectations with respect to the original conditioned landmark dynamics. 
In addition, we used the finite dimensionality of the system to derive the Fokker-Planck equation and apply it to the dynamics of moments of the probability distribution function. 

Some modifications to the standard theory of diffusion bridges were made to accommodate
the case of landmark dynamics and to improve the speed and accuracy of the
estimation of expectations over conditioned landmarks trajectories. 
The landmarks represent the simplest cases for numerical shape analysis, especially in the context of stochastic systems. 
We used a simple Heun method to solve the stochastic landmark equations. 
Higher order integration schemes could have been used, such as the stochastic variational integrators of \cite{holm2016stochastic}. 
The next step in extending the landmark example is to allow for inexact matching
and to study the trade-off between the effect of noise and the tolerance of the matching. 

Several issues regarding ergodicity and other properties of the Kolmogorov operator were left open in this paper, whose future treatments could add to the theoretical understanding of the model.
Finally, the stochastic LDDMM framework can be applied to other types of data structures, in particular to images with inexact matching as originally done in \cite{beg2005computing}. 
Studying the effects of the stochastic model on other nonlinear data structures such as curves or surfaces would also be of great interest for future works. 

As a second topic, we raised the issue of determining the noise correlation
from data sets which would allow the theory of stochastic
deformations to be used with observed data. 
We developed two independent methods which we implemented and applied to several test examples. 
First, the moment equation allows matching of the sample
moments. It is deterministic, making optimisation of the noise parameters stable
and efficient, and it does not require special conditions on the noise fields.
Its accuracy depends on the approximation order in the moment equation. 
Scaling the moment equation to a large number of landmarks or continuous shapes
such as curves may be challenging as well as optimising for a high number of unknown parameters. 
In the landmark experiments we presented above, this approach allowed us to reliably estimate the underlying noise, but an extension of this method to infinite dimensional representations of shapes is not possible unless a discretized version of the equations is used.
For this method, we also made two approximations that could possibly be improved elsewhere. 
One of them is the truncation to retain only second order terms in the moment equation, and the other is to approximate the expectation of a kernel function as the kernel of the expected values. 
Both approximations were shown to work well for cases with small enough noise, which would be the case in most applications. 
Finally, it is also important to notice that we did not use the freedom of the initial value of the variance of the momentum and the position/momentum correlation. 
These parameters could either be inferred using this scheme (with a larger parameter space) or be obtained by using other information about the data. 

The second method is the MLE optimisation, a Monte Carlo method which evaluates expectations over conditioned stochastic trajectories.
The bridge sampling scheme we used requires the noise fields to span the entire $\mathbf q$-space to allow guiding the landmarks towards their target. 
With high nonlinearity as may happen with large initial momentum and high gradients of the noise fields,
guiding the trajectories towards their target with high-probability bridges can be challenging. 
In general, the stochastic nature of the algorithm makes it harder to control than the matching provided by the moment equation.
The bridge sampling scheme can be interpreted as a gradient flow, as discussed in \cite{arnaudon2016stochastic2} when applied to images. 
It allows the likelihood of observed images to be evaluated without a prior image registration step.
The method may thus be applicable to image analysis problems, and more generally for inexact matching of shapes in which case the requirement of the noise to span the $\mathbf q$-space may be relaxed.

The inference of noise parameters treated here can be extended to more general statistical inference problems on shape spaces. 
Inferring the initial $\mathbf q_0$ positions can be regarded as estimating a most-likely mean, thereby drawing similarities to the Frech\'et mean \cite{frechet_les_1948} and to means defined by the maximum likelihood of probability distributions in nonlinear spaces \cite{sommer_anisotropic_2015}. 
When generalised to images, the approach can be used for simultaneous estimation of template images \cite{joshi_unbiased_2004}, possible time-dependent transformations in the momentum as caused for example by 
disease processes \cite{muralidharan_sasaki_2012}, and population variation in the spatial noise correlation.

It is possible to generalise the stochastic equations we have introduced here to allow for time-dependent noise amplitude as done in \cite{holm2017stochastic} for fluid dynamics. 
In this case, the noise fields could be advected by the diffeomorphism and only the initial condition of the noise field would have to be inferred. This requires the choice of a meaningful advection scheme. By construction of its metric LDDMM is right-invariant, and the flow energy is therefore measured in Eulerian coordinates. This leads us to define stochastic flows that are compatible with this right-invariant geometry thus giving noise in Eulerian coordinates. In the deterministic setting, left-invariant metrics \cite{schmah_diffeomorphic_2015} provide a Lagrangian view of the metric that thus, in a medical context, follows the advected anatomy. We leave it as an open and very relevant problem to consider advected, or left-invariant Lagrangian noise.

Extending the inference methods presented here to other data structures, in particular to infinite dimensional shapes spaces, would again constitute an interesting future direction.
As discussed in detail at the end of section \ref{theory}, we believe that the methods presented here with suitable modifications can be applied also for infinite-dimensional representations of shapes, and that additional methods could be introduced, such as stochastic filtering for further data assimilation of the results in infinite-dimensional cases, see e.g. \cite{bain2009fundamentals}. 

\subsubsection*{Acknowledgements}

{\small 
We are grateful to M. Bruveris, M. Bauer, A Pai, N. Ganaba, C. Tronci, M.R. Schauer, T. Tyranowski and F. Van Der Meulen for helpful discussions of this material and to the anonymous referees for thoughtful comments which improved the exposition of this paper.
AA and DH were partially supported by the European Research Council Advanced Grant 267382 FCCA held by DH. DH is also grateful for support from EPSRC Grant EP/N023781/1.
AA acknowledges partial support from an Imperial College London Roth Award and the EPSRC through award EP/N014529/1 funding the EPSRC Centre for Mathematics of Precision Healthcare.
SS is partially supported by the CSGB Centre for Stochastic Geometry and Advanced Bioimaging funded by a grant from the Villum Foundation. 
The research was partially completed while the authors were visiting the Institute for Mathematical Sciences, National University of Singapore in 2016.
}

\bibliographystyle{alpha}
\bibliography{biblio.bib}

\appendix

\section{Bridge Sampling} \label{sec:bridgesamplingscheme}

We here follow \cite{delyon2006simulation} and the
later paper \cite{marchand_conditioning_2011} to argue for the almost sure hitting of a target $\mathbf v$ for the guided process \eqref{eq:guiddiff} and to find 
the correction term $\varphi(\mathbf q,\mathbf p,t)$.
For completeness, we will explicitly derive the correction term following the program in \cite[Theorem 5]{delyon2006simulation}. 
The guided SDE \eqref{eq:guiddiff} differs from the previous schemes in
using the function $\phi_{t,T}:\mathbb R^{2dN}\rightarrow\mathbb R^{dN}$ to predict
the endpoint, and, importantly, in that the diffusion field $\Sigma$ is not
invertible resulting in a semi-elliptic diffusion. We handle the first issue by repeated application of the Girsanov theorem. 
This also accounts for the unboundedness of the
drift term $b(\hat{\mathbf q},\hat{\mathbf p})$ coming from the momentum of the
landmarks in the same way as \cite{marchand_conditioning_2011}.
We do not here argue for the $t\rightarrow T$ limit of the expectation of the
correction term that in the
elliptic case follows from \cite{delyon2006simulation,marchand_conditioning_2011}.

Let 
$b(\mathbf q,\mathbf p)$ be the drift in \eqref{sto-Ham} in It\^o
form. Because this drift is unbounded, we construct
$\tilde{b}(\mathbf q,\mathbf p)$ in Theorem~\ref{thm:endpoint} to be an approximation so that the $\mathbf q$-part
$\tilde{b}_{\mathbf q}$ is bounded on $\mathbb R^{dN}$.
To construct a map $\phi_{t,T}:\mathbb R^{2dN}\rightarrow\mathbb R^{dN}$ satisfying the conditions of the theorem, let $\phi_{t,T}$ be the $\mathbf q$-part of the time $T$ solution to the ODE 
$\partial_t(\mathbf q_t,\mathbf p_t)=\tilde{b}(\mathbf q_t,\mathbf p_t)$
started at time $t$ with initial conditions $(\mathbf q,\mathbf p)$.
This ODE corresponds to the deterministic ODE \eqref{EP}, however using 
the drift approximation to ensure $\partial_t\phi_{t,T}(\mathbf q,\mathbf p)$ is bounded.
Then the process
$\frac{\tilde{\phi}_{t,T}(\hat{\mathbf q},\hat{\mathbf p})-\hat{\mathbf q}}{T-t}$ is
defined, bounded and continuous on $[0,T]$. 

The SDE
\begin{align}
  \begin{pmatrix}
    d\hat{\mathbf q} \\
    d\hat{\mathbf p}
  \end{pmatrix}
  = \tilde{b}(\hat{\mathbf q},\hat{\mathbf p})dt
  -\frac{\Sigma(\hat{\mathbf q},\hat{\mathbf p})\Sigma_{\hat{\mathbf q}}(\hat{\mathbf q})^\dagger (\hat{\mathbf q}-\mathbf v)}{T-t}dt 
    + \Sigma(\hat{\mathbf q},\hat{\mathbf p})dW 
  \label{eq:boundedguid}
\end{align}
differs from the It\^o form of the SDE \eqref{eq:guiddiff} by 
\begin{align}
  (b(\mathbf q,\mathbf p)
  -
  \tilde{b}(\mathbf q,\mathbf p)
  )dt
    -\frac{\Sigma(\mathbf q,\mathbf p)\Sigma_{\mathbf q}(\mathbf q)^\dagger
    (\varphi_{t,T}(\mathbf q,\mathbf p)-\mathbf q)}{T-t}dt
  \ .
  \label{eq:unboundeddiff}
\end{align}
As argued by \cite{marchand_conditioning_2011}, \ref{eq:boundedguid} has a
unique solution satisfying $\lim_{t\rightarrow T}\hat{\mathbf q}=v$ a.s., and 
the processes \ref{eq:boundedguid} and \eqref{eq:guiddiff} are absolutely
continuous with respect to each other.
The correction term $\varphi(\mathbf q,\mathbf p,t)$ can be derived from 
\cite[Theorem 3]{marchand_conditioning_2011} and the difference \ref{eq:unboundeddiff}.
For completeness, we give the derivation in the landmark case that proves
Theorem~\ref{thm:endpoint} below.

\begin{proof}[Proof of Theorem \ref{thm:endpoint}]
  Let $f:W(\mathbb R^{2dN})\rightarrow\mathbb R$ be a non-negative measurable
  function on $[0,t]$, $t<T$. Following 
  \cite{delyon2006simulation}, we define
  \begin{align*}
    h(\mathbf q,t)
    :=
      -\frac{\Sigma_{\mathbf q}(\mathbf q)^\dagger(\mathbf q-\mathbf v)}{T-t}
  \end{align*}
  noting that in the present case, we use the pseudo-inverse 
  $\Sigma_{\mathbf q}(\mathbf q)^\dagger$ in $h$
  since $\Sigma_{\mathbf q}(\mathbf q)$ is not invertible.
  Let now $(\tilde{\mathbf q},\tilde{\mathbf p})$ be a solution to the SDE
\begin{align}
  \begin{pmatrix}
    d\tilde{\mathbf q} \\
    d\tilde{\mathbf p}
  \end{pmatrix}
  = \tilde{b}(\tilde{\mathbf q},\tilde{\mathbf p})dt
  +\Sigma(\tilde{\mathbf q},\tilde{\mathbf p})h(\tilde{\mathbf q},t)dt 
    + \Sigma(\tilde{\mathbf q},\tilde{\mathbf p})dW 
  \label{eq:boundedguid_no_phi}
\end{align}
  From the Girsanov theorem with unbounded drift \cite[Thm. 1]{delyon2006simulation}, we have
\begin{align}
    \mathbb E_{(\tilde{\mathbf q},\tilde{\mathbf p})}\left [ f(\tilde{\mathbf
    q},\tilde{\mathbf p}) 
    \varphi^b(\tilde{\mathbf q},\tilde{\mathbf p},t)
     \right ]
    =
    \mathbb E_{(\mathbf q,\mathbf p)}\left[f(\mathbf q,\mathbf p,t)
    \tilde{\varphi}(\mathbf q,\mathbf p,t)
    \right]
    \label{eq:girsanov}
  \end{align}
  where
  \begin{align}
      \log\tilde{\varphi}(\mathbf q,\mathbf p,t)
      &:=
    \int_0^th^T(\mathbf q,\mathbf p,s)dW-\frac{1}{2}\int_0^t\|h(\mathbf q,\mathbf p,s)\|^2ds
     \ ,
     \label{eq:phixh}
     \\
     \nonumber
     \log \varphi^b(\mathbf q,\mathbf p,t)
      &:=
    \int_0^t(b_\mathbf q(\mathbf q,\mathbf p)-\tilde{b}_\mathbf q(\mathbf q,\mathbf p))^T\Sigma_\mathbf q(\mathbf q)^{\dagger,T}dW
      \\&-\int_0^t\frac{1}{2}
    \|\Sigma_\mathbf q(\mathbf q)^{\dagger}
    (b_\mathbf q(\mathbf q,\mathbf p)-\tilde{b}_\mathbf q(\mathbf q,\mathbf p)) \|^2 ds  \ .
  \end{align}
  We now define an intermediate function 
\begin{align}
    g(\mathbf q,t):=\frac{(\mathbf q-\mathbf v)^TA(\mathbf q)(\mathbf q-\mathbf v)}{T-t}\, , 
\end{align}
  and compute 
  \begin{align*}
    dg(\tilde{\mathbf q},t)
    &=
    \frac{(\tilde{\mathbf q}-\mathbf v)^TA(\tilde{\mathbf
    q})(\tilde{\mathbf q}-\mathbf v)}{(T-t)^2}dt
    +
    \frac{d\big((\tilde{\mathbf q}-\mathbf v)^TA(\tilde{\mathbf q})(\tilde{\mathbf q}-\mathbf v)\big)}{T-t}
    \ .
  \end{align*}
  Applying the product rule, we obtain for the second term
  \begin{align*}
    d\big((\tilde{\mathbf q}-\mathbf v)^TA(\tilde{\mathbf q})(\tilde{\mathbf q}-\mathbf v)\big)
    &=
    2(\tilde{\mathbf q}-\mathbf v)^TA(\tilde{\mathbf q})d\tilde{\mathbf q}
      +(\tilde{\mathbf q}-\mathbf v)^T\big(dA(\tilde{\mathbf
      q})\big)(\tilde{\mathbf q}-\mathbf v)
    \\
   &\qquad
    +\sum_{i,j}
    2d[A_{ij}(\tilde{\mathbf q}),(\tilde{\mathbf q}-\mathbf v)_i(\tilde{\mathbf
    q}-\mathbf v)_j)]
       \, .
  \end{align*}
  Writing the process $\frac12 g(\tilde{\mathbf q},t)$ in integral form,
  \begin{align*}
    \frac{1}{2}\int_0^tdg(\tilde{\mathbf q},s)
    &
    =
    \int_0^t\frac{(\tilde{\mathbf q}-\mathbf v)^TA(\tilde{\mathbf
    q})(\tilde{\mathbf q}-\mathbf v)}{2(T-s)^2}ds 
    +\int_0^t\frac{(\tilde{\mathbf q}-\mathbf v)^TA(\tilde{\mathbf
    q})d\tilde{\mathbf q}}{T-s}
    \\&
    +\int_0^t\frac{ (\tilde{\mathbf q}-\mathbf v)^T\big(dA(\tilde{\mathbf
    q})\big)(\tilde{\mathbf q}-\mathbf v)}{2(T-s)}
    +\sum_{i,j}\int_0^t\frac{d[A_{ij}(\tilde{\mathbf q}),(\tilde{\mathbf
    q}-\mathbf v)_i(\tilde{\mathbf q}-\mathbf v)_j)]}{T-s}
      \ .
  \end{align*}
  Note that the first term 
  $\int_0^t\frac{(\tilde{\mathbf q}-\mathbf v)^T A(\tilde{\mathbf
  q})(\tilde{\mathbf q}-\mathbf v)}{2(T-s)^2}ds$
  of the right hand side is the negative of the term
  $-\frac{1}{2}\int_0^t\|h(\tilde{\mathbf q},\mathbf p,s)\|^2ds$
  in \eqref{eq:phixh}. The second term expands to
  \begin{equation*}
    \int_0^t\frac{(\tilde{\mathbf q}-\mathbf v)^T A(\tilde{\mathbf
    q})d\tilde{\mathbf q}}{T-s}
    =
    \int_0^t\frac{(\tilde{\mathbf q}-\mathbf v)^T A(\tilde{\mathbf q})\tilde{b}(\tilde{\mathbf q},\tilde{\mathbf p})ds}{T-s}
    +
    \int_0^t\frac{(\tilde{\mathbf q}-\mathbf v)^T \Sigma_\mathbf q(\tilde{\mathbf q})^{\dagger,T}dW}{T-s}
  \end{equation*}
  where the second term of the right hand side is the negative of
  $\int_0^th^T(\tilde{\mathbf q},\mathbf p,s)dW$.
  Rearranging terms and inserting in \eqref{eq:phixh},
  \begin{align*}
    &\log\tilde{\varphi}(\tilde{\mathbf q},\mathbf p,t)
    =
    \int_0^th^T(\tilde{\mathbf q},\mathbf p,s)dW
    -\frac{1}{2}\int_0^t\|h(\tilde{\mathbf q},\mathbf p,s)\|^2ds
    \\\qquad&
    =
    -\frac{1}{2}g(\tilde{\mathbf q}(t),t)
    +\frac{1}{2}g(\tilde{\mathbf q}(0),0)
    +\int_0^t\frac{(\tilde{\mathbf q}-\mathbf v)^TA(\tilde{\mathbf q})\tilde{b}(\tilde{\mathbf q},\tilde{\mathbf p})ds }{T-s}
    \\&
    \quad\,
    +\int_0^t\frac{
      (\tilde{\mathbf q}-\mathbf v)^T\big(dA(\tilde{\mathbf q})\big)(\tilde{\mathbf q}-\mathbf v)
    }{2(T-s)}
    +\sum_{i,j}\int_0^t\frac{d[A_{ij}(\tilde{\mathbf q}),(\tilde{\mathbf q}-\mathbf v)_i(\tilde{\mathbf q}-\mathbf v)_j)]}{T-s}
    \ .
  \end{align*}
  We can now use the Girsanov theorem again to change the drift from 
  \eqref{eq:boundedguid_no_phi} to \eqref{eq:boundedguid}. For this, we define
  \begin{align*}
     \log \varphi^\varphi(\mathbf q,\mathbf p,t)
      &:=
      \int_0^t\frac{(\varphi_{t,T}(\mathbf q,\mathbf p)-\mathbf q)^T\Sigma_\mathbf q(\mathbf q)^{\dagger,T}dW}{T-t}
      \\&-\int_0^t\frac{
    \|\Sigma_\mathbf q(\mathbf q)^{\dagger}
    (\varphi_{t,T}(\mathbf q,\mathbf p)-\mathbf q)\|^2 ds}{2(T-t)^2}
    \ .
  \end{align*}
  Then let $\varphi(\mathbf q,\mathbf p,t)$ be the function satisfying
  \begin{align*}
    \log\varphi(\mathbf q,\mathbf p,t)&=-\log\tilde{\varphi}(\mathbf q,\mathbf p,t)-\frac{1}{2}g(\mathbf q,t)+\frac{1}{2}g(\mathbf q(0),0)
    \\&\quad\,
    +\log\varphi(\mathbf q,\mathbf p,t)^b
    +\log \varphi^\varphi(\mathbf q,\mathbf p,t)
    \ .
  \end{align*}
  Now \eqref{eq:girsanov} and the definition of $\log\varphi$ gives
  \begin{equation}
    \mathbb E_{P_{(\hat{\mathbf q},\hat{\mathbf p})}}\left ( f(\hat{\mathbf q},\hat{\mathbf p}) 
    \varphi(\hat{\mathbf q},\hat{\mathbf p},t)
    \right)
    =
    e^{\frac{1}{2}g(\mathbf q(0),0)}
    \mathbb E_{P_{(\mathbf q,\mathbf p)}}\left[f(\mathbf q,\mathbf p)
    e^{-\frac{1}{2}g(\mathbf q,t)}
     \right ]
     \label{eq:E_t}
   \end{equation}
  with $(\hat{\mathbf q},\hat{\mathbf p})$ a solution to \eqref{sto-Ham}.
  Thus
  \begin{align*}
    \lim_{t\rightarrow T}
    \frac{
    \mathbb E_{(\hat{\mathbf q},\hat{\mathbf p})}\left [ f(\hat{\mathbf q},\hat{\mathbf p}) 
    \varphi(\hat{\mathbf q},\hat{\mathbf p},t)
    \right]
    }{
    \mathbb E_{(\hat{\mathbf q},\hat{\mathbf p})}\left [ 
    \varphi(\hat{\mathbf q},\hat{\mathbf p},t)
    \right]
    }
    =
    \lim_{t\rightarrow T}
    \frac{
    \mathbb E_{(\mathbf q,\mathbf p)}\left[
    f(\mathbf q,\mathbf p)
    e^{-\frac{1}{2}g(\mathbf q,t)}
    \right]
    }{
    E_{P_{(\mathbf q,\mathbf p)}}\left[
    e^{-\frac{1}{2}g(\mathbf q,t)}
    \right]
    }
    =
    \mathbb E_{(\mathbf q,\mathbf p)|v}\left[ f(\mathbf q,\mathbf p) 
    \right]
  \end{align*}
  where convergence of the right-hand side limit to the conditioned process follows from Lemma~\ref{lem:conditioning} below. The limit expression \eqref{eq:density_lim_E} for the density follows from Lemma~\ref{lem:density_lemma}
  and \eqref{eq:E_t}.
\end{proof}
The lemmas \ref{lem:conditioning} and \ref{lem:density_lemma} below follow
\cite{delyon2006simulation} and \cite{marchand_conditioning_2011} with minor
modifications to clarify where the assumption~\ref{bridge-assumption} on the density of the process $(\mathbf
q_t,\mathbf p_t)$ is needed in the semi-elliptic case.

\begin{lemma}
  Let $(\mathbf q,\mathbf p)$ be a solution to \eqref{sto-Ham} satisfying
  assumption~\ref{bridge-assumption} and the conditions of Theorem~\ref{thm:endpoint}.
  Let $0<t_1<t_2<\cdots<t_n<T$ be a finite set of time point in $[0,T]$ and let $f\in C_b(\mathbb{R}^{n2dN})$.
  Let $\psi_t$ be the process
  \begin{equation}
    \psi_t
    =
    e^{-\frac{1}{2}g(\mathbf q,t)}
    =
    e^{-\frac{\|\Sigma_{\mathbf q}(\mathbf q)^\dagger(\mathbf q-\mathbf v)\|^2}{2(T-t)}}
    \ ,
    \label{eq:psit}
  \end{equation}
  with $g$ as defined above. Then
  \begin{equation}
    \lim_{t\to T}
    \frac{
      \mathbb E_{(\mathbf q,\mathbf p)}
      [f((\mathbf q_1,\mathbf p_1),\ldots,(\mathbf q_n,\mathbf p_n))\psi_t]
    }{
      \mathbb E_{(\mathbf q,\mathbf p)}
      [\psi_t]
    }
    =
      \mathbb E_{(\mathbf q,\mathbf p)}
      [f((\mathbf q_1,\mathbf p_1),\ldots,(\mathbf q_n,\mathbf p_n))|\mathbf q_T=\mathbf v]
    \ .
    \label{eq:conditioning}
  \end{equation}
  \label{lem:conditioning}
\end{lemma}
\begin{proof}
  Following \cite{delyon2006simulation}, we write
  \begin{equation*}
    \frac{
      \mathbb E_{(\mathbf q,\mathbf p)}
      [f(\mathbf q_1,\mathbf p_1,\ldots,\mathbf q_n,\mathbf p_n)\psi_t]
    }{
      \mathbb E_{(\mathbf q,\mathbf p)}
      [\psi_t]
    }
    =
    \frac
    {\int_{\mathbb R^d}\Phi_f(t,\mathbf q)e^{-\frac{\|\Sigma_{\mathbf q}(\mathbf q)^\dagger(\mathbf q-\mathbf v)\|^2}{2(T-t)}}d\mathbf q}
    {\int_{\mathbb R^d}\Phi_1(t,\mathbf q)e^{-\frac{\|\Sigma_{\mathbf q}(\mathbf q)^\dagger(\mathbf q-\mathbf v)\|^2}{2(T-t)}}d\mathbf q}
  \end{equation*}
  with
  \begin{align*}
    \Phi_f(t,\mathbf q)
    &=
    \int_{\mathbb R^{n2dN}}
    f(\mathbf q_1,\mathbf p_1,\ldots,\mathbf q_n,\mathbf p_n)\cdot
    \\
    &\qquad\qquad
    \mathbb P(\mathbf q_0,\mathbf p_0;\mathbf q_1,\mathbf p_1, t_1)
    \cdots
    \mathbb P(\mathbf q_n,\mathbf p_n;\mathbf q, t-t_n)
    d(\mathbf q_1,\mathbf p_1)
    \cdots
    d(\mathbf q_n,\mathbf p_n)
  \end{align*}
  Note that $\Phi_f$ is continuous by assumption. We now apply a change of
  variable $\mathbf q=\mathbf v+(T-t)^{\frac12}\mathbf q'$ to get
  \begin{align*}
    &
    (T-t)^{-\frac d2}
    \int_{\mathbb R^{dN}}
    \Phi_f(t,\mathbf q)
    e^{-\frac{\|\Sigma_{\mathbf q}(\mathbf q)^\dagger(\mathbf q-\mathbf v)\|^2}{2(T-t)}}
    d\mathbf q
    \\&\qquad\qquad
    =
    \int_{\mathbb R^{dN}}
    \Phi_f(t,\mathbf v+(T-t)^{\frac12}\mathbf q')
    e^{-\frac{\|\Sigma_{\mathbf q}(\mathbf v+(T-t)^{\frac12}\mathbf q')^\dagger\mathbf q'\|^2}{2}}
    d\mathbf q'
    \\&\qquad\qquad
    \to
    \Phi_f(T,\mathbf v)
    \int_{\mathbb R^{dN}}
    e^{-\frac{\|\Sigma_{\mathbf q}(\mathbf v)^\dagger\mathbf q'\|^2}{2}}
    d\mathbf q'
    \ .
  \end{align*}
  From assumption~\ref{bridge-assumption},
  $\Phi_f(t,\mathbf q)$ is continuous and bounded. Because
  $\Sigma$ is bounded,
  $e^{-\frac{\|\Sigma_{\mathbf q}(\mathbf v+(T-t)^{\frac12}\mathbf q')^\dagger(\mathbf q')\|^2}{2}}
  \le e^{-\frac{c\|\mathbf q'\|^2}{2}}$ for some constant $c$ and the dominated
  convergence theorem implies the limit.
  We conclude that
  \begin{equation*}
    \lim_{t\to T}
    \frac{
      \mathbb E_{(\mathbf q,\mathbf p)}
      [f(\mathbf q_1,\mathbf p_1,\ldots,\mathbf q_n,\mathbf p_n)\psi_t]
    }{
      \mathbb E_{(\mathbf q,\mathbf p)}
      [\psi_t]
    }
    =
    \frac
    {\Phi_f(T,\mathbf v)}
    {\Phi_1(T,\mathbf v)}
    \ .
  \end{equation*}
  The result now follows from the definition of $\Phi_f$, see \cite{delyon2006simulation}.
\end{proof}
If only the density $\mathbb P(\mathbf q_0, \mathbf p_0;\mathbf q,t)$ is
of interest, the following result holds assuming only continuity and boundedness
of $\mathbb P(\mathbf q_0, \mathbf p_0;\mathbf q,t)$ for fixed 
initial conditions $(\mathbf q_0,\mathbf p_0)$.
\begin{lemma}
  Let $(\mathbf q,\mathbf p)$ be a solution to \eqref{sto-Ham} with the conditions of Theorem~\ref{thm:endpoint} and assume the process has 
  a density $\mathbb P(\mathbf q_0,\mathbf p_0;\mathbf q,\mathbf p,t)$ and that 
  $\mathbb P(\mathbf q_0,\mathbf p_0; \mathbf q,t)$ is 
  continuous in $t$ and $\mathbf q$ and bounded on $\{(\mathbf q,t)|T-\epsilon\le t\le T\}$. Let $\psi_t$ be the process defined above.
  Then
  \begin{equation}
    \mathbb P(\mathbf q_0,\mathbf p_0; \mathbf q,T)
    =
    \left(\frac{\left|A(\mathbf v)\right|}{2\pi T}\right)^{\frac d2}
    \lim_{t\to T}
      \mathbb E_{(\mathbf q,\mathbf p)}
      [\psi_t]
    \ .
    \label{eq:density_lemma}
  \end{equation}
  \label{lem:density_lemma}
\end{lemma}
\begin{proof}
  The result follows from the convergence of $\Phi_1(t,\mathbf q)$ in
  Lemma~\ref{lem:conditioning}.
\end{proof}

\section{ Moment equation for stochastic landmark}\label{moment-landmark}

\subsection{Cluster expansion method}\label{cluster}

We explain the basics of this method, which can be found in more details in, for example, \cite{kira2011semiconductor} with application in the context of semiconductor physics.
This method is used when one seeks the dynamics of the expected value of $N$ particles that we will write here $\Braket{N}$. 
One cannot solve the complete system, especially if the number of particles is large, thus we want to approximate the expected value of products in term of only a few independent variables. 
For this, we apply the cluster expansion, which begins by writing
\begin{align}
    \braket{2}= \braket{2}_s + \Delta_2\braket{2}:= \braket{1}\braket{1} + \Delta_2\braket{2}, 
\end{align}
The next decomposition is 
\begin{align}
    \braket{3} = \braket{3}_s + \braket{1}\Delta_2\braket{2} + \Delta_3\braket{3},
\end{align}
and so on and so forth. 
We then only compute the dynamics for the singlets $\braket{1}$ and the correlations, up to some chosen order. 
In the sequel, we will only consider the doublet correlations $\Delta_2$, and in this case, we have the general decomposition 
\begin{align*}
    \braket{N}= \braket{N}_s + \braket{N-2}_s\Delta_2\braket{2} + \braket{N-4}_s\Delta_2^2\braket{2} + \sum_i \braket{N-2i}_s \Delta_2^i\braket{2} + O(\Delta_3) \, . 
\end{align*}

In the context of quantum mechanics, where the particle operators do not commute, extra care is needed especially for the sign of the term. 
Here we will consider $q_i^\alpha $ and $p_i^\alpha$ as our particles, and as they commute, the expansions are simpler than in \cite{kira2011semiconductor}. 
We directly compute two of them for illustration, up to quadratic order, 
\begin{align*}
    \braket{ q_i^\alpha p_j^\beta } &= \braket{ q_i^\alpha}\braket{ p_j^\beta }  + \Delta_2 \braket{ q_i^\alpha p_j^\beta }\\
    \braket{ q_i^\alpha q_j^\beta p_k^\gamma } &\approx \braket{ q_i^\alpha}\braket{ q_j^\beta }\braket{ p_k^\gamma } + \braket{ q_i^\alpha}\Delta_2  \braket{q_j^\beta p_k^\gamma }  + \braket{q_j^\beta }\Delta_2 \braket{ q_i^\alpha p_k^\gamma }  +\braket{p_k^\gamma }\Delta_2\braket{ q_i^\alpha q_j^\beta}\, . 
\end{align*}

This sort of expansion can fit a more geometrical framework, where the final equations for the first moment will preserve the original structure of the equations. 
This was developed first in \cite{holm1990moment} and later in \cite{holm2007geometric,holm2010double}.
We will not use this method here for a good reason related to the form of the equations. 
A key step in these papers is to expand the expected value of the Hamiltonian in terms of a finite number of moments, to enable computation of the equation of motion. 
In our case, the Hamiltonian has a kernel function, which generally cannot be expanded in a finite sum of polynomial terms. 
By doing the computations directly, we will be able to do another approximation for the kernels, that is, we will assume that they commute with the operation of expectation. 
A more subtle approximation can be done using the Heaviside function but will give a much larger number of terms in the expansion, see appendix \ref{second-order-moment} for no clear improvements of the solution.

To perform this expansion on the Fokker-Planck equation associated to the landmark dynamics we will use several simplifications:
\begin{itemize}
    \item Gaussian noise fields $\sigma_l$ in \eqref{kernels}, 
    \item for a kernel $K(x)$, we will assume that $\braket{K(x)} \approx K(\braket{x})$ and 
    \item only the second order correlations $\Delta_2$ will be considered in this expansion.  
\end{itemize}
These assumptions can be relaxed but the resulting equation may be difficult to compute.

\subsection{First moments}

Recall the backward Kolmogorov operator on $q_i^\alpha$
\begin{align}
    \mathscr Lq_i^\alpha &= \frac{\partial h}{\partial p_i^\alpha} +\frac12\sum_{l,\gamma,\delta}\frac{\partial\sigma_l^\alpha(\mathbf q_i) }{\partial q_i^\gamma} \sigma_l^\gamma(\mathbf q_i)\, , 
    \label{Lq-appendix}
\end{align}
which is used to compute the time evolution of the singlet 
\begin{align}
    \frac{d}{dt} \braket{q_i^\alpha} = \braket{p_j^\alpha} K(\braket{q_i}-\braket{q_j})
    -\sum_{l,\gamma,\delta}\frac{1}{2\sigma_l^2} \sigma_l^\alpha(\braket{\mathbf q_i})(\braket{q_i^\gamma}-\delta_l^\gamma) \sigma_l^\gamma(\braket{\mathbf q_i})\, . 
\end{align}
In this case, the equation only depends on the singlet of the momentum variable. 
We thus compute
\begin{align}
    \mathscr L p_i^\alpha &= -\frac{\partial h}{\partial q_i^\alpha}  +\frac12\sum_{l,\gamma,\delta}p_j^\gamma\frac{\partial\sigma_l^\gamma(\mathbf q_i)}{\partial q_i^\delta}\frac{\partial\sigma_l^\delta(\mathbf q_i)}{\partial q_i^\alpha} - \frac12\sum_{l,\gamma,\delta} p_i^\gamma \frac{\partial^2\sigma_l^\gamma(\mathbf q_i) }{\partial q_i^\alpha \partial q_i^\delta} \sigma_l^\delta(\mathbf q_i)\, .
    \label{Lp}
\end{align}
which similarly gives the time evolution of the momentum singlet in two terms as 
\begin{align*}
    \frac{d}{dt} \braket{p_i^\alpha} &= A_p + B_p\, , 
\end{align*}
where
\begin{align*}
    A_p &= \frac{1}{\alpha^2}\sum_{j,\gamma} K(\braket{q_i}- \braket{q_j})\braket{p_i^\gamma p_j^\gamma  (q_i^\alpha-q_j^\alpha)}\\
    B_p &= \sum_{l,\gamma}\frac{1}{2\sigma_l^2} \braket{p_i^\gamma} \sigma_l^\gamma(\braket{\mathbf q_i})  \sigma_l^\alpha (\braket{\mathbf q_i})\, . 
\end{align*}
We then expand $A_p$ further using the cluster-expansion method on the triplet to get 
\begin{align*}
    A_p &= \frac{1}{\alpha^2}\sum_{j,\gamma} K(\braket{\mathbf q_i}- \braket{\mathbf q_j})(\braket{p_i^\gamma p_j^\gamma  q_i^\alpha}-\braket{p_i^\gamma p_j^\gamma  q_j^\alpha})  \\
    &\approx \frac{1}{\alpha^2}\sum_{j,\gamma} K(\braket{\mathbf q_i}- \braket{\mathbf q_j})\Big (
    \braket{p_i^\gamma }\braket{p_j^\gamma}\braket{  q_i^\alpha}
    +\Delta_2\braket{p_i^\gamma p_j^\gamma }\braket{ q_i^\alpha}
    +\braket{p_i^\gamma}\Delta_2\braket{ p_j^\gamma  q_i^\alpha}\\
    &\hspace{45mm}+\Delta_2\braket{p_i^\gamma   q_i^\alpha}\braket{p_j^\gamma}
    -\braket{p_i^\gamma }\braket{p_j^\gamma }\braket{ q_j^\alpha}
    -\Delta_2\braket{p_i^\gamma p_j^\gamma}\braket{  q_j^\alpha}\\
    &\hspace{45mm}-\braket{p_i^\gamma}\Delta_2\braket{ p_j^\gamma  q_j^\alpha}
    -\Delta_2\braket{p_i^\gamma   q_j^\alpha}\braket{p_j^\gamma}\Big )\, . 
\end{align*}
Already this term depends on the mixed correlations which we will compute shortly, but we first compute the position correlation. 

\subsection{$\braket{qq}$ correlation}

Recall the formula of the Kolmogorov operator applied to $q_i^\alpha q_j^\beta$, 
\begin{align}
    \mathscr L (q_i^\alpha q_j^\beta) &=  q_i^\alpha\frac{\partial h}{\partial p_j^\beta}+  \sum_l \sigma_l^\alpha(\mathbf q_i ) \sigma_l^\beta(\mathbf q_j) +\frac12 \sum_l  q_i^\alpha\sigma_l^\gamma(\mathbf q_j) \frac{\partial \sigma_l^\beta(\mathbf q_j)}{\partial q_j^\gamma}+ (i\leftrightarrow j)\, ,
    \label{Lqq-appendix}
\end{align}
which together with \eqref{Lq} gives the time evolution of the position correlation in the form
\begin{align*}
    \frac{d}{dt} \Delta_2\braket{q_i^\alpha q_j^\beta} &= A_{qq} + B_{qq} + C_{qq} \,, 
\end{align*}
where 
\begin{align*}
    A_{qq} &= \Braket{q_i^\alpha\frac{\partial h}{\partial p_j^\beta}}- \Braket{q_i^\alpha}\Braket{\frac{\partial h}{\partial p_j^\beta}} +(i \leftrightarrow j)\\
    B_{qq}&=  \sum_l \Braket{\sigma_l^\alpha(\mathbf q_i ) \sigma_l^\beta(\mathbf q_j)}\\
    C_{qq} &= \frac12 \sum_l  \Braket{q_i^\alpha\sigma_l^\gamma(\mathbf q_j) \frac{\partial \sigma_l^\beta(\mathbf q_j)}{\partial q_j^\gamma}}- \frac12\sum_l  \Braket{q_i^\alpha}\Braket{\sigma_l^\gamma(\mathbf q_j) \frac{\partial \sigma_l^\beta(\mathbf q_j)}{\partial q_j^\gamma}}+ (i \leftrightarrow j)\,.
\end{align*}
We will denote by $A$ the terms corresponding to the drift, by $B$ the terms which are not present in the first moments equation, and by $C$ the other terms which only depend on the noise  and the derivative of the noise fields. 
We proceed by first approximating the expectation of the kernels to get 
\begin{align*}
    B_{qq} & \approx  \sum_l\sigma_l^\alpha(\braket{\mathbf q_i} ) \sigma_l^\beta(\Braket{\mathbf q_j})\\
    C_{qq} &\approx  -\frac{1}{2\alpha_l^2} \sum_{l,\gamma}  \Delta_2\Braket{q_i^\alpha q_j^\gamma}\sigma_l^\gamma(\braket{\mathbf q_j}) \sigma_l^\beta(\braket{\mathbf q_j})+(i \leftrightarrow j)\, . 
\end{align*}
where we also used the explicit form of $\sigma_l$ as a Gaussian and its derivative. 
We will now approximate the $A_{qq}$ term  to get 
\begin{align*}
    A_{qq} &= \sum_k \Braket{q_i^\alpha  p_k^\beta K(\mathbf q_j- \mathbf  q_k) }- \sum_k \Braket{q_i^\alpha}\Braket{p_k^\beta K(\mathbf q_j - \mathbf q_k)} + (i \leftrightarrow j)\\
    &\approx \sum_k \Delta_2 \Braket{q_i^\alpha  p_k^\beta }K(\Braket{\mathbf q_j}- \Braket{\mathbf  q_k}) +(i \leftrightarrow j)\, . 
\end{align*}
It is now clear that the $B$ term will linearly increase the position correlation, which will then exponentially increase by the $C$ term and be affected by the momentum-position correlation by the $A$ term. 
We now proceed by computing the momentum correlation. 

\subsection{$\Braket{pp}$ correlation}

We compute the Kolmogorov operator on $p_i^\alpha p_j^\beta$ to get 
\begin{align*}
    \mathscr L( p_i^\alpha p_j^\beta) &= -p_i^\alpha \frac{\partial h}{\partial q_j^\beta} - \frac12\sum_{l,\gamma,\delta} p_i^\alpha p_j^\gamma \sigma_l^\delta(\mathbf q_j)\frac{\partial^2 \sigma_l^\gamma(\mathbf q_j)}{\partial q_j^\beta \partial q_j^\delta} +(i\leftrightarrow j) \\
    &+ \frac12 \sum_{l,\gamma,\delta}p_i^\delta p_j^\gamma \frac{\partial \sigma_l^\delta (\mathbf q_i)}{\partial q_i^\alpha}\frac{\partial \sigma_l^\gamma(\mathbf q_j)}{\partial q_j^\beta} + \frac12\sum_{l,\gamma,\delta} p_i^\alpha p_j^\delta \frac{\partial \sigma_l^\delta (\mathbf q_j)}{\partial q_j^\gamma}\frac{\partial \sigma_l^\gamma(\mathbf q_j) }{\partial q_j^\beta} + (i\leftrightarrow j)\, , 
\end{align*}
and using \eqref{Lp}, we obtain the time evolution of the correlation in three terms as 
\begin{align*}
    \frac{d}{dt}\Delta_2\braket{p_i^\alpha p_j^\beta} &= A_{pp} + B_{pp} + C_{pp} \, , 
\end{align*}
where 
\begin{align*}
    A_{pp} & = -\Braket{p_i^\alpha \frac{\partial h}{\partial q_j^\beta}} +  \Braket{p_i^\alpha}\Braket{ \frac{\partial h}{\partial q_j^\beta}} +(i \leftrightarrow j)\\ 
    B_{pp} &=\frac12 \sum_{l, \gamma, \delta} \Braket { p_i^\delta p_j^\gamma \frac{\partial \sigma_l^\delta (\mathbf q_i)}{\partial q_i^\alpha}\frac{\partial \sigma_l^\gamma(\mathbf q_j)}{\partial q_j^\beta}} +(i \leftrightarrow j)\\
    C_{pp}& = - \frac12 \sum_{l, \gamma, \delta}\Braket{p_i^\alpha p_j^\gamma \sigma_l^\delta(\mathbf q_j)\frac{\partial^2 \sigma_l^\gamma(\mathbf q_j)}{\partial q_j^\beta \partial q_j^\delta}} + \frac12 \sum_{l, \gamma, \delta}\Braket{p_i^\alpha}\Braket{ p_j^\gamma \sigma_l^\delta(\mathbf q_j)\frac{\partial^2 \sigma_l^\gamma(\mathbf q_j)}{\partial q_j^\beta \partial q_j^\delta}} \\
    &+ \frac12 \sum_{l, \gamma, \delta}\Braket{p_i^\alpha p_j^\delta \frac{\partial \sigma_l^\delta (\mathbf q_j)}{\partial q_j^\gamma}\frac{\partial \sigma_l^\gamma(\mathbf q_j) }{\partial q_j^\beta}}- \frac12 \sum_{l, \gamma, \delta}\Braket{p_i^\alpha}\Braket{ p_j^\delta \frac{\partial \sigma_l^\delta (\mathbf q_j)}{\partial q_j^\gamma}\frac{\partial \sigma_l^\gamma(\mathbf q_j)}{\partial q_j^\beta}} + (i \leftrightarrow j)\, . 
\end{align*}

We first approximate
{\footnotesize
\begin{align*}
    C_{pp} &\approx  \frac12 \sum_{l,\gamma,\delta}\frac{1}{\alpha_l^2} \sigma_l^\beta (q_j) \sigma^\gamma_l(q_j) \Delta_2 \Braket{p_i^\alpha p_j^\gamma}\\
    & -\frac12 \sum_{l,\gamma,\delta} \sigma_l^\delta(q_j) \sigma_l^\gamma(q_j)\left (\Braket{p_i^\alpha p_j^\gamma(q_j^\beta-\delta_l^\beta) ( q_j^\delta - \delta_l^\delta)}  - \Braket{p_i^\alpha}\Braket{ p_j^\gamma(q_j^\beta-\delta_l^\beta) ( q_j^\delta - \delta_l^\delta)}  \right)\\
    &+\frac12\sum_{l,\gamma,\delta}  \sigma_l^\delta (q_j)\sigma_l^\gamma(q_j)\left (  \Braket{p_i^\alpha p_j^\delta  (q_j^\gamma-\delta_l^\gamma)(q_j^\beta-\delta_l^\beta)} 
    +  \Braket{p_i^\alpha}\Braket{ p_j^\delta (q_j^\gamma-\delta_l^\gamma)(q_j^\beta-\delta_l^\beta)}\right ) +(i \leftrightarrow j)\, . 
\end{align*}
}
The last two terms cancel as they are symmetric under the transpose operation because of the sum on the free indices, thus the $C$ term is 
\begin{align}
    C_{pp} \approx  \frac12 \sum_{l,\gamma,\delta}\frac{1}{\alpha_l^2} \sigma_l^\beta (q_j) \sigma^\gamma_l(q_j) \Delta_2 \Braket{p_i^\alpha p_j^\gamma} + (i \leftrightarrow j) \, .
\end{align}

We proceed with the $B_{pp}$ term, which is also symmetric under the transpose operation, thus giving the approximation 
{\footnotesize
\begin{align*}
    B_{pp}  &=\sum_{l,\gamma,\delta}\frac{1}{\alpha_l^4}\sigma_l^\delta (q_i)\sigma_l^\gamma(q_j)\Braket { p_i^\delta p_j^\gamma  (q_i^\alpha-\delta_l^\alpha)(q_j^\beta-\delta_l^\beta)}\\
    &\approx \sum_{l,\gamma,\delta}\frac{1}{\alpha_l^4}\sigma_l^\delta (q_i)\sigma_l^\gamma(q_j) \Big (\Braket{p_i^\delta}\Braket{p_j^\gamma}\Braket{q_i^\alpha}\Braket{q_j^\beta} +  
     \Delta_2\Braket { p_i^\delta p_j^\gamma}\Braket{ q_i^\alpha}\Braket{ q_j^\beta}\\
    &\hspace{40mm}+\Braket { p_i^\delta}\Delta_2\Braket{ p_j^\gamma q_i^\alpha}\Braket{ q_j^\beta}
    +\Braket { p_i^\delta}\Braket{ p_j^\gamma}\Delta_2\Braket{ q_i^\alpha q_j^\beta}\\
    &\hspace{40mm}+\Delta_2\Braket { p_i^\delta q_i^\alpha}\Braket{ p_j^\gamma}\Braket{ q_j^\beta}
    +\Delta_2\Braket { p_i^\delta q_j^\beta }\Braket{p_j^\gamma}\Braket{ q_i^\alpha }\\
    &\hspace{40mm}+\Braket { p_i^\delta}\Braket{ q_i^\alpha}\Delta_2\Braket{p_j^\gamma  q_j^\beta}
    +\Delta_2\Braket { p_i^\delta p_j^\gamma}\Delta_2\Braket{ q_i^\alpha q_j^\beta}\\
    &\hspace{40mm}+\Delta_2\Braket { p_i^\delta q_i^\alpha }\Delta_2\Braket{ p_j^\gamma q_j^\beta}
    +\Delta_2\Braket { p_i^\delta q_j^\beta}\Delta_2\Braket{p_j^\gamma q_i^\alpha }\\
    &\hspace{40mm}-\Braket{p_i^\delta}\Braket{p_j^\gamma}\Braket{q_i^\alpha}\delta_l^\beta 
    -  \Delta_2\Braket{p_i^\delta p_j^\gamma} \Braket{q_i^\alpha}\delta_l^\beta\\
    &\hspace{40mm}- \Braket{p_i^\delta }\Delta_2\Braket{p_j^\gamma q_i^\alpha}\delta_l^\beta 
    - \Delta_2\Braket{p_i^\delta  q_i^\alpha}\braket{p_j^\gamma} \delta_l^\beta\\
    &\hspace{40mm}-\Braket{p_i^\delta}\Braket{p_j^\gamma}\Braket{q_j^\beta}\delta_l^\alpha 
    -   \Delta_2 \Braket{p_i^\delta p_j^\gamma}\Braket{ q_j^\beta} \delta_l^\alpha \\
    &\hspace{40mm}- \Braket{p_i^\delta}\Delta_2\Braket{ p_j^\gamma q_j^\beta} \delta_l^\alpha
    - \Delta_2\Braket{p_i^\delta q_j^\beta}\Braket{ p_j^\gamma} \delta_l^\alpha \\
    &\hspace{40mm}+ \Braket{p_i^\delta}\Braket{p_j^\gamma}\delta_l^\alpha\delta_l^\beta 
    + \Delta_2\Braket{p_i^\delta p_j^\gamma}  \delta_l^\alpha\delta_l^\beta \Big )\, . 
\end{align*}
}
We treat the two Hamiltonian terms separately by first writing them explicitly as 
{\footnotesize
\begin{align*}
    A_{pp} 
    & = \sum_{k,\gamma} \frac{1}{\alpha^2} \Braket{p_i^\alpha p_j^\gamma p_k^\gamma (q_j^\beta -q_k^\beta) K(q_j-q_k)}\\
    &-\frac{1}{\alpha^2}  \sum_{k,\gamma}\Braket{p_i^\alpha}\Braket{p_j^\gamma p_k^\gamma  (q_j^\beta-q_k^\beta) K(q_j-q_k)} +(i \leftrightarrow j)\\
    &=: \sum_{k,\gamma} \frac{1}{\alpha^2}( A_{pp}^1 - A_{pp}^2) +(i \leftrightarrow j)\, .
\end{align*}
}
We expand the first term to arrive at 
{\footnotesize
\begin{align*}
    A_{pp}^1 & \approx  K(\braket{q_j}-\braket{q_k})\left (\Braket{p_i^\alpha p_j^\gamma p_k^\gamma q_j^\beta} -\Braket{p_i^\alpha p_j^\gamma p_k^\gamma q_k^\beta)} \right )\\
     & \approx  K(\braket{q_j}-\braket{q_k})\Big (
     \braket{p_i^\alpha}\braket{ p_j^\gamma}\braket{ p_k^\gamma}\braket{ q_j^\beta} 
    +\Delta_2\Braket{p_i^\alpha p_j^\gamma}\braket{ p_k^\gamma}\braket{ q_j^\beta}\\
    &\hspace{35mm}+\Braket{p_i^\alpha}\Delta_2\braket{ p_j^\gamma p_k^\gamma}\braket{ q_j^\beta}
    +\Braket{p_i^\alpha}\braket{ p_j^\gamma}\Delta_2\braket{ p_k^\gamma q_j^\beta}\\
    &\hspace{35mm}+\Delta_2\Braket{p_i^\alpha  p_k^\gamma}\braket{p_j^\gamma }\braket{q_j^\beta}
    +\Delta_2\Braket{p_i^\alpha q_j^\beta}\braket{p_j^\gamma}\braket{ p_k^\gamma }\\
    &\hspace{35mm}+\braket{p_i^\alpha}\braket{ p_k^\gamma}\Delta_2 \braket{ p_j^\gamma q_j^\beta} 
    +\Delta_2\Braket{p_i^\alpha p_j^\gamma}\Delta_2\braket{ p_k^\gamma q_j^\beta}\\
    &\hspace{35mm}+\Delta_2\Braket{p_i^\alpha p_k^\gamma}\Delta_2\braket{ p_j^\gamma  q_j^\beta}
    +\Delta_2\Braket{p_i^\alpha q_j^\beta}\Delta_2\braket{ p_k^\gamma p_j^\gamma}\\
    &\hspace{35mm}-\braket{p_i^\alpha}\braket{ p_j^\gamma}\braket{ p_k^\gamma}\braket{ q_k^\beta} 
    -\Delta_2\Braket{p_i^\alpha p_j^\gamma}\braket{ p_k^\gamma}\braket{ q_k^\beta}\\
    &\hspace{35mm}-\Braket{p_i^\alpha}\Delta_2\braket{ p_j^\gamma p_k^\gamma}\braket{ q_k^\beta}
    -\Braket{p_i^\alpha}\braket{ p_j^\gamma}\Delta_2\braket{ p_k^\gamma q_k^\beta}\\
    &\hspace{35mm}-\Delta_2\Braket{p_i^\alpha  p_k^\gamma}\braket{p_j^\gamma }\braket{q_k^\beta}
    -\Delta_2\Braket{p_i^\alpha q_k^\beta}\braket{p_k^\gamma}\braket{ p_k^\gamma }\\
    &\hspace{35mm}-\Delta_2\Braket{p_k^\gamma q_k^\beta}\braket{p_i^\alpha}\braket{ p_k^\gamma }
    -\Delta_2\Braket{p_i^\alpha p_j^\gamma}\Delta_2\braket{ p_k^\gamma q_k^\beta}\\
    &\hspace{35mm}-\Delta_2\Braket{p_i^\alpha p_k^\gamma}\Delta_2\braket{ p_j^\gamma  q_k^\beta}
    -\Delta_2\Braket{p_i^\alpha q_k^\beta}\Delta_2\braket{ p_k^\gamma p_j^\gamma}\Big )\, , 
\end{align*}
}
and the second term 
{\footnotesize
\begin{align*}
    A_{pp}^2 & \approx   K(\braket{q_j}-\braket{q_k})\left (\Braket{p_i^\alpha}\Braket{ p_j^\gamma p_k^\gamma q_j^\beta} -\Braket{p_i^\alpha}\Braket{p_j^\gamma p_k^\gamma q_k^\beta)} \right )\\
     & \approx  K(\braket{q_j}-\braket{q_k})\Big (
     \braket{p_i^\alpha}\braket{ p_j^\gamma}\braket{ p_k^\gamma}\braket{ q_j^\beta} + \Braket{p_i^\alpha}\Delta_2\braket{ p_j^\gamma p_k^\gamma}\braket{ q_j^\beta}\\
    &\hspace{35mm}+\Braket{p_i^\alpha}\Delta_2\braket{ p_j^\gamma p_j^\beta}\braket{ q_k^\gamma}
    \Braket{p_i^\alpha}\braket{ p_j^\gamma}\Delta_2\braket{ p_k^\gamma q_j^\beta}\\
    &\hspace{35mm}-\braket{p_i^\alpha}\braket{ p_j^\gamma}\braket{ p_k^\gamma}\braket{ q_k^\beta} 
    -\Braket{p_i^\alpha}\Delta_2\braket{ p_j^\gamma p_k^\gamma}\braket{ q_k^\beta}\\
    &\hspace{35mm}-\Braket{p_i^\alpha}\braket{ p_j^\gamma}\Delta_2\braket{ p_k^\gamma q_k^\beta}
    -\Braket{p_i^\alpha}\braket{ p_k^\gamma}\Delta_2\braket{ p_j^\gamma q_k^\beta}\Big )\, . 
\end{align*}
}
This term cancels the terms of the $A_{pp}^1$ proportional to $\braket{p_i^\alpha}$ to give the approximation 
{\footnotesize
\begin{align*}
    \begin{split}
    A_{pp} & \approx \frac{1}{\alpha^2} \sum_{k,\gamma} K(\braket{q_j}-\braket{q_k})\Big (
    \Delta_2\Braket{p_i^\alpha p_j^\gamma}\braket{ p_k^\gamma}\braket{ q_j^\beta}
    +\Delta_2\Braket{p_i^\alpha  p_k^\gamma}\braket{p_j^\gamma }\braket{q_j^\beta}\\
    &\hspace{45mm}+\Delta_2\Braket{p_i^\alpha q_j^\beta}\braket{p_j^\gamma}\braket{ p_k^\gamma }
    +\Delta_2\Braket{p_i^\alpha p_j^\gamma}\Delta_2\braket{ p_k^\gamma q_j^\beta}\\
   &\hspace{45mm}+\Delta_2\Braket{p_i^\alpha p_k^\gamma}\Delta_2\braket{ p_j^\gamma  q_j^\beta}
    +\Delta_2\Braket{p_i^\alpha q_j^\beta}\Delta_2\braket{ p_k^\gamma p_j^\gamma}\\
    &\hspace{45mm}-\Delta_2\Braket{p_i^\alpha p_j^\gamma}\braket{ p_k^\gamma}\braket{ q_k^\beta}
    -\Delta_2\Braket{p_i^\alpha  p_k^\gamma}\braket{p_j^\gamma }\braket{q_k^\beta}\\
    &\hspace{45mm}-\Delta_2\Braket{p_i^\alpha q_k^\beta}\braket{p_j^\gamma}\braket{ p_k^\gamma }
    -\Delta_2\Braket{p_i^\alpha p_j^\gamma}\Delta_2\braket{ p_k^\gamma q_k^\beta}\\
    &\hspace{45mm}-\Delta_2\Braket{p_i^\alpha p_k^\gamma}\Delta_2\braket{ p_j^\gamma  q_k^\beta}
    -\Delta_2\Braket{p_i^\alpha q_k^\beta}\Delta_2\braket{ p_k^\gamma p_j^\gamma}\Big ) +(i \leftrightarrow j) \, . 
    \end{split}
\end{align*}
}
We end this computation by approximating the dynamics of the mixed correlation. 

\subsection{$\Braket{pq}$ correlation}

We compute 
\begin{align*}
    \mathscr L( p_i^\alpha q_j^\beta) &= -q_j^\beta \frac{\partial h}{\partial q_i^\alpha} + p_i^\alpha \frac{\partial h}{\partial p_j^\beta} \\
    & +\frac12\sum_{l,\gamma} p_i^\alpha  \sigma_l^\gamma(\mathbf q_j) \frac{\partial \sigma_l^\beta(\mathbf q_j)}{\partial q_j^\gamma}  - \sum_{l,\gamma}p_i^\gamma\sigma_l^\beta(\mathbf q_j) \frac{\partial \sigma_l^\gamma(\mathbf q_i)}{\partial q_i^\alpha}\\
    &-\frac12\sum_{l,\gamma,\delta} q_j^\beta p_i^\gamma \frac{\partial ^2\sigma_l^\gamma(\mathbf q_i)}{\partial q_i^\alpha \partial q_i^\delta} \sigma_l^\delta (\mathbf q_i)  + \frac12\sum_{l,\gamma,\delta} p_i^\delta q_j^\beta \frac{\partial \sigma_l^\gamma(\mathbf q_i)}{\partial q_i^\alpha}\frac{\partial \sigma_l^\delta(\mathbf q_i)}{\partial q_i^\gamma}\, . 
\end{align*}
Then, using \eqref{Lq} and \eqref{Lp} we obtain the time evolution of $\Delta\braket{p_i^\alpha q_j^\beta}$ as 
\begin{align*}
    \frac{d}{dt} \Delta\braket{p_i^\alpha q_j^\beta} &= A_{pq} + B_{pq} + C_{pq}\, , 
\end{align*}
where
\begin{align*}
    A_{pq} & = -\Braket{q_j^\beta \frac{\partial h}{\partial q_i^\alpha}} + \Braket{q_j^\beta}\Braket{\frac{\partial h}{\partial q_i^\alpha}} + \Braket{p_i^\alpha \frac{\partial h}{\partial p_j^\beta}}-  \Braket{p_i^\alpha }\Braket{\frac{\partial h}{\partial p_j^\beta}} \\
    B_{pq} &=  -\sum_{l,\gamma} \Braket{p_i^\gamma\sigma_l^\beta(\mathbf q_j) \frac{\partial \sigma_l^\gamma(\mathbf q_i)}{\partial q_i^\alpha} }\\
    C_{pq} &=  \frac12 \sum_{l,\gamma}\Braket{p_i^\alpha  \sigma_l^\gamma(\mathbf q_j) \frac{\partial \sigma_l^\beta(\mathbf q_j)}{\partial q_j^\gamma}} -\frac12\sum_{l,\gamma} \braket{p_i^\alpha}\Braket{\sigma_l^\gamma(\mathbf q_j)  \frac{\partial\sigma_l^\beta(\mathbf q_j) }{\partial q_j^\gamma} }\\
    &-\frac12 \sum_{l,\gamma,\delta}\Braket{q_j^\beta p_i^\gamma \frac{\partial ^2\sigma_l^\gamma(\mathbf q_i) }{\partial q_i^\alpha \partial q_i^\delta} \sigma_l^\delta (\mathbf q_i) } + \frac12 \sum_{l,\gamma,\delta}\Braket{q_j^\beta}\Braket{ p_i^\gamma \frac{\partial^2\sigma_l^\gamma(\mathbf q_i) }{\partial q_i^\alpha \partial q_i^\delta} \sigma_l^\delta(\mathbf q_i)}\\
    &+ \frac12 \sum_{l,\gamma,\delta}\Braket{p_i^\delta q_j^\beta \frac{\partial \sigma_l^\gamma(\mathbf q_i)}{\partial q_i^\alpha}\frac{\partial \sigma_l^\delta(\mathbf q_i)}{\partial q_i^\gamma}}
     -\frac12 \sum_{l,\gamma,\delta}\braket{q_j^\beta}\Braket{ p_i^\delta \frac{\partial\sigma_l^\gamma(\mathbf q_i)}{\partial q_i^\alpha}\frac{\partial\sigma_l^\delta(\mathbf q_i)}{\partial q_i^\gamma}} \, . 
\end{align*}
We first approximate 
{\footnotesize
\begin{align*}
    B_{pq}&\approx  \frac{1}{\alpha_l^2} \sum_{l,\gamma} \sigma_l^\beta(\mathbf q_j) \sigma_l^\gamma(\mathbf q_i)\left (\Braket{p_i^\gamma q_i^\alpha}-\Braket{p_i^\gamma }\delta_l^\alpha\right )\\
    C_{pq} & \approx - \sum_{l,\gamma}\frac{1}{2\alpha_l^2} \sigma_l^\beta(\mathbf q_j)\sigma_l^\gamma(\mathbf q_j) \Delta_2\Braket{p_i^\alpha  q_j^\gamma} + \sum_{l,\gamma}\frac{1}{2\alpha_l^2 }\Delta_2 \Braket{q_j^\beta p_i^\gamma} \sigma_l^\gamma(\braket{\mathbf q_i})\sigma_l^\alpha(\braket{\mathbf q_i})\, . 
\end{align*}
}
For the Hamiltonian term we obtain 
{\footnotesize
\begin{align*}
    A_{pq} & \approx \frac{1}{\alpha^2}\sum_{k,\gamma} K(\braket{\mathbf q_i}-\braket{\mathbf q_k})\Big (\Braket{q_j^\beta p_i^\gamma p_k^\gamma q_i^\alpha} - \Braket{q_j^\beta}\Braket{p_i^\gamma p_k^\gamma q_i^\alpha}\\
    &- \Braket{q_j^\beta p_i^\gamma p_k^\gamma q_k^\alpha }+ \Braket{q_j^\beta}\Braket{p_i^\gamma p_k^\gamma q_k^\alpha }\Big ) + \sum_{k,\gamma} K(\braket{\mathbf q_j}-\braket{\mathbf q_k})\Delta_2\Braket{p_i^\alpha p_k^\beta}\\
    & \approx \frac{1}{\alpha^2}\sum_{k,\gamma}K(\braket{\mathbf q_i}-\braket{\mathbf q_k})\Big (
    \Delta_2 \Braket{ p_i^\gamma q_j^\beta}\braket{ p_k^\gamma}\braket{ q_i^\alpha}
    +\Delta_2 \Braket{ p_k^\gamma q_j^\beta }\braket{p_i^\gamma}\braket{ q_i^\alpha}\\
    &\hspace{45mm}+\Delta_2 \Braket{ q_i^\alpha q_j^\beta }\braket{p_i^\gamma}\braket{ p_k^\gamma}
    +\Delta_2 \Braket{ p_i^\gamma q_j^\beta}\Delta_2 \braket{ p_k^\gamma q_i^\alpha}\\
    &\hspace{45mm}+\Delta_2 \Braket{ p_k^\gamma q_j^\beta}\Delta_2 \braket{ p_i^\gamma q_i^\alpha}
    +\Delta_2 \Braket{q_i^\alpha q_j^\beta }\Delta_2 \braket{ p_k^\gamma  p_i^\gamma}\\
    &\hspace{45mm}- \Delta_2 \Braket{ p_i^\gamma q_j^\beta}\braket{ p_k^\gamma}\braket{ q_k^\alpha } 
    - \Delta_2 \Braket{ p_k^\gamma q_j^\beta}\braket{ p_i^\gamma}\braket{ q_k^\alpha } \\
    &\hspace{45mm}- \Delta_2 \Braket{q_j^\beta q_k^\alpha}\braket{p_i^\gamma}\braket{ p_k^\gamma  } 
    - \Delta_2 \Braket{ p_i^\gamma q_j^\beta}\Delta_2 \braket{ p_k^\gamma q_k^\alpha } \\
    &\hspace{45mm}- \Delta_2 \Braket{ p_k^\gamma q_j^\beta}\Delta_2 \braket{ p_i^\gamma q_k^\alpha } 
    - \Delta_2 \Braket{q_j^\beta q_k^\alpha}\Delta_2 \braket{ p_k^\gamma  p_i^\gamma }\Big ) \\
    &+ \sum_{k,\gamma} K(\braket{\mathbf q_j}-\braket{\mathbf q_k})\Delta_2\Braket{p_i^\alpha p_k^\beta}\, . 
\end{align*}
}

\subsection{Kernel approximation}\label{second-order-moment}
One of the approximations we did in the previous derivation of the moment equation is to replace the expectation of a kernel by the kernel of the expected values. 

If we expand the kernel in powers of its argument, we could compute the errors up to any order. 
For example for a Gaussian kernel, the first few terms are
\begin{align*}
  \Braket{K(q_i-q_j)}- K(\braket{q_i}- \braket{q_j}) =  - \frac{1}{2\alpha^2} (\Delta_2 \braket{q_i^2} - 2\Delta_2\braket{q_iq_j} + \Delta_2\braket{q_j^2}) + \dots
\end{align*}
The main problem with this approximation with polynomials is that the approximation to any order corresponds to having a kernel with unbounded values for large arguments. 
This results in non-physical and large interactions of particles far away, which should normally not interact. 
Obtaining a reliable expansion of a kernel function is thus a difficult task in the moment approximation. 

Nevertheless, one could consider the following higher order approximation of the expected value of a Gaussian kernel
\begin{align}
    K(\mathbf q_i-\mathbf q_j) &= e^{-\|\mathbf q_i-\mathbf q_j\|^2/(2\alpha^2)} \approx \theta(\mathbf q_i-\mathbf q_j) \left ( 1 -f_\alpha \frac{1}{2\alpha^{2}}  \|\mathbf q_i- \mathbf q_j\|^2  \right )\, , 
    \label{exp-approx}
\end{align}
where the function $\theta(\mathbf x)$ is given by 
\begin{align}
    \theta (\mathbf x) = 
    \begin{cases}
        1 & \mathrm{if} \quad 1- f_\alpha \frac{\|\mathbf x\|^2}{2\alpha^{2}}\geq 0  \\
        0 &\mathrm{if}  \quad  1- f_\alpha \frac{\|\mathbf x\|^2}{2\alpha^{2}}<0\, ,
    \end{cases}
\end{align}
and the coefficient $f_\alpha$ is found such that this approximation is the best fit to the Gaussian. 
In practice, we have $f_\alpha \approx 0.6$, but this value depends on $\alpha$ in general.  
This cutoff function $\theta$ is necessary here, otherwise, this approximation will not be bounded, leading to large errors in the dynamics. 
The expected value of this approximation assumes that the $\theta$ function commutes with it, and only takes into account the approximation of the quadratic term. 
It turned out that for all our experiments, these correction terms did not substantially improve the result, thus we did not include them in the equations. 

\end{document}